\documentclass[preprint,11pt,authoryear]{elsarticle}

\usepackage[margin=2.5cm]{geometry}
\usepackage{natbib}
\usepackage{lineno,hyperref}
\modulolinenumbers[5]
\usepackage{amssymb}
\usepackage{amsthm}
\usepackage{graphicx}
\usepackage{wrapfig}
\usepackage[tight]{subfigure}
\usepackage{microtype}
\usepackage{algorithm}
\usepackage{algorithmic}
\usepackage{enumerate}
\usepackage{amsmath}
\usepackage{amsthm}
\usepackage{amsfonts}
\usepackage[font=small,labelfont=bf]{caption}
\usepackage[usenames,dvipsnames]{color}
\usepackage[english]{babel}
\usepackage{multirow}
\usepackage{wrapfig}
\usepackage{tabularx}
\usepackage{pstool}
\usepackage{xcolor}
\usepackage{colortbl}
\newtheorem{prop}{Proposition}
\newcolumntype{C}[1]{>{\centering\let\newline\\\arraybackslash\hspace{0pt}}m{#1}}

\def\ie{{i.e.}}

\def\G{\Gamma}
\def\P{\Psi}

\def\V{\mathcal{V}}

\DeclareMathOperator{\shrink}{shrink}
\DeclareMathOperator*{\argmin}{arg\,min}
\DeclareMathOperator*{\argmax}{arg\,max}
\DeclareMathOperator*{\I}{\mathcal{I}}
\DeclareMathOperator*{\J}{\mathcal{J}}
\graphicspath{{f/}}

\newcommand{\myparagraph}[1]{\smallskip\noindent\textbf{#1.} }

\begin{document}

\begin{frontmatter}

\title{Joint Spatial-Angular Sparse Coding for dMRI with Separable Dictionaries}

\author{Evan Schwab, Ren\'{e} Vidal, and Nicolas Charon}
\address{Center for Imaging Science, Johns Hopkins University, Baltimore, MD USA}




\begin{abstract}

Diffusion MRI (dMRI) provides the ability to reconstruct neuronal fibers in the brain, \textit{in vivo}, by measuring water diffusion along angular gradient directions in $q$-space. High angular resolution diffusion imaging (HARDI) can produce better estimates of fiber orientation than the popularly used diffusion tensor imaging, but the high number of samples needed to estimate diffusivity requires longer patient scan times. To accelerate dMRI, compressed sensing (CS) has been utilized by exploiting a sparse dictionary representation of the data, discovered through sparse coding. The sparser the representation, the fewer samples are needed to reconstruct a high resolution signal with limited information loss, and so an important area of research has focused on finding the sparsest possible representation of dMRI. Current reconstruction methods however, rely on an angular representation \textit{per voxel} with added spatial regularization, and so, for non-zero signals, one is required to have at least one non-zero coefficient per voxel.  This means that the global level of sparsity must be greater than the number of voxels. In contrast, we propose a joint spatial-angular representation of dMRI that will allow us to achieve levels of global sparsity that are below the number of voxels. A major challenge, however, is the computational complexity of solving a global sparse coding problem over large-scale dMRI.  In this work, we present novel adaptations of popular sparse coding algorithms that become better suited for solving large-scale problems by exploiting spatial-angular separability. Our experiments show that our method achieves significantly sparser representations of HARDI than is possible by the state of the art.
\end{abstract}


\end{frontmatter}


\section{Introduction}
\label{sec:intro}
Diffusion magnetic resonance imaging (dMRI) is a medical imaging modality used to analyze neuroanatomical biomarkers for brain diseases such as Alzheimer's.  dMRI are 6D signals consisting of a set of 3D spatial MRI volumes acquired in $k$-space that are each weighted with a different diffusion signal measured in $q$-space.  In each voxel of a brain dMRI, the $q$-space diffusion signals are reconstructed to estimate orientations and integrity of neuronal fiber tracts, \textit{in vivo}.  Different dMRI protocols measure $q$-space in different ways. For example, diffusion spectrum imaging (DSI) \citep{Wedeen:MRM05} measures $q$-space densely on a 3D grid.  Alternatively, diffusion tensor imaging (DTI) \citep{Basser:JMR94} simplifies acquisition by modeling a Gaussian distribution on the unit $q$-sphere.  High angular resolution diffusion imaging (HARDI) \citep{Tuch:MRM2004} also restricts measurements to the unit sphere, but increases the angular resolution from that of DTI.  Multi-Shell HARDI (MS-HARDI) \citep{Wu:Neuroimage07} expands its radial range to include multiple spheres, or shells.
Since DTI collects the fewest number of measurements, it has become the most widely used clinical dMRI protocol.  However, its simple tensor model is unable to capture the complex diffusion profiles in each voxel.  On the other hand, protocols like HARDI, MS-HARDI, and especially DSI, collect a higher number of $q$-space measurements to estimate more accurate diffusion profiles at the expense of longer scan times, making them currently unsuitable for clinical studies.

An ongoing research goal has been to find ways to reduce acquisition times of HARDI, MS-HARDI, or DSI, while maintaining accurate estimations of diffusion.  One avenue is from a hardware perspective: maintain dense signal measurement configurations while devising faster physical acquisition techniques like simultaneous multi-slice acquisition \citep{Setsompop:Neuroimage12} and simultaneous image refocusing \citep{Reese:JMRI09}. 
The other is from a signal processing perspective: maintain accurate signal reconstructions while devising methods to exploit redundancies in the data and reduce the number of required measurements to accelerate acquisition. This paradigm is known as Compressed Sensing (CS) \citep{Donoho:TIT06}.  

CS is a class of mathematical results and algorithms that exploits sparse representations of signals, discovered through sparse coding, to obtain extremely accurate reconstructions at sub-Nyquist rates.  A classical application of CS has been to accelerate structural MRI by subsampling the spatial frequency domain, $k$-space \citep{Lustig:MRM07}, known as $k$-space CS or $k$-CS. These ideas have also been previously applied to dMRI by subsampling the angular frequency domain, $q$-space, \citep{Ning:MIA15} (analogously called $q$-CS) and more recently, to subsample both $k$- and $q$-space \citep{Cheng:IPMI15,Sun:IPMI15}, commonly called ($k,q$)-space CS or ($k,q$)-CS, to further increase acceleration.  However, because the goal of dMRI reconstruction is to estimate diffusivity profiles at each voxel, dMRI signals are traditionally represented as a set of voxel-wise $q$-space signals in the angular domain. Spatial regularization is an important technique used to improve these estimations over an entire dMRI volume \citep{Goh:MICCAI09}, but the underlying data representation of dMRI is still angular and local to each voxel.  Therefore, when applying sparse coding for dMRI, the sparsest possible global representation over an entire volume can be no less than the number of voxels since at least one dictionary atom would be required to model $q$-space signals in each voxel.  

To overcome this fundamental limitation, we propose a global spatial-angular representation of dMRI that allows global sparsity levels to fall \textit{below} one atom per voxel by exploiting redundancies in the spatial and angular domains, \textit{jointly} with a global dictionary.
A major challenge, however, of optimizing over a global dictionary is the computational complexity of solving a massive global sparse coding problem over large-scale dMRI data. Yet, by imposing that our global dictionary is separable over the spatial and angular domains we can greatly improve computational efficiency while preserving good sparsity levels for typical signals.  
One of our main contributions in this paper is a set of novel adaptations of popular sparse coding algorithms to solve general large-scale sparse coding problems using separable dictionaries.
Our experiments on phantom and real HARDI brain data show that it is possible to achieve accurate global HARDI reconstructions with a sparse representation of less than one dictionary atom per voxel, exceeding the theoretical limit of the state of the art in sparse coding. Sparse coding has many important applications like de-noising \citep{Ouyang:IPI13}, dictionary learning \citep{Cheng:MICCAI15} and super-resolution \citep{Yoldemir:CDMRI14}, and, in particular, applying our \textit{joint} spatial-angular sparse coding framework within the application of ($k,q$)-CS will be the subject of future work.

The remainder of this paper is organized as follows: In Section~\ref{sec:state-of-the-art}, we review state-of-the-art sparse coding methods for dMRI and illustrate the limitations of their performance on a phantom HARDI dataset.  In Section~\ref{sec:proposed_framework}, we present our joint spatial-angular dMRI representation and formalize the global spatial-angular sparse coding problem.  Then, in Section~\ref{sec:reconstruction}, we develop and compare a set of novel sparse coding algorithms using separable dictionaries to efficiently solve our large-scale global optimization. Finally, in Section~\ref{sec:experiments} we provide experimental results showing the performance of our method over the state-of-the-art and conclude with a discussion in Section~\ref{sec:conclusion}.

\section{State of the Art}
\label{sec:state-of-the-art}
\subsection{Angular (Voxel-Wise) Reconstruction} A dMRI can be modeled as a 6D signal $\mathcal{S}(v,q)$, where $v \in \Omega \subset \mathbb{R}^3$ is the location of a voxel in the 3D spatial domain $\Omega$ and $q\in\mathbb{R}^3$ is a point in the so-called $q$-space.\footnote{The $q$-space is the frequency domain associated with the angular domain, while the $k$-space is the frequency domain associated with the spatial domain.} A dMRI signal is measured at a discrete number of voxels, $V$, and a discrete number of $q$-space points, $G$.   While dMRI signals can be viewed as a set of $G$ diffusion weighted images (DWIs) or volumes, 
the most common view-point for dMRI processing and analysis is voxel-wise, \ie for each voxel $v\in \Omega$, we acquire a vector of $G$ diffusion measurements $\mathcal{S}(v,q_g)_{g=1}^G = s_v(q_g)_{g=1}^G = s_v$ at points $q_g$ in 3D $q$-space.
The latter interpretation is most common for modeling because a major goal of dMRI reconstruction is to estimate 3D probability distribution functions (PDFs) of fiber tract orientation at each voxel. Accordingly, the signal vector $s_v$ is represented by a $q$-space 
basis, $\Gamma = [\Gamma_i (q)]_{i=1}^{N_\Gamma}$, with $N_\Gamma$ atoms,
such that 
\begin{equation}
s_v = \Gamma a_v.
\end{equation}
where $a_v$ is the vector of angular coefficients at voxel $v$.
The dMRI literature has produced a wide array of dMRI reconstruction algorithms for different acquisition protocols, an artillery of $q$-space bases and varying models for estimating orientation distributions.  
The vast majority of research reconstructs $q$-space signals in each voxel with a $q$-space basis (see the dMRI challenge \citep{Daducci:TMI14} for a comprehensive summary and comparison of state of the art reconstruction frameworks). To enforce or exploit desirable properties of dMRI signals, many methods will add a set of constraints $\mathcal{C}$ on the angular coefficients such as angular smoothing \citep{Ye:MICCAI16}, non-negativity of PDFs \citep{Schwab:MICCAI12,Wolfers:ISBI14}, or orientational symmetry \citep{Gramfort:MIA14}, solving:  
\begin{equation}
\label{eq:recon}
a_v^* = \argmin_{a_v} \frac{1}{2}||\Gamma a_v -s_v||_2^2 \ \ \ \textnormal{s.t.} \ \ a_v \in \mathcal{C}.
\end{equation}  
The constraint of particular interest in our paper is that of enforcing sparsity on the coefficients of the reconstruction, known as \textit{Sparse Coding}.  

\subsection{Angular (Voxel-Wise) Sparse Coding}
Sparse coding is a reconstruction problem which seeks a sparse representation, \ie \ a coefficient vector with few nonzero elements.  Given a sparsifying $q$-space basis $\Gamma$ for which the dMRI signal in each voxel is expected to have a sparse representation, the angular (voxel-wise) sparse coding problem can be formulated as: 
%
\begin{equation}
\label{eq:l0}
a_v^* = \argmin_{a_v} \frac{1}{2}||\Gamma a_v -s_v||_2^2 \ \ \ \textnormal{s.t.} \ \ ||a_v||_0 \leq K_v,
\end{equation}  
where $||a_v||_0$ counts the number of nonzero elements of vector $a_v$, and $K_v$ is the sparsity level at voxel $v$.  This problem is known to be NP-hard, and therefore the two main methodologies to tackle \eqref{eq:l0} are to $a)$ approximate a solution using greedy algorithms such as Orthogonal Matching Pursuit (OMP) \cite{Tropp:TIT04} or $b)$ replace the $L_0$ semi-norm by its convex relaxation, the $L_1$ norm, and solve either the Basis Pursuit or LASSO problem given by: 
\begin{equation}
\label{eq:l1}
a_v^* = \argmin_{a_v} \frac{1}{2}||\Gamma a_v -s_v||_2^2 + \lambda||a_v||_1
\end{equation}
using algorithms such as Alternating Direction Method of Multipliers (ADMM) \citep{Boyd:FTML10} or Fast Iterative Thresholding Algorithm (FISTA) \citep{Beck2009}, where $\lambda$ is the trade-off parameter between data fidelity and sparsity.  
\begin{figure}
\center
\label{fig:odflincomb}
 \includegraphics[width=.8\linewidth,height=.2\linewidth,trim=0 0 0 0,clip]{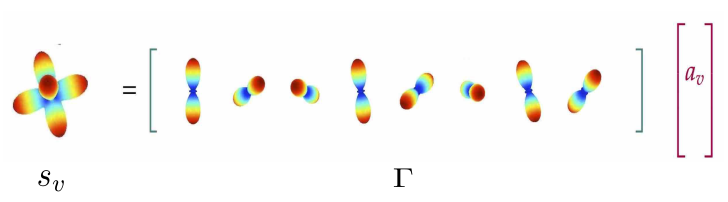}
 \caption{Illustration of voxel-wise angular HARDI representation $a_v$ using a sparsifying dictionary $\Gamma$.}
\end{figure}

An important application of sparse coding in the dMRI community is that of CS. Angular sparse coding and $q$-CS have been widely researched for dMRI to reduce long acquisition times.  Many groups have done extensive work 
choosing sparsifying $q$-space bases \citep{Merlet:ISBI11,Ning:MIA15,Aranda:MIA15}, developing dictionary learning methods \citep{Bilgic:MRM12,Cheng:MICCAI13,Cheng:MICCAI15,Gramfort:MICCAI12,Gramfort:MIA14,Merlet:MICCAI12,Merlet:MIA13,Sun:IPMI13}, and testing $q$-space subsampling schemes for DSI \citep{Gramfort:MIA14,Merlet:MICCAI10, Menzel:MRM11,Merlet:MIA13-2,Paquette:MRM15}, MS-HARDI \citep{Cheng:MICCAI11,Rathi:MICCAI11,Merlet:MIA13,Duarte:MRM14,Daducci:NeuroImage15}, HARDI \citep{Michailovich:MICCAI10,Michailovich:TIP10,TristanVega:MICCAI11,Duarte:MICCAI12, Alaya:ISP16} and DTI \citep{Landman:Neuroimage12} with promising results in sparsity and measurement reduction for clinical tractography \citep{Kuhnt:Neurosurgery13, Kuhnt:PLoS13}.  However, a major limitation for this family of methods is that the sparsest possible representation of an entire dMRI dataset can be no less than the number of voxels since $||a_v||_0 \ge 1 \  \forall \ v \in \Omega$. In CS applications, this induces fundamental limitations in the amount of subsampling factors that may be achievable in $q$-space. 
In practice, the spatial-angular sparsity level will be much greater than the number of voxels, to account for noise.  For example the work of \citep{Michailovich:ISBI08,Michailovich:TIP10} report an average sparsity level of $6$ to $10$ atoms per voxel. The methods presented in the next section attempt to improve upon these results by exploiting spatial redundancies and reducing measurements in $k$-space.

\subsection{Angular Sparse Coding with Spatial Regularization}
Incorporating spatial information into voxel-wise reconstruction is a well utilized technique for increasing the accuracy of reconstruction.  The following is a general formulation for including spatial regularization into the angular sparse coding problem:
\begin{equation}
\label{eq:VoxReconMult}
A^* = \argmin_A ||\Gamma A - S||_F^2 + \lambda ||A||_1 + \mathcal{R} (A),
\end{equation}
where $S = [s_1 \dots s_V] \in \mathbb{R}^{G \times V}$ is the concatenation of signals $s_v\in \mathbb{R}^{G}$ sampled at $G$ gradient directions over $V$ voxels, $A = [a_1 \dots a_V] \in \mathbb{R}^{N_\Gamma \times V}$ is the concatenation of angular coefficients and $\mathcal{R}(A)$ is a spatial regularizer that depends on the angular representation $A$. Here $||X||_F = \sqrt{\sum_i \sum_j |X_{i,j}|^2}$ is the Frobenius norm and $||X||_1 = \sum_i \sum_j |X_{i,j}|$ is the 1-norm taken over all elements of the matrix. In particular when $\mathcal{R} = 0$, this reduces to solving \eqref{eq:l1} for each voxel.  When $\lambda=0$ and $\mathcal{R}(A) = \sum_i\sum_{j \in \mathcal{N}_i} \|a_i -a_j\|^2$ (Laplacian regularization), where $\mathcal{N}_i$ is a local spatial neighborhood of voxel $i$, this is the general non-sparse reconstruction with spatial coherence \citep{Goh:MICCAI09}. Some have found incorporating both the angular sparsity constraint $\lambda ||A||_1$ and spatial coherence $\mathcal{R}(A)$ beneficial for applications such as de-noising \citep{Ye:ISBI12,Ouyang:IPI13,Ouyang:IJBRA14,Ye:MICCAI16} and tractography \citep{Ye:MIA17}.  

Spatial regularization within sparse coding is more prominently used for the application of reducing redundancies for CS.
For example, \citep{Michailovich:TMI11,Rathi:MIA14,Auria:Neuroimage15} enforce spatial smoothing for $q$-CS while \citep{Ning:Neuroimage16,Yin:CCPR16} combine $q$-CS with super-resolution reconstruction of the spatial domain. To further accelerate dMRI, the recent work of \citep{Chao:MRM17} combines CS with parallel imaging but reconstructs the signals in $k$-space and $q$-space separately in sequence.  A joint ($k,q$)-space reconstruction is important for maintaining coherence throughout the dataset.  As such, the works of \citep{Awate:ISBI13,Shi:MRM15,Cheng:IPMI15,Sun:IPMI15,Mcclymont:MRM15,Mani:MRM15} combine $k$- and $q$-CS by adding a data fidelity term for $k$-space subsampling and an additional spatial sparsity term.  
In total, however, while each of these works may be applied to different diffusion models and acquisition protocols testing various subsampling schemes, sparsifying transforms and dictionaries, each are based on an angular representation of dMRI data, $A$. In fact, they stem from the same optimization problem formulation \eqref{eq:VoxReconMult} with
\begin{equation}
\label{eq:spatialreg}
 \mathcal{R}(A) = \alpha ||\mathcal{T}( \Gamma A)||_1,
 \end{equation}
where $\alpha \geq 0$ is an additional trade-off weighting parameter, and $\mathcal{T}(\cdot)$ is a sparsifying transform (or dictionary) applied to the spatial domain such as wavelets or the finite difference gradient operator, leading to the usual total variation (TV) norm.  In \eqref{eq:spatialreg}, $\Gamma A$ is a reconstruction of the signal $S$ based on the angular representation $A$. 

While adding these spatial and angular sparsity terms may exploit redundancies in both the spatial and angular domains, because they are separate disjoint terms the minimal global sparsity level will be still limited by the size of the data since $||A||_0$ should be greater than $V$ and $||\Psi (\Gamma A)||_0$ should be greater than $G$. Indeed, when $||A||_0 < V$, there must exist voxels $v$ such that $a_v = 0$, leading to a zero valued signal $s_v$ (column of $S$) in that voxel.  Likewise, when $||\Psi (\Gamma A)||_0 < G$, there must exist some gradient directions, $q_g$, such that the signal in the entire volume $s_{q}$ (rows of $S$) equals zero.  This becomes a problem because zero valued signals are not physically representative of real dMRI data.  This also becomes a heuristic limitation of prior methods for appropriately choosing trade-off parameters $\lambda$ and $\xi$ that result in a physically accurate sparsity level. 

In the next section we will explicitly show the limitation of  sparsity on phantom HARDI data. 
Table~\ref{table:state-of-the-art} organizes the recent literature's usages of sparse coding and CS for dMRI and places the proposed work in context compared to the state of the art. There we use the term ``Spatial + Angular" Sparse Coding to emphasize that the state of the art perform both spatial and angular sparse coding, but not jointly. As an important note, though we frame our proposed sparse coding method in Table~\ref{table:state-of-the-art} with the backdrop of CS, we do not propose or implement CS in the current manuscript.
\begin{table}[t]
\center
\footnotesize{
\begin{tabular}{cc|c|C{3.5cm}|C{3.5cm}|c|}
\cline{3-6}
& & \multicolumn{4}{ c| }{Sparse Coding} \\ \cline{3-6}
& & \multicolumn{1}{ c| }{Spatial} & \multicolumn{1}{  c| }{Angular} & \multicolumn{1}{ c| }{Spatial + Angular} & \multicolumn{1}{ c| }{ Joint Spatial-Angular } \\ \cline{1-6}
\multicolumn{1}{ |c  }{\multirow{15}{*}{CS}}&
\multicolumn{1}{ |c| }{$k$} & 
\citep{Lustig:MRM07} & & & \\ \cline{2-6}
\multicolumn{1}{ | }{}  &
\multicolumn{1}{ |c| }{$q$} & & 
\citep{Paquette:MRM15}\newline \citep{Auria:Neuroimage15}\newline \citep{Ning:MIA15}\newline \citep{Cheng:MICCAI15}\newline \citep{Daducci:NeuroImage15}\newline \citep{Aranda:MIA15}\newline \citep{Alaya:ISP16} \newline \citep{Ning:Neuroimage16}\newline \citep{Yin:CCPR16}
& &   \\ \cline{2-6}
\multicolumn{1}{ |  }{}  &
\multicolumn{1}{ |c| }{($k,q$)} & & &  \citep{Shi:MRM15}\newline \citep{Cheng:IPMI15}\newline \citep{Sun:IPMI15} \newline \citep{Mcclymont:MRM15} \newline \citep{Mani:MRM15} & Proposed$^\dag$  \\ \cline{1-6}
\end{tabular}}
\caption{Summary of the state-of-the-art dMRI sparse reconstruction methods organized by domains of sparsity (spatial,angular) and CS subsampling ($k,q$).  The literature has provided a natural extension from $k$-CS in MRI using spatial sparse coding to $q$-CS in dMRI angular sparse coding.  However, for ($k,q$)-CS, the state of the art enforce sparsity in the spatial and angular domains separately, (called ``Spatial + Angular" Sparse Coding) with a purely angular representation. In contrast, the proposed work considers a joint spatial-angular representation for sparse coding which is a more natural model for joint ($k,q$)-CS. ($^\dag$Though our proposed spatial-angular sparse coding framework is intended for the application of $(k,q)$-CS as illustrated by this table, the work presented in this paper is only for sparse coding.)}
\label{table:state-of-the-art}
\end{table}

\begin{figure}
\center
 \includegraphics[width=.25\linewidth,trim=0 0 0 0,clip]{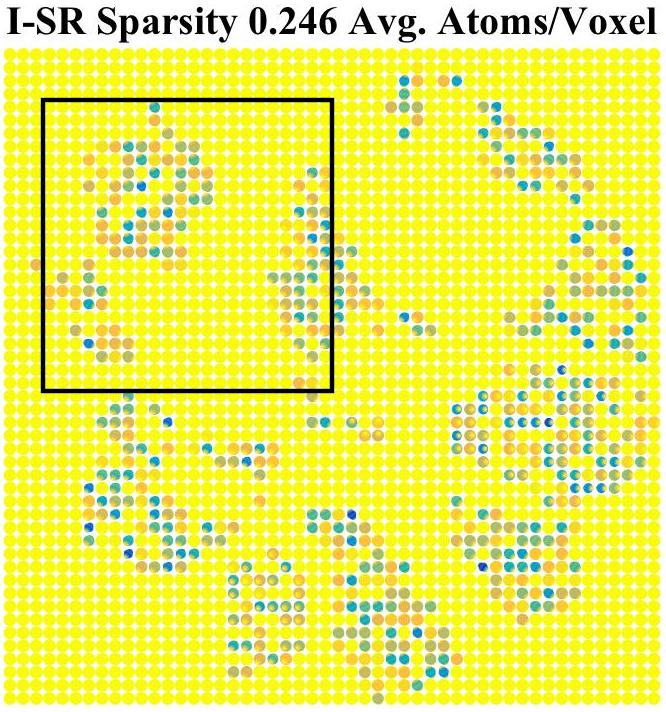}
     \includegraphics[width=.25\linewidth,trim=0 0 0 0,clip]{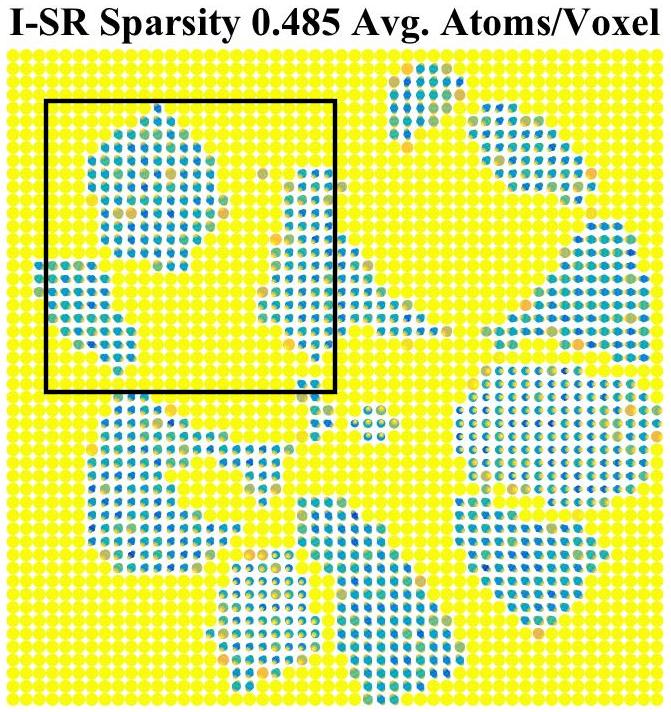}
      \includegraphics[width=.25\linewidth,trim=0 0 0 0,clip]{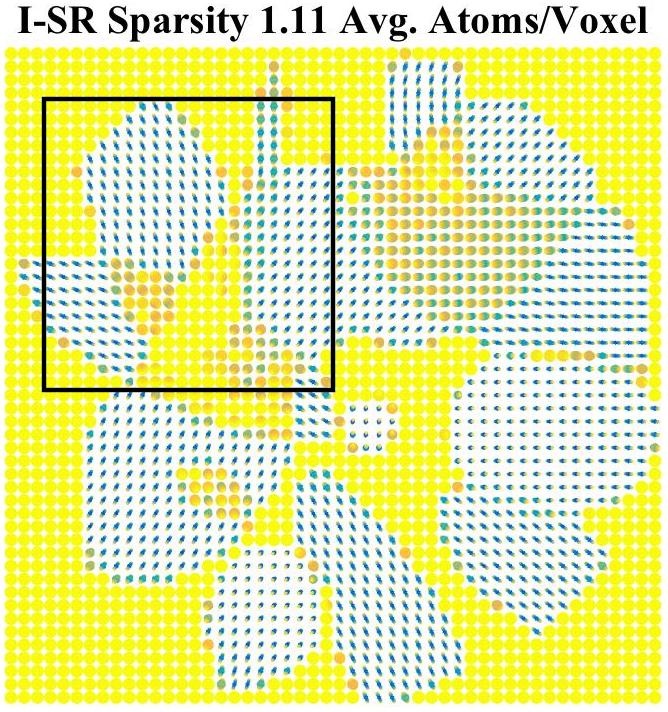}\\
      \includegraphics[width=.25\linewidth,trim=0 0 0 0,clip]{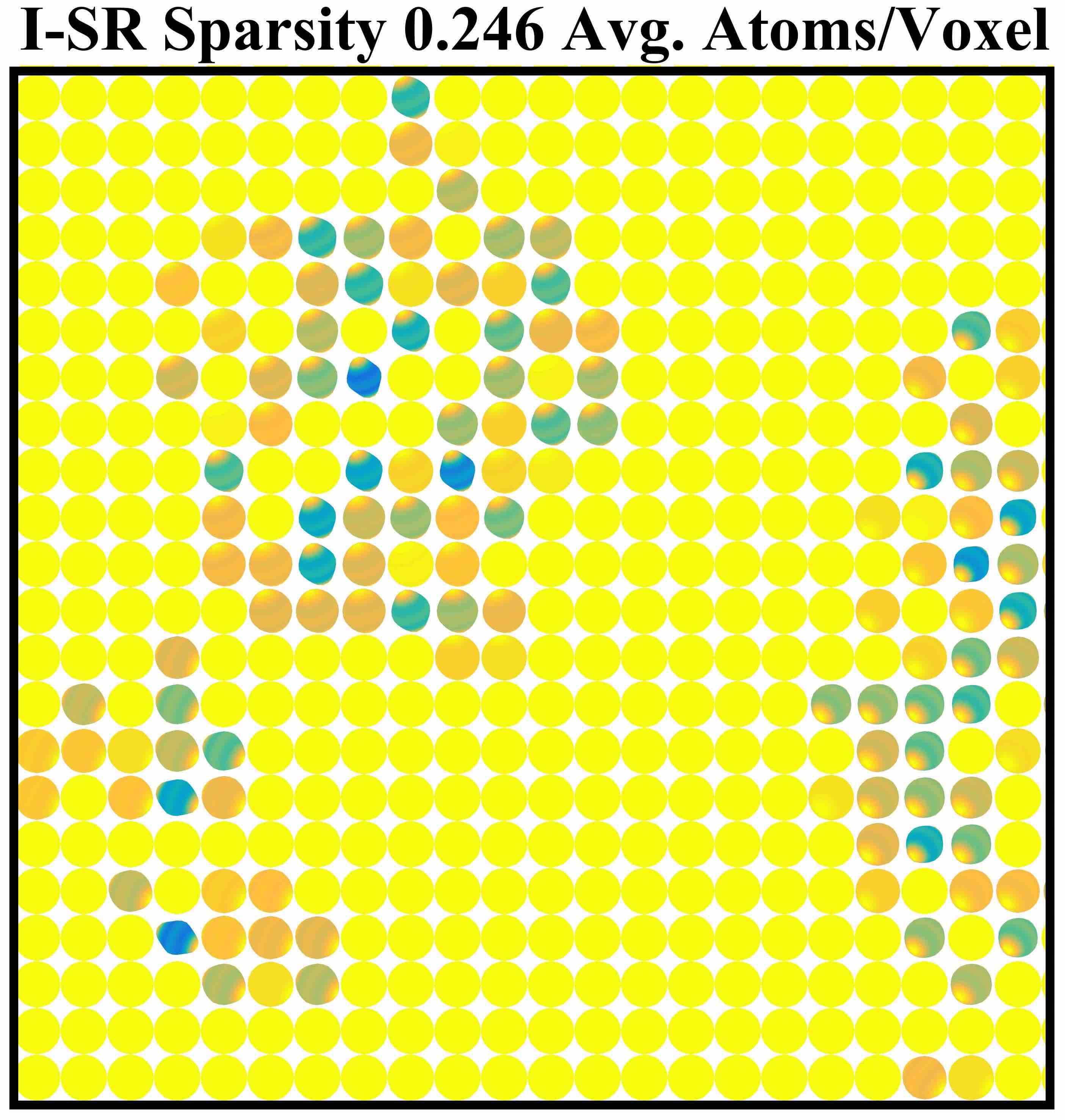}
    \includegraphics[width=.25\linewidth,trim=0 0 0 0,clip]{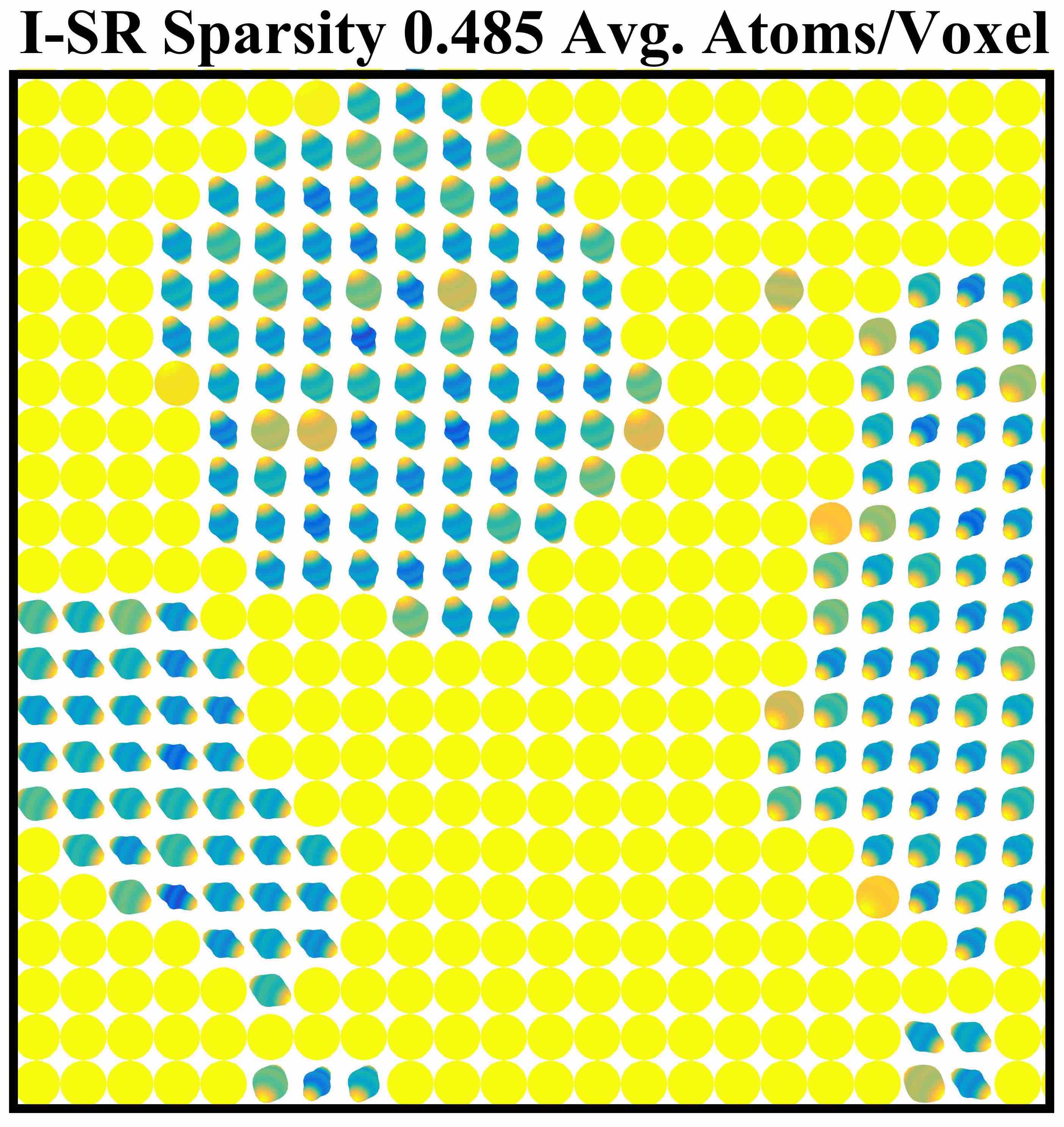}
     \includegraphics[width=.25\linewidth,trim=0 0 0 0,clip]{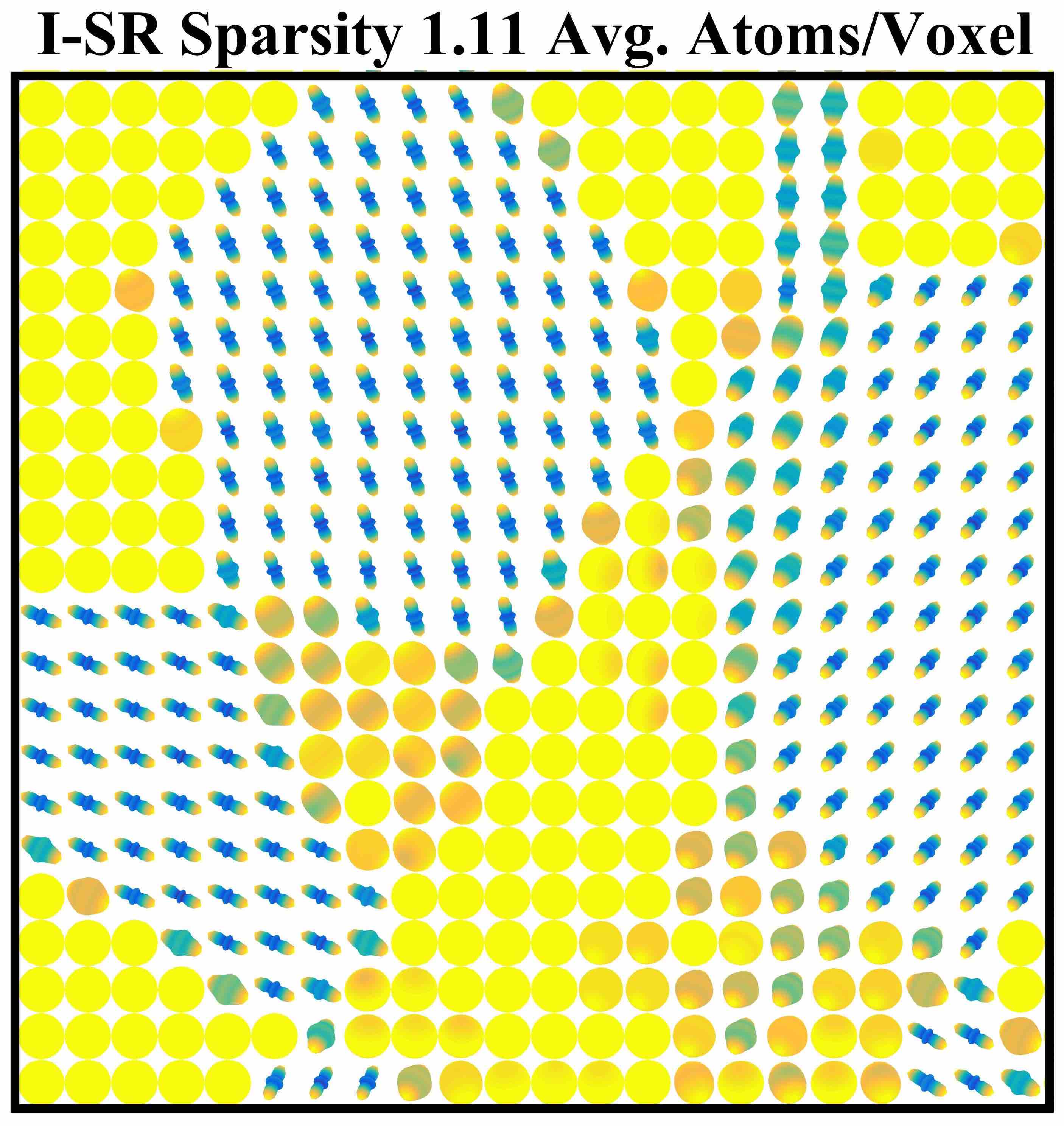}\\
 \includegraphics[width=.25\linewidth,trim=0 0 0 0,clip]{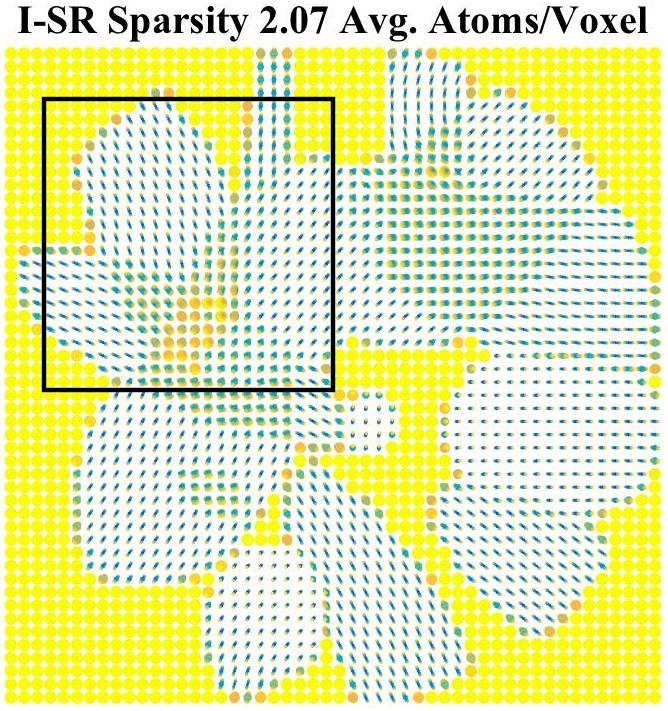}
  \includegraphics[width=.25\linewidth,trim=0 0 0 0,clip]{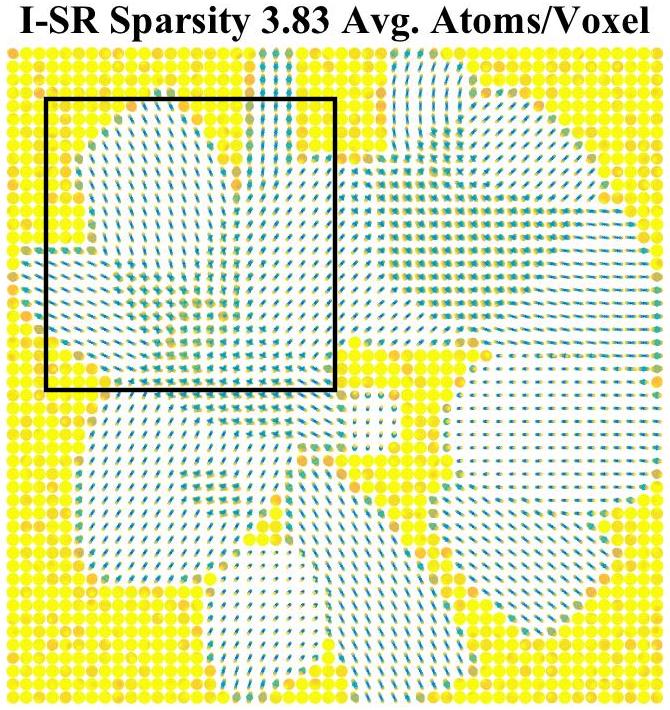}
        \includegraphics[width=.25\linewidth,trim=0 0 0 0,clip]{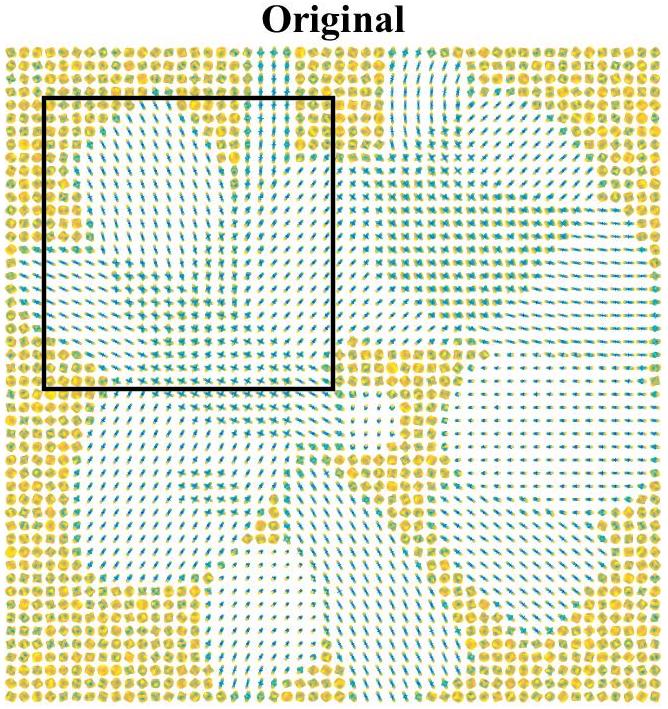}\\
      \includegraphics[width=.25\linewidth,trim=0 0 0 0,clip]{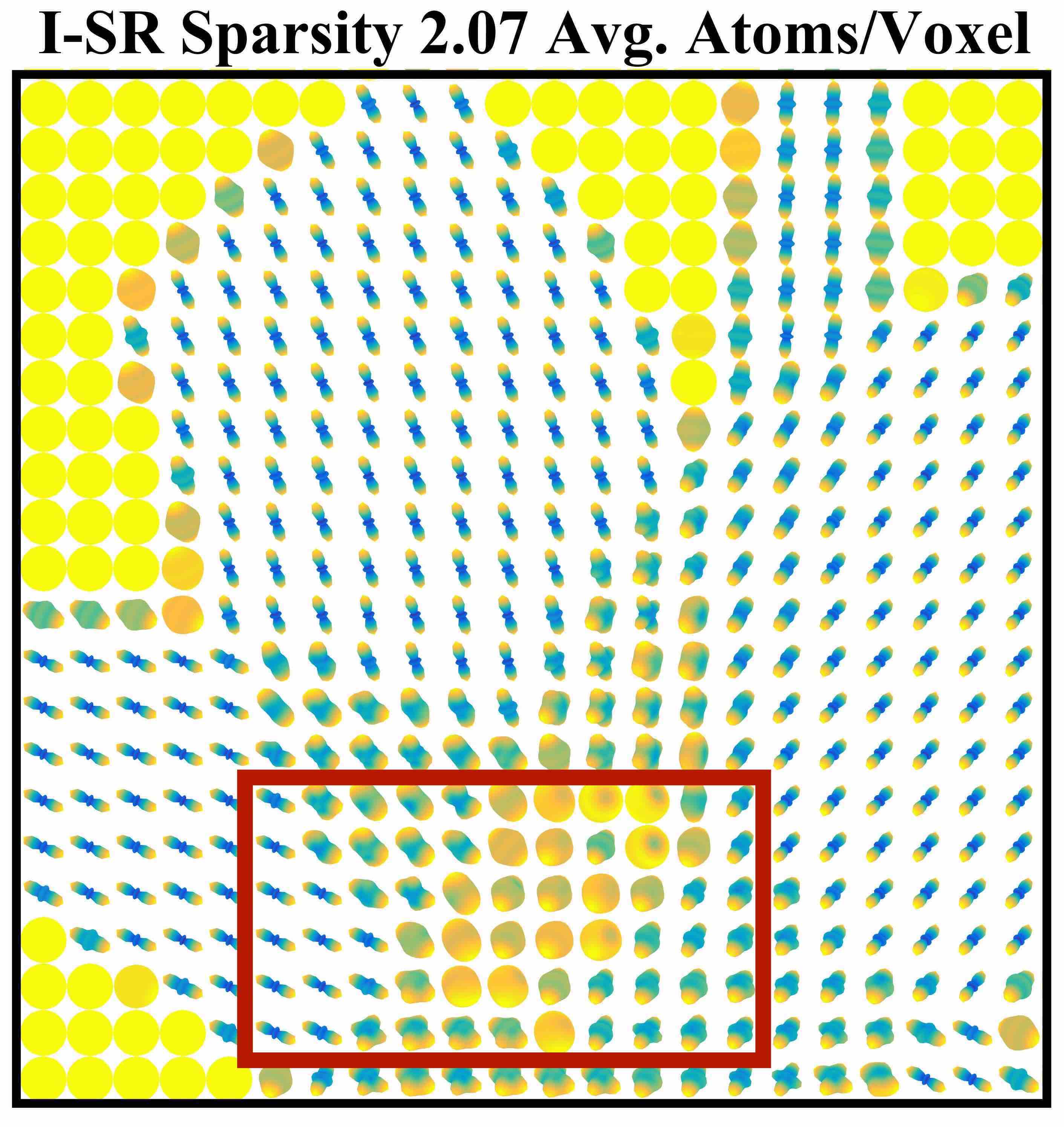}
        \includegraphics[width=.25\linewidth,trim=0 0 0 0,clip]{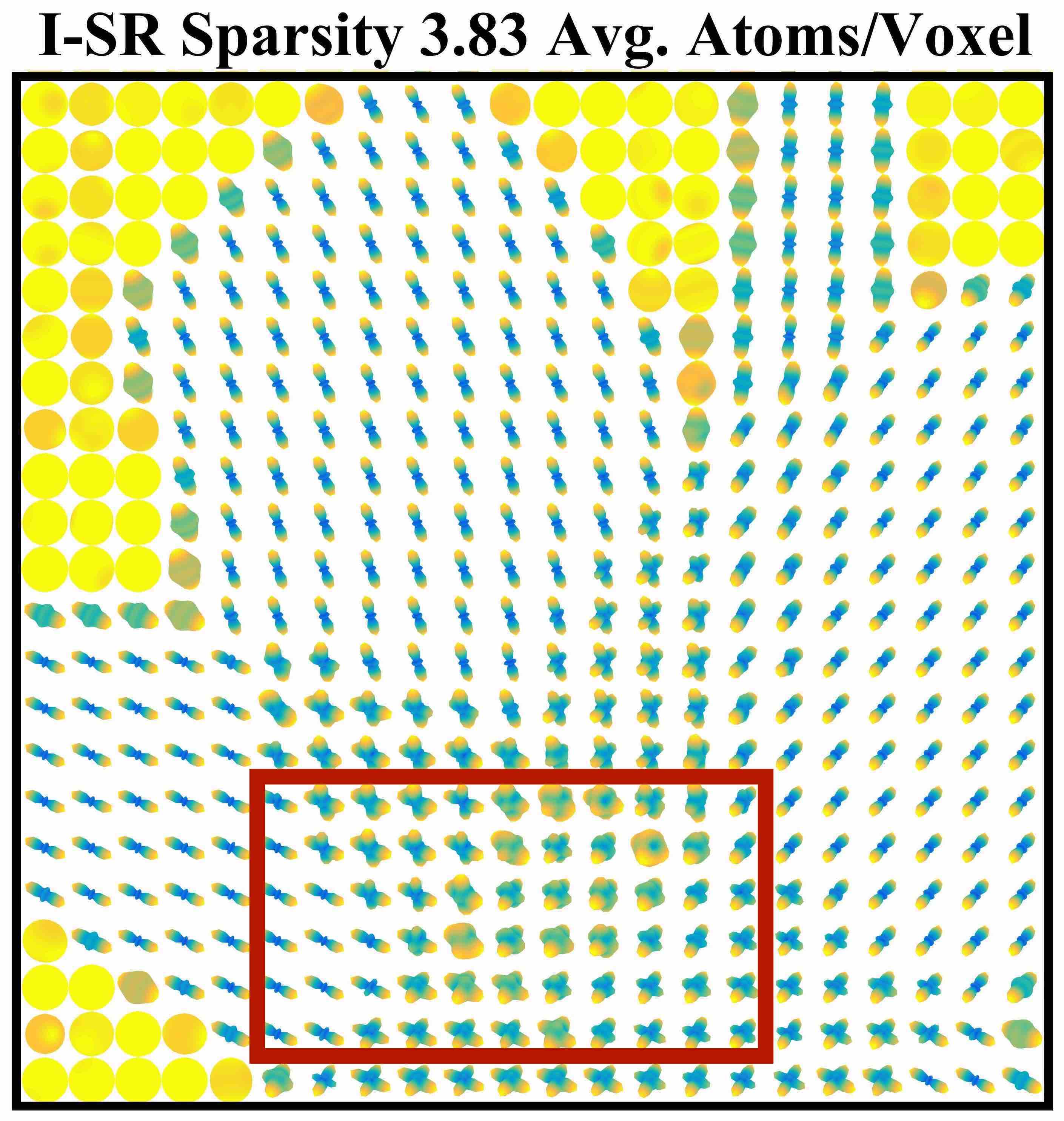}
    \includegraphics[width=.25\linewidth,trim=210 190 170 65,clip]{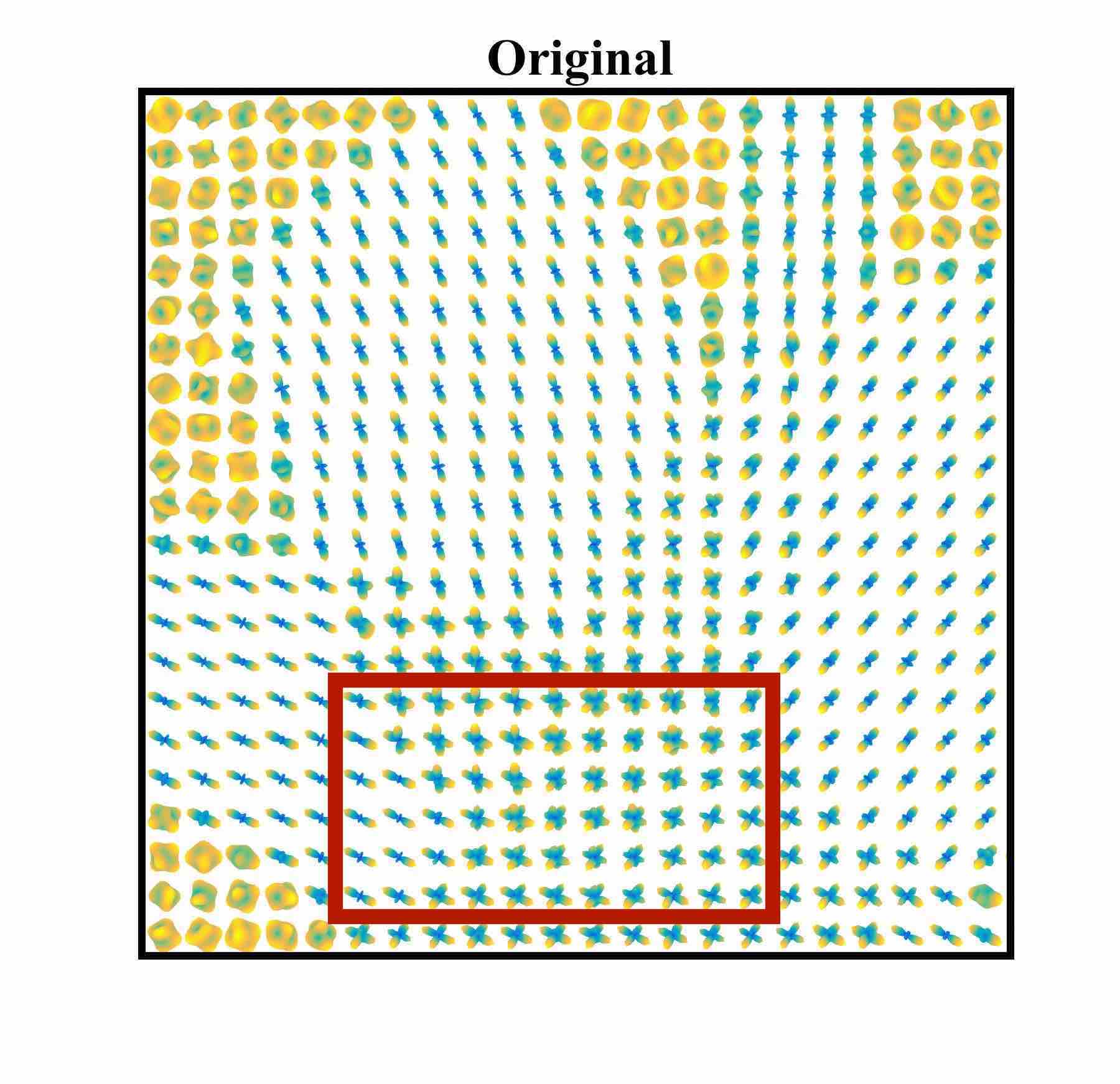}\\
      \includegraphics[width=.25\linewidth,height=.135\linewidth,trim=0 0 0 0,clip]{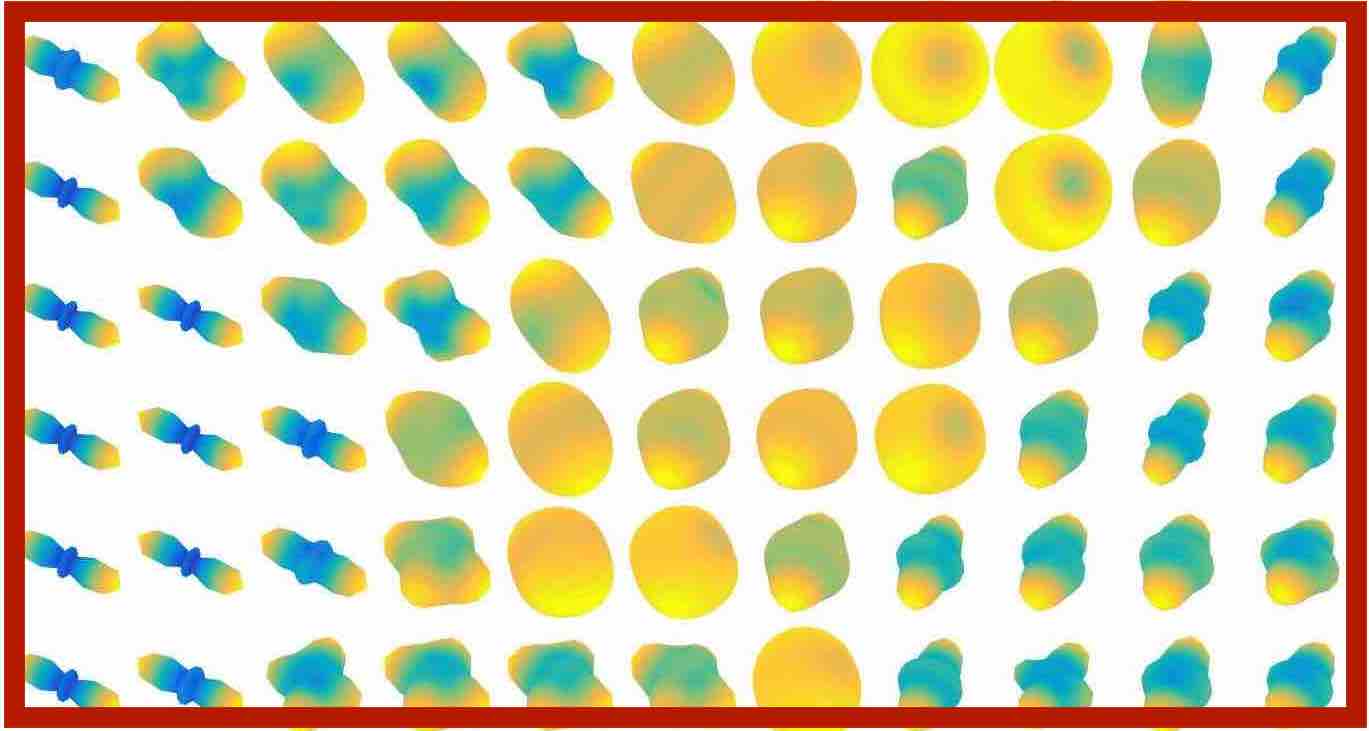}
    \includegraphics[width=.25\linewidth,trim=0 0 0 0,clip]{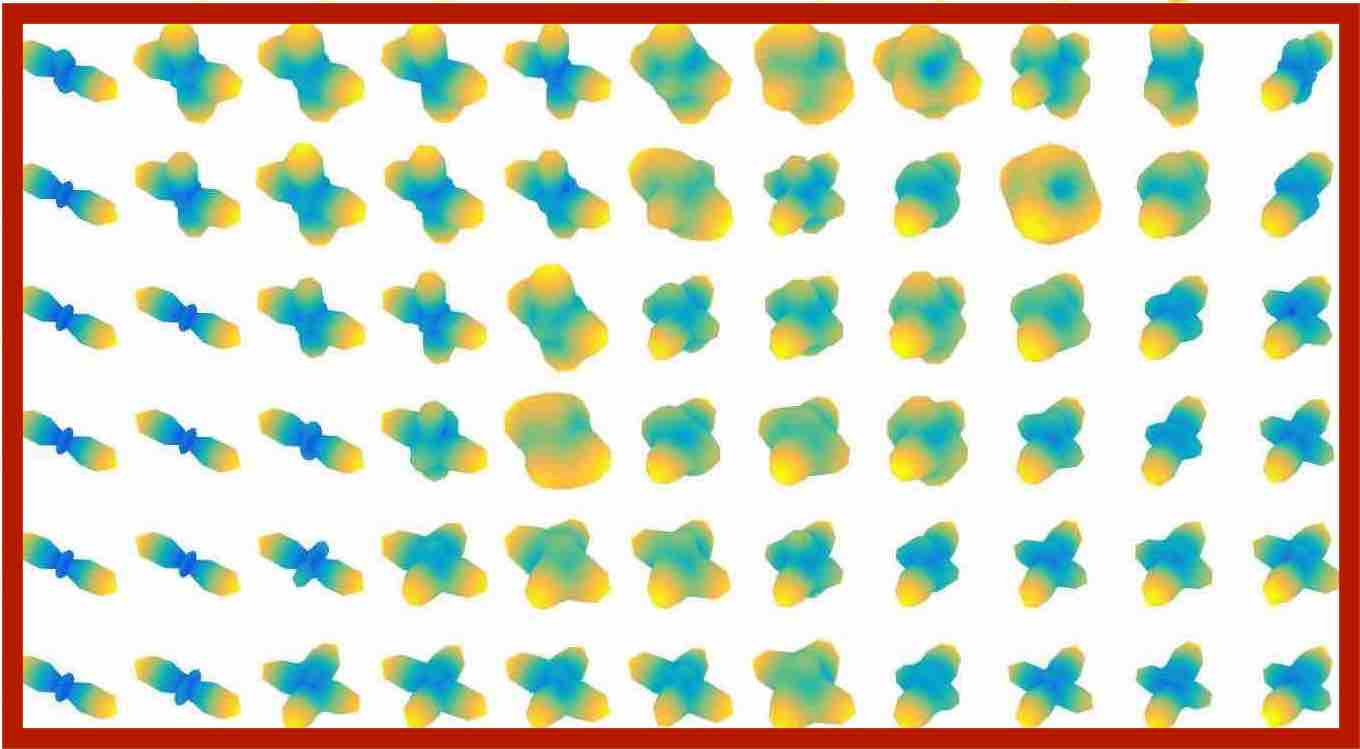}
        \includegraphics[width=.25\linewidth,height=.135\linewidth,trim=0 0 0 0,clip]{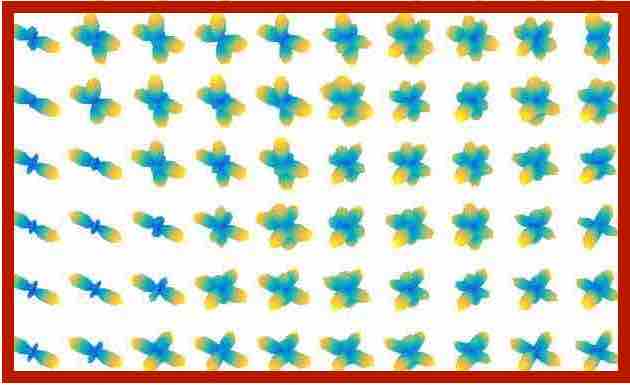}
\caption{Qualitative demonstration of state-of-the-art sparse coding limitations \eqref{eq:VoxReconMult} with the spherical ridgelets (SR) dictionary for 5 different spatial-angular sparsity levels compared to the original signal (bottom right) with ROI closeups underneath.  For high spatial-angular sparsity levels (top left, middle), voxels with complex signals are forced to zero (yellow spheres).  Regions with crossing fibers are unable to be accurately reconstructed even when using an average of 2.07 atoms/voxel. The label I-SR refers to Identity-SR, explained in the next section.}
 \label{fig:IDanalysis}
\end{figure}

\subsection{Limitations of Angular Representations for Sparse Coding}
We illustrate the limitations of sparse coding using a per-voxel angular representation on a HARDI phantom dataset with $V\!=\!50\!\times\!50$ and $G=64$ gradient directions (the same data is used in our experiments in Section~\ref{sec:experiments}). First, we solve \eqref{eq:VoxReconMult} with $\mathcal{R}(A) = 0$, showing qualitative reconstruction results in Figure~\ref{fig:IDanalysis}, for various sparsity levels given by the value of $\lambda$.  Our second result considers the effect of spatial regularization $\mathcal{R}(A) \neq 0$ on the amount of angular sparsity as a function of the reconstruction error in Figure~\ref{fig:angTV}.

For this setting, we choose angular basis $\Gamma$ to be the well performing overcomplete spherical ridglet (SR) dictionary \citep{Michailovich:ISBI08,Michailovich:MICCAI10, TristanVega:MICCAI11}. Figure~\ref{fig:IDanalysis} shows the ODF estimations (computed using the spherical wavelets \citep{TristanVega:MICCAI11}) from the sparse signal reconstruction for various sparsity levels compared to the ODFs estimated from the original signal, as well as close-ups of a region of interest (ROI) containing ODFs with complex crossings of 2, 3 and 4 fibers. In order to compare spatial-angular sparsity levels we are interested in the average number of active dictionary atoms over all voxels, \ie\ $||A||_0/V$.  We use 5 different values of $\lambda$ which gives us average spatial-angular sparsity levels of $0.246, 0.485, 1.11, 2.07$, and $3.84$ atoms per voxel. As expected, when $||A||_0/V < 1$ (see top left/middle), many voxels are forced to zero (as indicated by yellow spheres in Figure~\ref{fig:IDanalysis}). This is especially true for isotropic signals surrounding the fiber tracts. Also as expected, when $||A||_0/V \approx 1$, (see top right) many of the complex signals in the fiber crossing ROI are pushed to zero.  This model requires close to $||A||_0/V = 4$ average atoms per voxel to achieve nearly accurate signal reconstruction (bottom middle).  In fact, the actual number of coefficients per voxel to accurately represent typical dMRI data with angular bases is substantially higher. We illustrate this in Figure~\ref{fig:nonzerocount} which shows the number of atoms used to represent the HARDI signals in each voxel for the reconstructions in Figure~\ref{fig:IDanalysis}. The bottom right image shows the ground truth number of fibers crossing in each voxel. This experiment demonstrates that voxels containing crossing fibers are forced to zero atoms when the average number of atoms per voxel is very small and tend to 6-12 atoms for accurate reconstruction when the sparsity level is decreased. This is consistent with the reports of \citep{Michailovich:ISBI08,Michailovich:TIP10} for the SR dictionary. 

Next, we explore the effect of adding spatial regularization $\mathcal{R}$ to the angular sparsity penalty, as a function of the reconstruction error. As a common spatial regularizer used in the literature, we consider for $\mathcal{T}$ in \eqref{eq:spatialreg} the finite difference (gradient) operator $\mathcal{T} = \nabla := [\partial_x, \partial_y, \partial_z]$ and the corresponding isotropic TV norm given by $||\nabla(X)||_{2,1} = ||\sqrt{|\partial_x X|^2 + |\partial_y X|^2 + |\partial_z X|^2}||_1$. This leads to the new optimization problem
\begin{equation}
\label{eq:TV}
    A^* = \argmin_A ||\Gamma A - S||_F^2 + \lambda ||A||_1 + \alpha||\nabla(\Gamma A)||_{2,1},
\end{equation}
for various $\lambda$ and $\alpha \geq 0$, the relative weight of spatial regularization. This can be solved using Split-Bregman as in \cite{Michailovich:TMI11}. Figure~\ref{fig:angTV} shows the effect of nonzero $\alpha$ on angular sparsity compared to the case of $\alpha=0$ ($\mathcal{R}=0$) on a small $30\times30$ segment of the phantom HARDI data. As we can see, in all cases, the minimal sparsity for accurate reconstruction does not go below the limit of 5 atoms per voxel. In addition, increasing the relative weight of the TV norm spatial regularization actually results in an increase in angular sparsity for a given reconstruction error. In a sense, this is not surprising since the additional regularizer $\mathcal{R}$ will enforce spatial smoothness of the reconstructed signal (which can be beneficial for noisy data and in compressed sensing scenarios) but cannot improve the resulting sparsity of the solution which is still represented by a set of coefficients per voxel in the angular basis $\Gamma$.  
As the goal of this paper is sparse coding, i.e finding sparsest possible representations of full HARDI data, in our later experimental comparisons, we will be using $\mathcal{R}=0$ when referring to state-of-the-art reconstruction.

In the following section, we present our global spatial-angular representation of dMRI which allows for spatial regularization with the possibility to achieve accurate reconstruction at sparsity levels below the number of voxels, unachievable with an angular representation alone.

\begin{figure}[t]
\center
 \includegraphics[width=.25\linewidth,trim=0 0 0 0,clip]{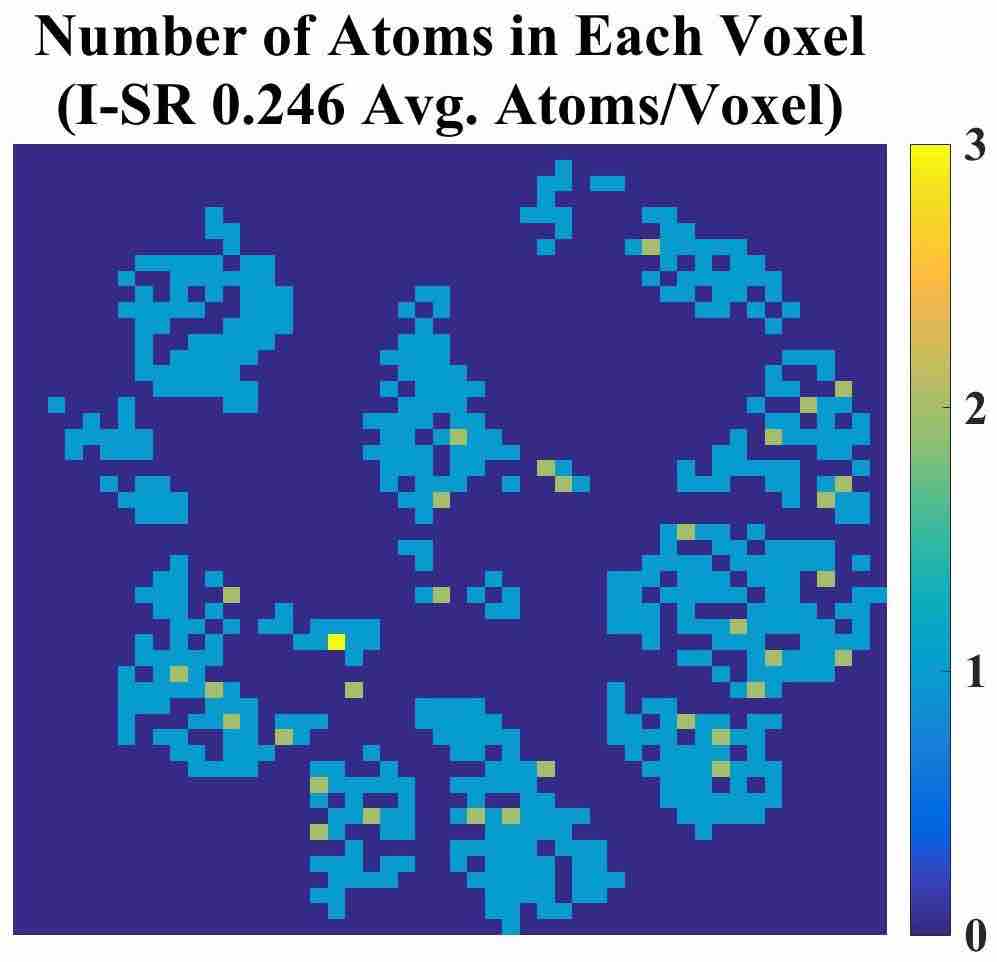}
 \includegraphics[width=.25\linewidth,trim=0 0 0 0,clip]{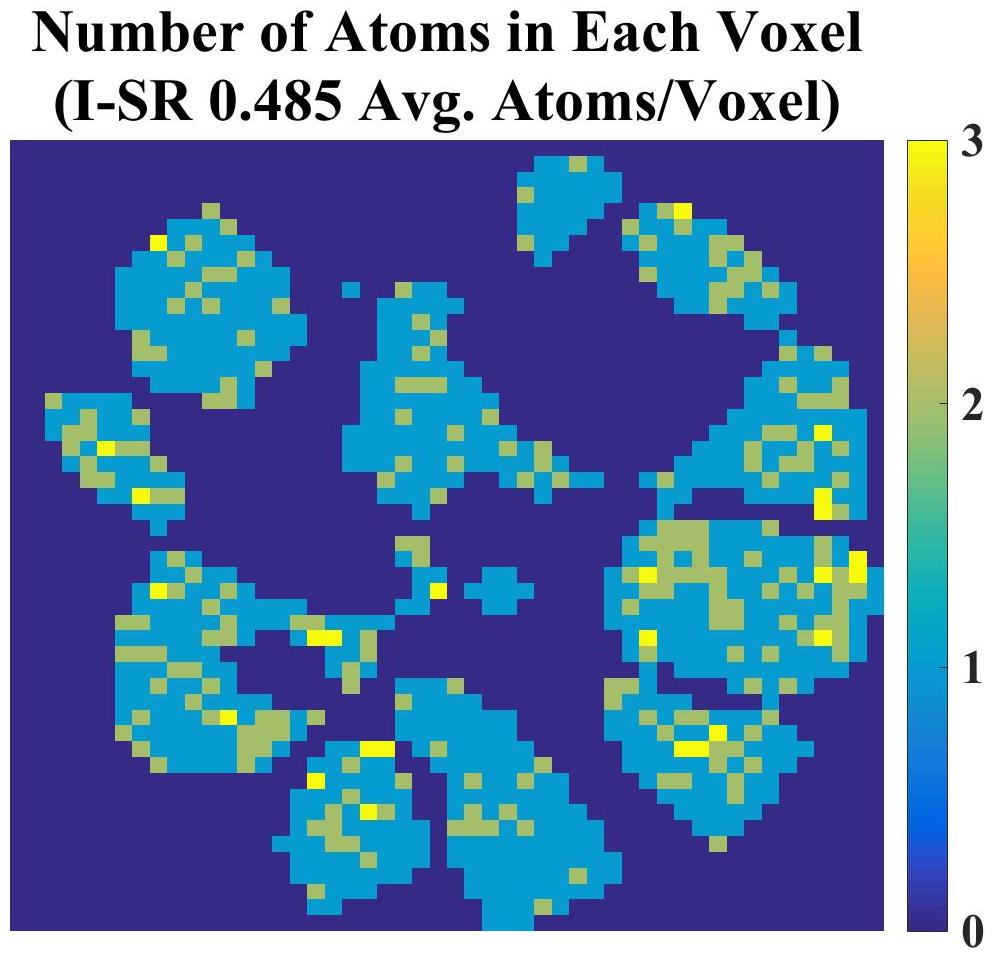}
 \includegraphics[width=.25\linewidth,trim=0 0 0 0,clip]{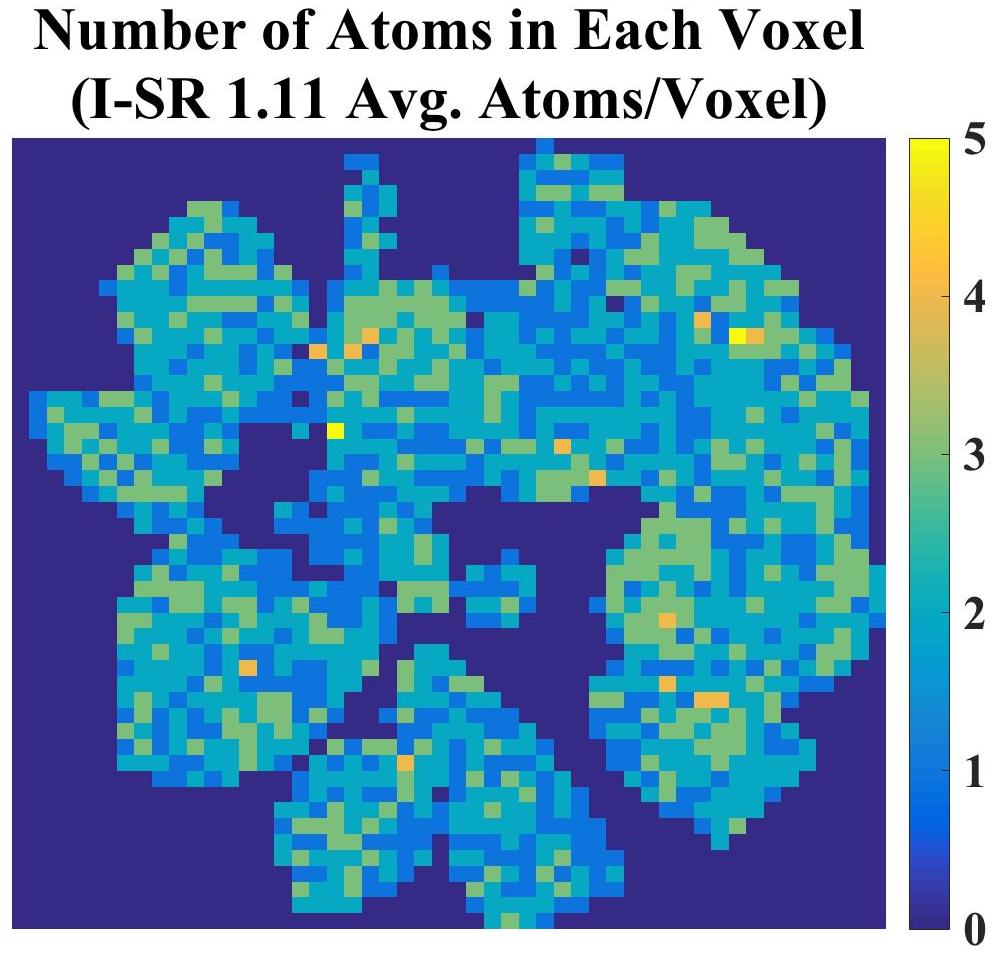}\\
 \includegraphics[width=.25\linewidth,trim=0 0 0 0,clip]{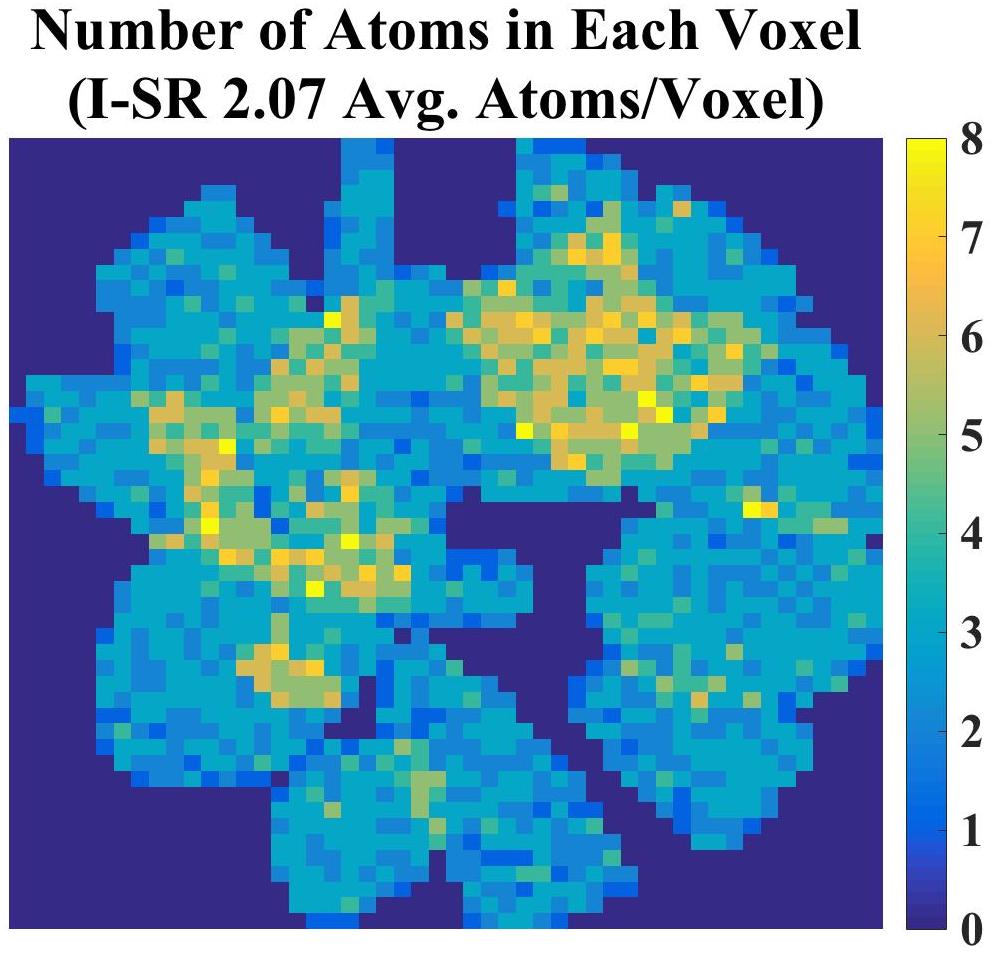}
 \includegraphics[width=.25\linewidth,trim=0 0 0 0,clip]{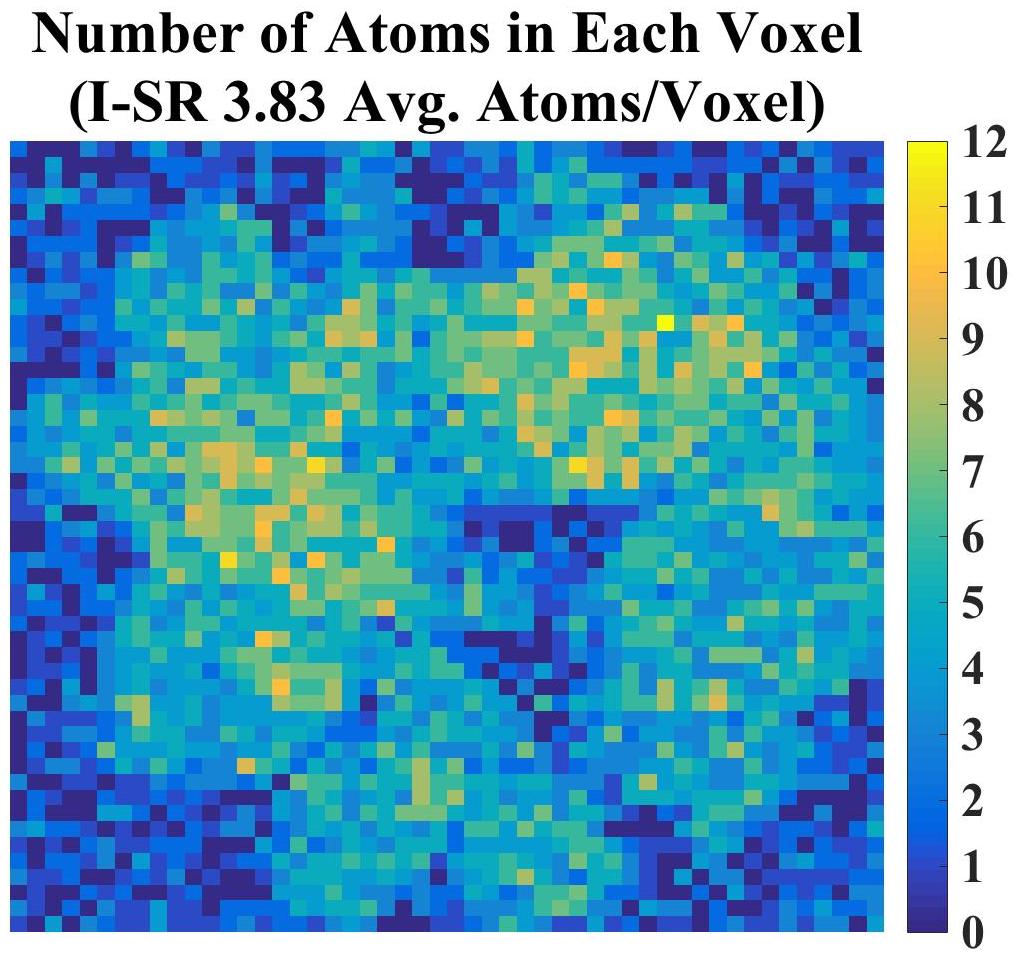}
 \includegraphics[width=.25\linewidth,trim=0 0 0 0,clip]{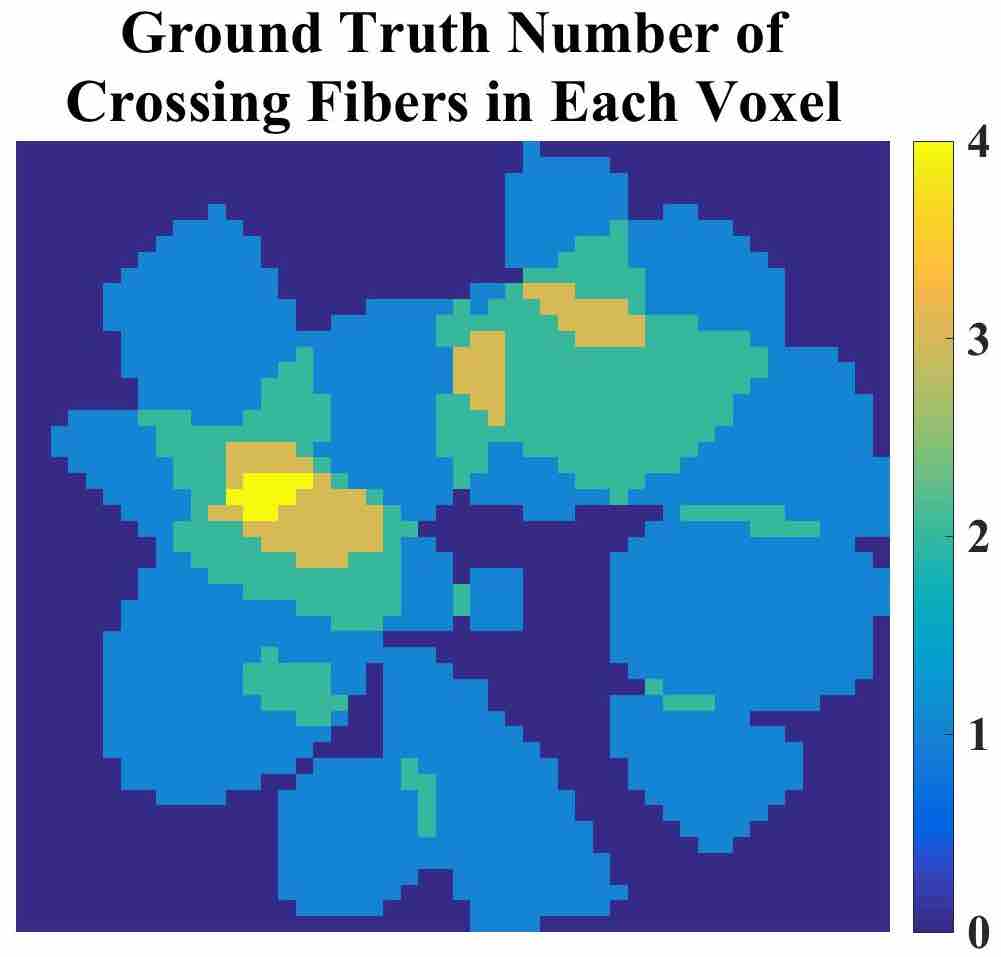}
 \caption{Number of atoms found in each voxel corresponding to the 5 levels of spatial-angular sparsity in Figure~\ref{fig:IDanalysis}. The bottom right figure shows the ground truth number of fibers crossing in each voxel to illustrate the complexity of each angular signal in relation to how many atoms are needed to sparsely model them.  Crossing fiber signals are either forced to zero for high spatial-angular sparsity levels (see: top row) or require between 3-5 atoms for single fiber signals (see: avg. sparsity 1.11 and 2.07) and 6-12 for double and triple crossing fiber signals (see: avg. sparsity 3.83). The label I-SR refers to Identity-SR, explained in the experiments Section~\ref{sec:experiments}.}
 \label{fig:nonzerocount}
\end{figure}

\begin{figure}
    \centering
    \includegraphics[width=1\linewidth]{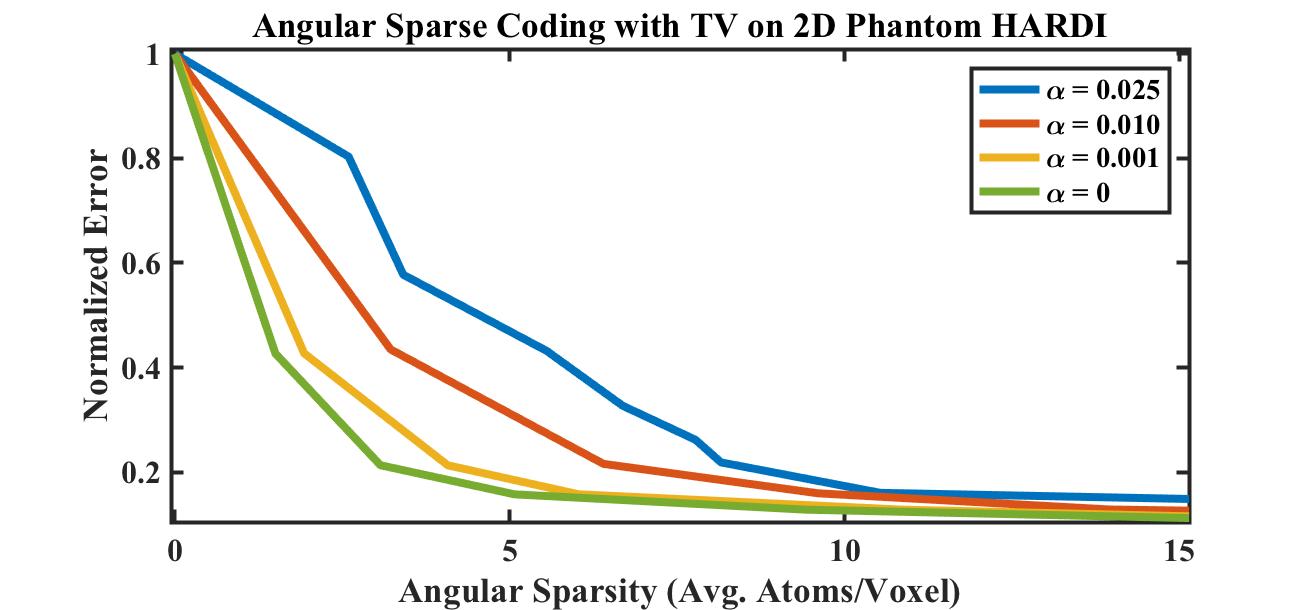}
    \caption{Reconstruction error vs. the average number of angular dictionary atoms per voxel using spatial regularization for the HARDI phantom data. As $\alpha$, the relative weight of spatial regularization (TV) in \eqref{eq:TV}, increases, the average number of angular atoms increases for a given reconstruction error. This suggests that sparser solutions for angular sparse coding can be achieved without spatial regularization, although using adequate spatial regularizers can improve the qualitative aspect of the reconstructed signal, in particular for noisy inputs.}
    \label{fig:angTV}
\end{figure}

\section{Joint Spatial-Angular dMRI Representation}
\label{sec:proposed_framework}
To overcome the sparsity limits of an angular representation, we propose to model a dMRI signal $\mathcal{S}: \Omega \times \mathbb{R}^3 \rightarrow \mathbb{R}$ globally with a joint spatial-angular dictionary, say $\varphi(v,q)$, such that 
\begin{equation}
\mathcal{S}(v,q) = \sum_k c_k \varphi_k(v,q)
\end{equation}
with a single set of global coefficients $c = [c_k]$.  A global dictionary allows us to find global representations with sparsity levels below the number of voxels without forcing some voxels to have zero signal.  In fact, the sparsest possible representation would be the absolute limit of 1 nonzero coefficient $c_k$, and so we find ourselves in a unrestricted setting for global sparse coding. To set up the spatial-angular sparse coding problem, we let $s \in \mathbb{R}^{GV}$ be the vectorization of $\mathcal{S}(v,q)$ where for $v = 1\dots V$ we stack the $q$-space signals, $s_v \in \mathbb{R}^G$, and $\Phi_k \in \mathbb{R}^{GV}$ be the vectorization $\varphi_k(v,q)$ to build the global dictionary $\Phi = [\Phi_1 \dots \Phi_{N_\Phi}] \in \mathbb{R}^{GV \times N_\Phi}$, with $N_\Phi$ atoms. Then, to find a globally sparse $c$, we can solve the $L_0$ minimization problem:
\begin{equation}
\tag{$P0vec$}
\label{eq:l0vec}
c^* = \argmin_c \frac{1}{2}||\Phi c -s||_2^2 \ \ \textnormal{s.t.} \ \ ||c||_0 \leq K,
\end{equation}  
for a sparsity level $K$ or the LASSO problem:
\begin{equation}
\tag{$P1vec$}
\label{eq:l1vec}
c^* = \argmin_c \frac{1}{2}||\Phi c -s||_2^2 + \lambda||c||_1,
\end{equation}
where $\lambda>0$ is the sparsity trade-off parameter.
However, typical dMRI contains on the order of $V\!\approx\!100^3$ voxels each with $G\!\approx\!100$ $q$-space measurements for a total of $100^4=100$ million signal measurements ($|s| \approx 10^8$). Since many sparse coding applications often utilize dictionaries that are over-redundant, this leads to a massive matrix $\Phi$ with $100^4$ rows and over $100^4$ columns ($|\Phi| \approx 10^{16}$). For some datasets, even committing $\Phi$ to memory is prohibitive.  Therefore solving this large-scale global dMRI sparse coding problem using traditional solvers like OMP to approximate \eqref{eq:l0vec} or ADMM and FISTA to solve \eqref{eq:l1vec}, prove intractable.  

\begin{figure}[H]
\center
 \includegraphics[width=1\linewidth,trim=0 0 0 0,clip]{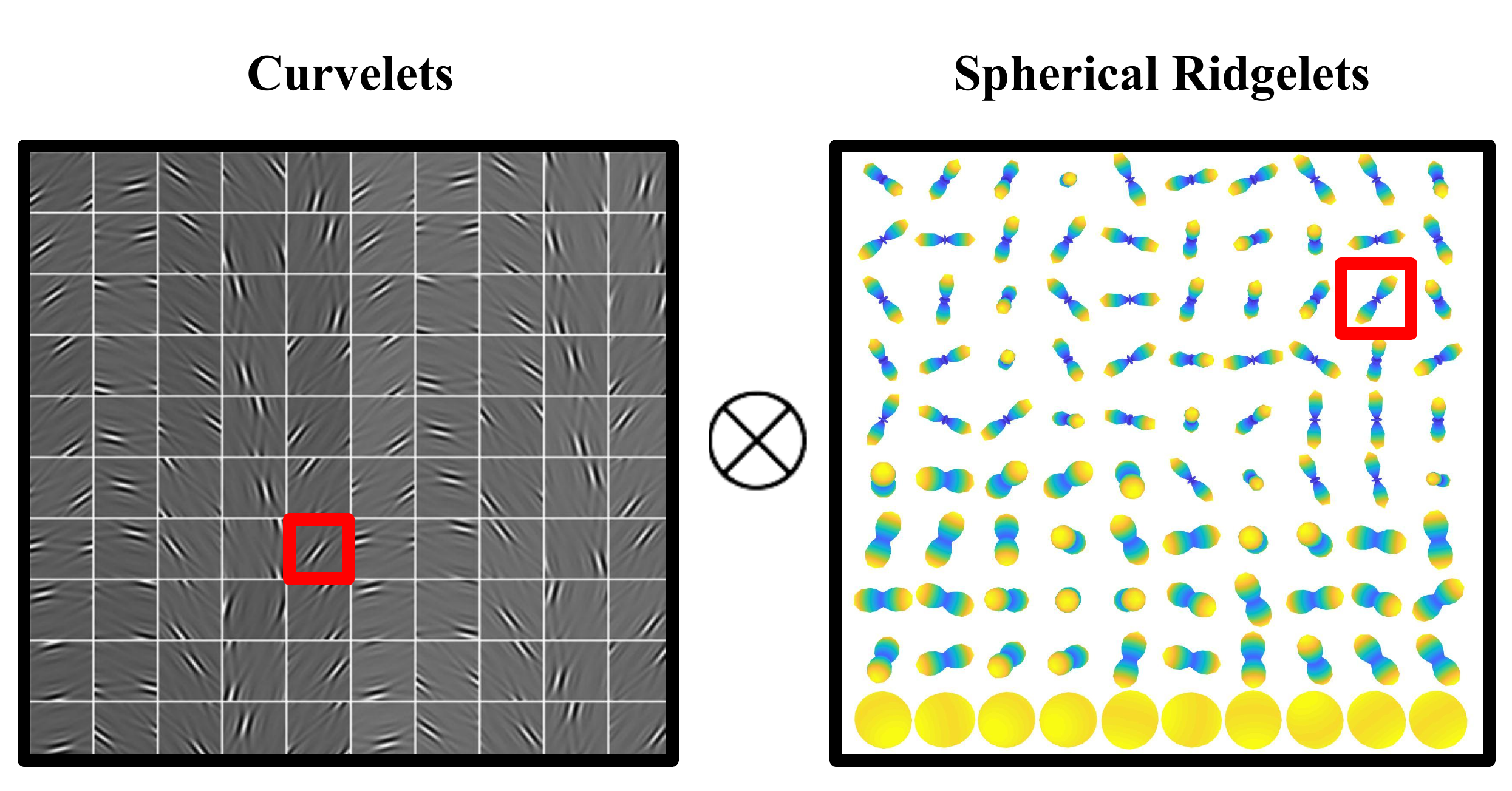}
 \includegraphics[width=1\linewidth,trim=0 0 0 0,clip]{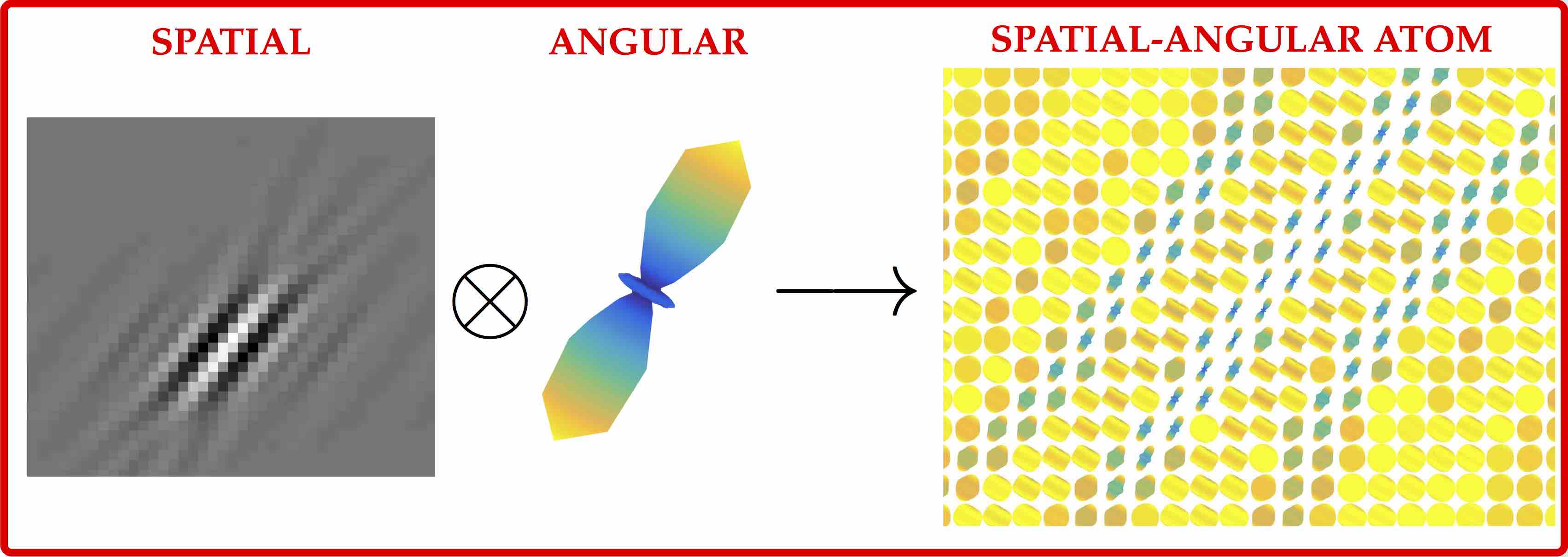}
 \caption{Top: A separable spatial-angular dictionary composed of the Kronecker product between curvelets $\Psi$ and spherical ridgelets $\Gamma$. A pair of spatial and angular atoms are highlighted in red and zoomed in below. Bottom: An example construction of a single spatial-angular basis atom $\Phi_k$ (right) by taking the Kronecker product of $\Psi_j$ (left) and $\Gamma_i$ (middle), \ie \ $\Psi_j \otimes \Gamma_i = \Phi_k$. With this particular combination of spatial (curvelet \citep{Candes:MMS06}) and angular (spherical wavelet \citep{TristanVega:MICCAI11}) atoms, we can see that it may be possible to represent an entire fiber tract with very few spatial-angular atoms.}
 \label{fig:kronbasis}
\end{figure} 

To address this challenge, we introduce additional structure on the dictionary atoms by considering separable functions over $\Omega$ and $\mathbb{R}^3$, namely a set of atoms of the form $\{\varphi_k(v,q)\} = \{\psi_j(v) \otimes \gamma_i(q)\}$, where $\{\psi_j(v)\}$ is a spatial basis for the space of functions from $\Omega \rightarrow \mathbb{R}$ and $\{\gamma_i(q)\}$ is an angular basis for the space of functions from $\mathbb{R}^3\rightarrow \mathbb{R}$ and $\otimes$ is the Kronecker product.  In discretized form for $V$ voxels and $G$ gradient directions, with $\Psi \in \mathbb{R}^{V\times N_\Psi}$ and $\Gamma \in \mathbb{R}^{G\times N_\Gamma}$, the matrix $\Phi = \Psi \otimes \Gamma$ is of the form:
\begin{equation}
s = \left( \begin{array}{c}
\label{eq:SpatialAngularLinCombVec}
s_1 \\
s_2 \\
\vdots \\
s_V \end{array} \right) = \left( \begin{array}{cccc} 
\Psi_{1,1}\Gamma & \Psi_{1,2} \Gamma & \cdots & \Psi_{1,N_\Psi} \Gamma \\
\Psi_{2,1}\Gamma & \Psi_{2,2} \Gamma & \cdots & \Psi_{2,N_\Psi} \Gamma \\
\vdots & \vdots & \ddots & \vdots \\
\Psi_{V,1}\Gamma & \Psi_{V,2} \Gamma & \cdots & \Psi_{V,N_\Psi} \Gamma \end{array} \right) 
\left( \begin{array}{c}
c_1 \\
c_2 \\
\vdots \\
c_{N_\Psi N_\Gamma} \end{array} \right)  = \Phi c .
\end{equation}
Figure~\ref{fig:kronbasis} illustrates the Kronecker structure of spatial-angular atom $\Phi_k$.  We can see that by representing a dMRI signal with this type of global spatial-angular atom, one can model an entire region of the brain with as few as a single atom instead of angular atoms at every voxel.


A motivating model for this separable structure for dMRI is as follows:  first, as is traditionally done, the signal at each voxel $v \in \Omega$ is written as a linear combination of angular basis functions $\{\Gamma_i(q)\}$:
\begin{equation} 
\label{eq:AngularLinComb}
\mathcal{S}(v,q) =  \sum_{i=1}^{N_\Gamma} a_i(v) \Gamma_i(q).
\end{equation}
Then, we notice that each spherical coefficient $a_i(v)$ forms a 3D volume and so can be written as a linear combination of spatial basis functions $\{\Psi_j(v)\}$:
\begin{equation} 
\label{eq:SpatialLinComb}
a_i(v) = \sum_{j=1}^{N_\Psi}c_{i,j} \Psi_j(v).
\end{equation}
Combining \eqref{eq:AngularLinComb} and \eqref{eq:SpatialLinComb} we arrive at our proposed separable spatial-angular dictionary
\begin{equation}
\label{eq:SpatialAngularLinComb}
\mathcal{S}(v,q) = \sum_{i=1}^{N_\Gamma}\sum_{j=1}^{N_\Psi} c_{i,j} \Psi_j(v)\Gamma_i(q),
\end{equation}
When stacking each $s_v$ in a large vector, \eqref{eq:SpatialAngularLinComb} results in the Kronecker product in  \eqref{eq:SpatialAngularLinCombVec}, $s = (\Psi \otimes \Gamma)c$.  Alternatively, when writing $S = [s_1,\dots,s_V]$ as a matrix, \eqref{eq:SpatialAngularLinComb} results in the equivalent matrix form:
\begin{equation}
\label{eq:SpatialAngularMatrixForm}
S = \Gamma C \Psi^\top.
\end{equation}
Table~\ref{table:dims} summaries the dimensions of the vector and matrix variables and Figure~\ref{fig:VecVsMat} illustrates the Kronecker decompositions in the vector and matrix forms.
\begin{table}
\center
\begin{tabular}{|c|c|c|c|c|c|c|c|}
\hline
 & \multicolumn{2}{c|}{Signal} & \multicolumn{2}{c|}{Coefficients} & \multicolumn{3}{c|}{Dictionaries}\\
 \hline
Variable & $s$ & $S$ & $c$ & $C$ & $\Phi$ & $\Gamma$ & $\Psi$ \\
\hline
Dimensions & $GV$ & $G\times V$ & $N_\Gamma N_\Psi$ & $N_\Gamma \times N_\Psi$ & $GV\times N_\Gamma N_\Psi$ & $G\times N_\Gamma$ & $V\times N_\Psi$ \\ 
\hline 
\end{tabular}
\caption{Sparse coding variable dimensions, where $G$ ($\approx 100$) is the number of gradient directions in $q$-space, $V$ ($\approx 100^3$) is the number of voxels in the volume, $N_\Gamma$ ($\gtrsim 100$) is the number of atoms of the angular dictionary $\Gamma$, and $N_\Psi$ ($\gtrsim 100^3$) is the number of atoms of the spatial dictionary $\Psi$.}
\label{table:dims}
\end{table}

Decomposing signals into Kronecker (or more general multi-tensor) structures has been well researched to increase algorithmic efficiency by reducing computations to the smaller, separate domains.  Many research groups have exploited properties of the Kronecker product, when solving problem types of the form of \eqref{eq:l0vec} and \eqref{eq:l1vec} for computational efficiency of larger sparse coding \citep{Caiafa:NC13}, dictionary learning \citep{Hawe:CVPR13} and CS \citep{Duarte:TIP12} applications.  The work of \citep{Caiafa:arXiv15} has applied multi-tensor sparse coding methods on dMRI data for the application of fiber tract data compression. 
In particular, a Kronecker Orthogonal Matching Pursuit (Kron-OMP) \citep{Rivenson:SPL09} has been utilized to solve \eqref{eq:l0vec}.  Although Kron-OMP becomes much more efficient than the classical OMP \citep{Tropp:TIT04}, the problem is not entirely separated into smaller domains, and the computationally expensive $\Phi$ matrix is still built explicitly.  For large-scale problems like that of dMRI reconstruction, solving \eqref{eq:l0vec} or \eqref{eq:l1vec} even with a Kronecker structure dictionary remains largely intractable/expensive for memory and computation time.

\begin{figure}[H]
    \centering
    \includegraphics[width=.6\linewidth]{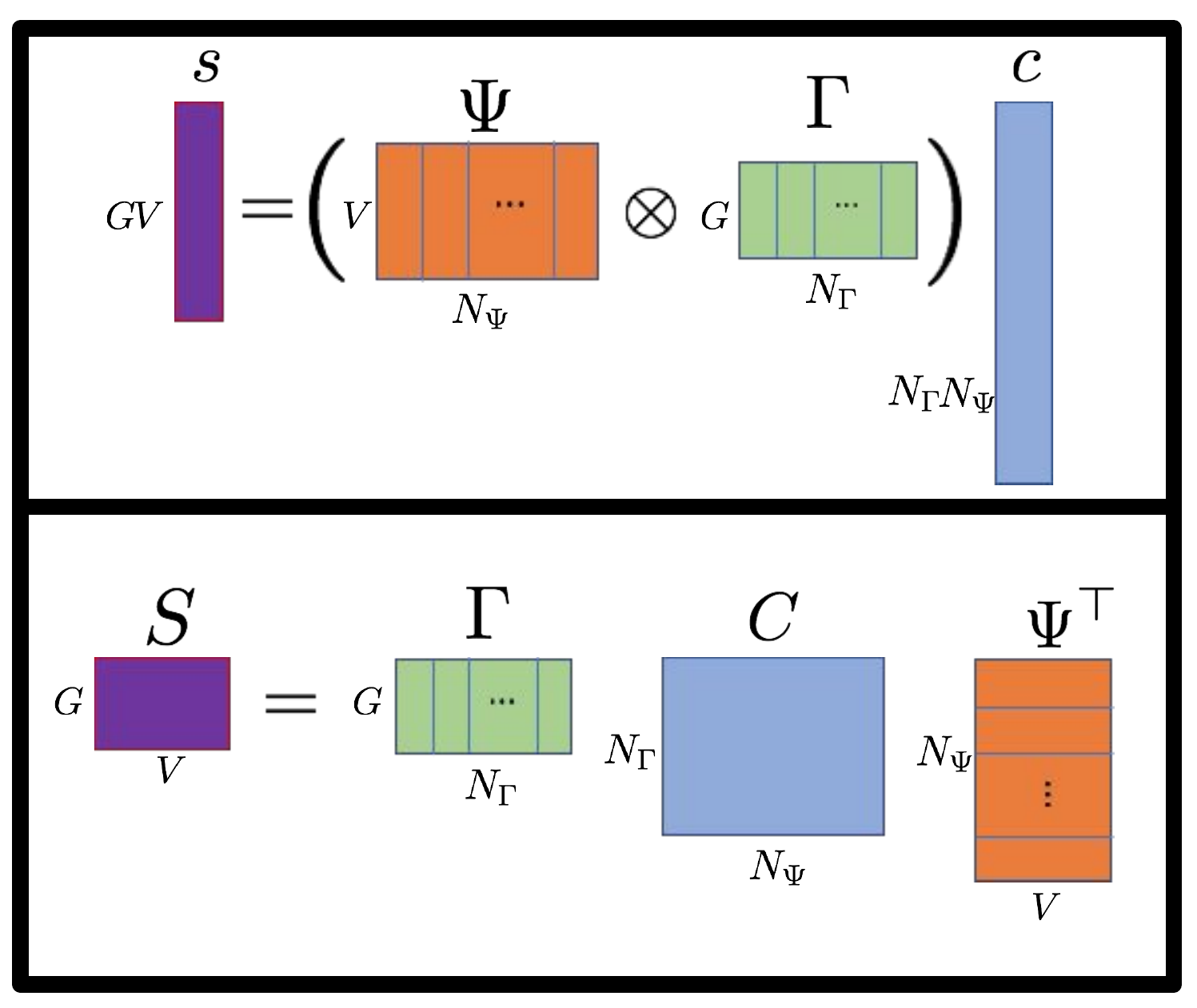}
    \caption{Equivalent vector form (top) and matrix form (bottom) for the Kronecker decomposition of a signal. We propose to use the matrix form which provides a more compact representation for signals of large size and exploits the full separability of the Kronecker product, reducing matrix multiplication complexity from $O(GVN_\Gamma N_\Psi)$ to $O(GVN_\Gamma)$.}
    \label{fig:VecVsMat}
\end{figure}

In this chapter, we use the matrix form \eqref{eq:SpatialAngularMatrixForm} which allows us to avoid the expensive uses of $\Phi$ and fully reduce computational complexity to the smaller separable basis domains of $\Gamma$ and $\Psi$.  In particular, we develop efficient algorithms to solve the completely separable spatial-angular sparse coding problems:
\begin{equation}
\tag{$P0mat$}
\label{eq:l0_separable}
C^* = \argmin_C \frac{1}{2}||\Gamma C \Psi^\top - S||_F^2 \ \ \textnormal{s.t.} \ \ ||C||_0 \leq K
\end{equation}  
and
\begin{equation}
\tag{$P1mat$}
\label{eq:l1_separable}
C^* = \argmin_C \frac{1}{2}||\Gamma C \Psi^\top - S||_F^2 + \lambda||C||_1.
\end{equation}
This becomes a general optimization to solve large-scale sparse coding problems with separable dictionaries and can also be extended to the tensor setting.

As an important note, this matrix formulation is a generalization of the voxel-wise angular sparse coding problem \eqref{eq:VoxReconMult} in the special case of $\Psi =$ I$_V$, the $V\times V$ identity matrix, with $C\equiv A$.  We use the identity as a choice for $\Psi$ in the experiments of Section~\ref{sec:experiments} when comparing the performance of purely angular sparse coding with our proposed framework\footnote{Using $\Psi =$ I$_V$ identity with spherical ridgelets (SR) we adopt the notation I-SR for the dictionary used in the state-of-the-art illustration Figure~\ref{fig:IDanalysis} and Section~\ref{sec:experiments}.}. Note that the optimization problem \eqref{eq:l1_separable} is on the coefficient matrix $C$ of size $N_\Gamma \times N_\Psi$ in comparison to the matrix $A$ of size $N_\Gamma \times V$ of state-of-the-art sparse coding \eqref{eq:VoxReconMult}. This leads to a slight increase in the dimension of the problem proportional to the redundancy factor of the spatial dictionary $\Psi$ (i.e the ratio $N_\Psi / V$). On the other hand, our formulation only involves a single sparsity penalty in comparison to a sum of angular and spatial terms, thus reducing the number of weighting parameters to select.

\section{Efficient Kronecker Sparse Coding Algorithms}
\label{sec:reconstruction}
In what follows we present three novel adaptations of existing sparse coding algorithms for solving large-scale sparse coding problems with a Kronecker dictionary structure. These are Kronecker extensions of OMP (Section~\ref{sec:KronOMP-PGD}), ADMM (Section~\ref{sec:KronADMM}), and FISTA (Section~\ref{sec:KronFISTA}). We compare these to existing Kronecker sparse coding algorithms, a Kronecker OMP \citep{Rivenson:SPL09} (Section~\ref{sec:KronOMP}) as well as Kronecker Dual ADMM (Section~\ref{sec:KronDADMM}), developed in our prior work \citep{Schwab:MICCAI16} and derived in a new formulation for comparison in this paper. 
We compare these algorithms in terms of complexity for various types of bases in Section \ref{sec:complexity} and show experimental time comparisons in Section \ref{sec:experiments}.

\subsection{Kronecker OMP}
\label{sec:KronOMP}

To approximate a solution to the $L_0$ problem \eqref{eq:l0vec}, Orthogonal Matching Pursuit (OMP) \citep{Tropp:TIT04} is a popular greedy algorithm that iteratively selects the atom that is most correlated with the signal, orthogonalizes it to the previously selected atoms by solving a least squares optimization, and selects the next atom that is most correlated with the resulting residual. For the case of a Kronecker structured basis, a Kronecker OMP (Kron-OMP) algorithm has been previously proposed \citep{Rivenson:SPL09,Caiafa:NC13} that reduces computations of solving the least squares subproblem in each iteration by exploiting properties of the Kronecker product.  This form of Kron-OMP, however, represents the signal, coefficients, and basis atoms in vector form providing a solution to \eqref{eq:l0vec}. 

In Algorithm~\ref{alg:KronOMP} we rewrite the Kron-OMP algorithm adapted to the structure of our problem, where $vec(\cdot)$ and $mat(\cdot)$ convert matrices to vectors and vice versa. The main complexity gain of Kron-OMP over the vector OMP is obtained by separating the effects of $\Gamma$ and $\Psi$ when computing the maximally correlated atoms with the residual, $|\Gamma^\top R \Psi|$ (See Algorithm~\ref{alg:KronOMP} Step 1) with complexity $O(N_\Gamma GV + GN_\Gamma N_\Psi)$ instead of computing $|\Phi^\top r|$ with complexity $O(N_\Gamma N_\Psi GV)$. 
The other gain is in solving the least squares problem (See Algorithm~\ref{alg:KronOMP} Step 3) by exploiting properties of the Kronecker product ($A \odot B = [a_1 \otimes b_1,\dots, a_N \otimes B_N])$ to compute a rank-1 update.  However, the only real improvement on complexity is in memory since $\Phi$ can be built atom by atom from columns of $\Gamma$ and $\Psi$ instead of storing the entire matrix. The rank-1 update remains $O(k^2)$ for both vector and Kron-OMP.  
In the next section we present an alternative Kron-OMP algorithm that reduces complexity further by exploiting the full separability of the dictionary.
\begin{algorithm}[H]
\caption{Kron-OMP}
\begin{algorithmic}
\label{alg:KronOMP}
\STATE Choose: $K, \epsilon.$
\STATE Initialize: $k=1, \ \mathcal{I}^0 = \emptyset, \ \mathcal{J}^0 = \emptyset, \ R_0 = S, \ s =  vec(S).$ 
\WHILE {$k \leq K$ and error $> \epsilon$}
\STATE 1: $[i^k,j^k] = \argmax_{[i,j]} |\Gamma_i^\top R_k \Psi_j | ;$
\STATE 2: $\mathcal{I}^k = [\mathcal{I}^{k-1}, \ i^k]; \mathcal{J}^k = [\mathcal{J}^{k-1}, \ j^k]; \mathcal{A}^k = (\I^k,\J^k);$ 
\STATE 3: $c_k = \argmin_c \frac{1}{2}|| (\Gamma_{\I^k} \odot \Psi_{\J^k})c - s ||_2^2;$
\STATE 4: $R_k = mat(s - (\Gamma_{\I^k} \odot \Psi_{\J^k}) c_k);$ 
\STATE 5: $k \leftarrow k+1;$
\ENDWHILE
\STATE Return: $\mathcal{A}^K, c_K$
\end{algorithmic}
\end{algorithm}

\subsection{Kronecker OMP with Projected Gradient Descent (PGD)}
\label{sec:KronOMP-PGD}
In what follows, we develop a novel form of Kronecker OMP which solves the separable \eqref{eq:l0_separable} instead of \eqref{eq:l0vec}.  This allows us to reduce computation by not building columns of $\Phi$ and not repeating individual atoms of $\Gamma$ or $\Psi$.   Instead, indices of $\Gamma$ and $\Psi$ are updated only when they each have not been chosen before, fully exploiting the separability of the dictionary.   
Given the previous sets of respective of indices $\mathcal{I}^{k-1}$ and $\mathcal{J}^{k-1}$, we update sets by following $\mathcal{I}^k = [\mathcal{I}^{k-1} \ i^k]$ if $i^k \not\in \mathcal{I}^{k-1}$ and $\mathcal{I}^k = \mathcal{I}^{k-1}$ otherwise. Likewise, $\mathcal{J}^k = [\mathcal{J}^{k-1} \ j^k]$ if $j^k \not\in \mathcal{J}^{k-1}$ and $\mathcal{J}^k = \mathcal{J}^{k-1}$ otherwise.  With the selected indices, the size of $C_k$ will be $|\I^k| \times |\J^k|$ instead of $k\times k$.  To find $C_k$, it seems natural to solve: 
\begin{equation}
C_k = \argmin_C \frac{1}{2}|| \Gamma_{\mathcal{I}^k} C \Psi^\top_{\mathcal{J}^k} - S||_F^2.
\end{equation}
But the solution $C_k$ will contain possible nonzero coefficients that do not coincide with the chosen selection of indices since additional indices in all combinations of pairs between $\I^k$ and $\J^k$ will be updated in each iteration.  This is problematic for the correctness of the algorithm when choosing the next single most correlated coefficient.  Therefore we must enforce that these coefficients are zero:
\begin{equation}
C_k = \argmin_C \frac{1}{2}|| \Gamma_{\mathcal{I}^k} C \Psi^\top_{\mathcal{J}^k} - S||_F^2 \ \textnormal{s.t.} \ C_{i,j} = 0 \ \forall (i,j) \in \mathcal{O}^k.
\end{equation}
where $\mathcal{O}^k := \overline{({\I}^k,{\J}^k)}$. To solve this problem, we can use projected gradient descent (PGD).  The gradient of $f(C) = \frac{1}{2}|| \Gamma_{\mathcal{I}^k} C \Psi^\top_{\mathcal{J}^k} - S||_F^2$ at iteration $k$ is
\begin{equation}
\nabla f(C) = \Gamma_{\I^k}^\top\Gamma_{\I^k} C \Psi_{\J^k}^\top \Psi_{\J^k} - \Gamma_{\I^k}^\top S \Psi_{\J^k}.
\end{equation}
To save on computation we precompute $\mathcal{G} = \Gamma^\top\Gamma$, $\mathcal{P} = \Psi^\top \Psi$, and $\hat{S} =\Gamma^\top S \Psi$ and can access their entries at each iteration: $\mathcal{G}_{\I^k,\I^k}, \mathcal{P}_{\J^k,\J^k}, \hat{S}_{\I^k,\J^k}$.
Then setting $Z^1 = C_k$ we iteratively project the update in the gradient direction to the space of feasible solutions:
\begin{equation}
Z^{t+1} = P_{\mathcal{O}^k}(Z^t - \epsilon \nabla f(Z^t) ),
\end{equation}
where the projection $P_{\mathcal{O}^k}$ sets all elements in $\mathcal{O}^k$ to $0$ and step-size $\epsilon$ is estimated each iteration using a line search. Once the procedure has converged to $Z^*$, we set $C_k = Z^*$ and compute the residual 
$R_k = S - \Gamma_{\I^k} C_k \Psi^\top_{\J^k}$.
Then, for iteration $k+1$ we must find $(i^{k+1},j^{k+1}) = \argmax_{[i,j]} |\Gamma_i^\top R_k \Psi_j|$.  To save significantly on computation we can instead use our precomputed $\mathcal{G}$ and $\mathcal{P}$ to instead find $\argmax_{[i,j]} [\hat{R}_k]_{i,j}$, where $\hat{R}_k = |\hat{S} - \mathcal{G}_{\I^k} C_k \mathcal{P}^\top_{\J^k}|$ where $\mathcal{G}_{\I^k},\mathcal{P}_{\J^k}$ respectively access the $\I^k, \J^k$ columns and all rows. Maintaining matrix forms throughout allows us to combine computing the residual and the next atoms for a large reduction in computation at each iteration $k$.
Our proposed Kronecker OMP with projected gradient descent (Kron-OMP-PGD) is outlined in Algorithm~\ref{alg:KronOMP-PGD}. 
\begin{figure}
\centering
\includegraphics[width=.7\linewidth,trim=20 0 20 0,clip]{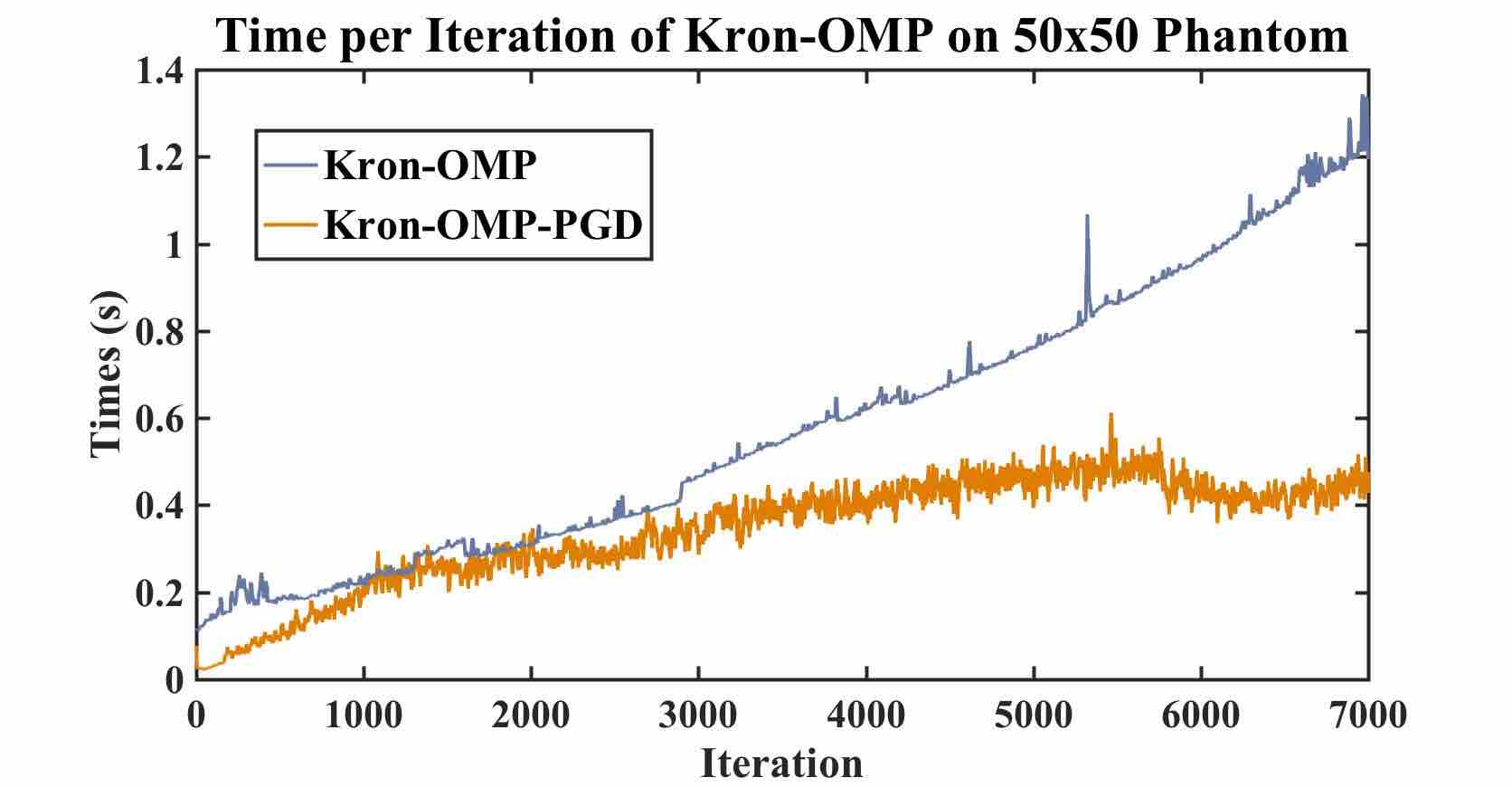}
\caption{Comparison of time per iteration for Kron-OMP and the proposed Kron-OMP-PGD.  The total time to choose $K=7000=2.8V$ atoms for this $V=50\times 50$ slice of a phantom dataset, is 68 min for Kron-OMP and 40 min for Kron-OMP-PGD. We can see that as the number of atoms grows, the time per iteration of Kron-OMP continues to grow at a much higher rate than Kron-OMP-PGD.}
\label{fig:OMPtimeComp}
\end{figure}

We show a comparison of time per iteration for a small $V=50\times50, G=64$ phantom dataset in Figure~\ref{fig:OMPtimeComp}. The steeper time increase for Kron-OMP is due to the fact that at iteration $k$ there is a complexity of $O(k^2 + kGV)$ that comes from Steps 3 (rank-1 update) and 4 of Algorithm~\ref{alg:KronOMP}.  On the other hand, Kron-OMP-PGD has complexity involving $|\I^k|, |\J^k| \leq k$ which are in practice significantly less than $k$.  Even though a PGD sub-routine must be performed at each iteration $k$, we found that by incorporating Nesterov acceleration with a line search, the time per iteration remains lower than Kron-OMP as the number of iterations $k$ increases. 

However, for dMRI data, typical sparsity levels are $K=O(V)$. So for $V\approx 100^3$ the number of iterations as well as the time per iteration of both Kron-OMP and Kron-OMP-PGD when $k$ approaches $K$ becomes astronomical.  Even on a relatively small 3D phantom dataset of spatial size $V = 50\times 50$, for example, one iteration takes on the order of a few seconds which results in over 34 hrs for these greedy algorithms to reach 1 atom/voxel atoms $(K=V)$. 
In this way, greedy algorithms such as OMP are not suitable for large-scale problems that require hundreds of thousands of iterations.  Instead, optimizing the LASSO problem \eqref{eq:l1_separable} can be accomplished with significantly less iterations, as we examine in the following section.

\begin{algorithm}[H]
\caption{Kron-OMP-PGD}
\begin{algorithmic}
\label{alg:KronOMP-PGD}
\STATE Choose: $K, \epsilon_1, \epsilon_2.$
\STATE Precompute: $\hat{S} = \Gamma^\top S \Psi, \ \mathcal{G} = \Gamma^\top \Gamma, \ \mathcal{P} = \Psi^\top \Psi.$
\STATE Initialize: $k=1, \ C_0=1, \ \mathcal{I}^0 = \emptyset, \ \mathcal{J}^0 = \emptyset, \ \hat{R}_0 = \hat{S}$. 
\WHILE {$k \leq K$ and error $> \epsilon_1$}
\STATE 1: $[i^k,j^k] = \argmax_{[i,j]} [\hat{R}_k]_{i,j} ;$
\STATE 2: $\mathcal{I}^k = \mathcal{I}^{k-1} \cup \{i^k\}; \mathcal{J}^k = \mathcal{J}^{k-1} \cup \{j^k\}; \mathcal{A}^k = (\I^k,\J^k); \mathcal{O}^k = \overline{\mathcal{A}^k};$
\STATE 3: $Z^1_{\J^{k-1},\I^{k-1}} = C_{k-1}; \  n_1 = 0; \ t=1;$
\WHILE {error $> \epsilon_2$}
\STATE 1: $\delta = \textnormal{linesearch}(Z^{t});$
\STATE 2: $X^{t+1} = P_{\mathcal{O}^k}(Z^{t} - \delta (\mathcal{G}_{\I^k,\I^k} Z^{t} \mathcal{P}_{\J^k,\J^k} - \hat{S}_{\I^k,\J^k}))$;
\STATE 4: $n_{t+1} = \frac{1}{2}(1 + \sqrt{1 + 4n_{t}^2});$
\STATE 5: $Z^{t+1} = X^{t+1} + \frac{n_{t}-1}{n_{t+1}}(X^{t+1} - X^{t});$
\STATE 6: $t\leftarrow t+1$;
\ENDWHILE
\STATE 4: $C_k = Z^*;$
\STATE 5: $\hat{R}_k = |\hat{S} - \mathcal{G}_{\mathcal{I}^k} C_k \mathcal{P}_{\mathcal{J}^k}|; $
\STATE 6: $k \leftarrow k+1;$
\ENDWHILE
\STATE Return: $\mathcal{A}^K, C_K$.
\end{algorithmic}
\end{algorithm}

\subsection{Kronecker ADMM}
\label{sec:KronADMM}
The Alternating Direction Method of Multipliers (ADMM) \citep{Boyd:FTML10} is a popular method for solving the LASSO problem \eqref{eq:l1vec}. However, its application in the case of a large dictionary $\Phi$ remains prohibitive, requiring computations involving $\Phi^\top s$ of order $O(GVN_\Gamma N_\Psi)$.  Instead, we apply ADMM to the separable LASSO problem \eqref{eq:l1_separable} to reduce computations by solving
\begin{equation}
\label{eq:Kron_l1_SeparableADMM}
\min_{C,Z} \frac{1}{2}|| \Gamma C \Psi^\top - S||_F^2 + \lambda||Z||_1 \ \ \textnormal{s.t.} \ \ C = Z.
\end{equation}
The augmented Lagrangian writes:
\begin{equation}
\label{eq:augLag}
\mathcal{L}_\mu(C,Z,\mathcal{T}) = \frac{1}{2} ||\Gamma C \Psi^\top -S||_F^2 + \lambda||Z||_1 + <\mathcal{T}, C-Z> +\frac{\mu}{2} || C-Z ||_F^2,
\end{equation}
and:
\begin{align}
\frac{\partial \mathcal{L}_\mu}{\partial C} &= \Gamma^\top (\Gamma C \Psi^\top - S) \Psi + \mathcal{T} + \mu(C-Z) = 0\\
&\implies \Gamma^\top \Gamma C \Psi^\top \Psi + \mu C = \mu Z - \mathcal{T} + \Gamma^\top S \Psi := Q.
\end{align}

In principle, one can solve for $C$ by solving a linear system of equations $h(C) = Q$, where $h(C) = \Gamma^\top \Gamma C \Psi^\top \Psi + \mu C$. However, solving this linear system directly is computationally challenging due to the size of the matrices involved. Therefore, to solve for $C$ efficiently, we begin by taking the SVDs of $\Gamma$ and $\Psi$. With $\Gamma = U_\Gamma \Sigma_\Gamma V_\Gamma^\top$ and $\Psi = U_\Psi \Sigma_\Psi V_\Psi^\top$, $\Gamma^\top \Gamma = V_\Gamma \Delta_\Gamma V_\Gamma^\top$ and $\Psi^\top \Psi = V_\Psi \Delta_\Psi V_\Psi^\top$, where $U_\Gamma, U_\Psi$ are the matrices of eigenvectors and $\Delta_\Gamma = \Sigma_\Gamma^\top \Sigma_\Gamma, \Delta_\Psi = \Sigma_\Psi^\top \Sigma_\Psi$ are the diagonal matrices of eigenvalues for $\Gamma$ and $\Psi$ respectively.  Then:
\begin{align}
V_\Gamma \Delta_\Gamma V_\Gamma^\top C V_\Psi \Delta_\Psi V_\Psi^\top + \mu C &= Q\\
\implies \Delta_\Gamma \tilde{C} \Delta_\Psi +\mu \tilde{C} &= \tilde{Q}
\label{eq:Cupdate}
\end{align}
where we introduced the notation $\tilde{X} = V_\Gamma^\top X V_\Psi$. Since $\Delta_\Gamma$ and $\Delta_\Psi$ are diagonal with elements $\delta_{\Gamma_i}$ and $\delta_{\Psi_j}$, respectively, we can solve for $\tilde{C}$ by:
\begin{equation}
\delta_{\Gamma_i} \tilde{C}_{i,j} \delta_{\Psi_j} + \mu \tilde{C}_{i,j} = \tilde{Q}_{i,j} \implies \tilde{C}_{i,j} = \frac{\tilde{Q}_{i,j}}{\delta_{\Gamma_i} \delta_{\Psi_j} + \mu}
\end{equation}
To write this in matrix form we define $[\Delta_\mu]_{i,j}\!\triangleq\!1/(\delta_{\Gamma_i} \delta_{\Psi_j} + \mu)$ and have
$\tilde{C} = (\Delta_\mu \circ \tilde{Q})$
where $\circ$ stands for element-wise matrix multiplication.  Finally, we can recover $C = V_\Gamma \tilde{C} V_\Psi^\top$ and the complete update for $C$ is:
\begin{equation}
\label{eq:C_under}
C_{k+1} = V_\Gamma (\Delta_\mu \circ ( V_\Gamma^\top Q_k V_\Psi)) V_\Psi^\top
\end{equation}
where $Q = \mu Z - \mathcal{T} + \Gamma^\top S \Psi$.  When minimizing $\mathcal{L}_\mu$ with respect to $Z$, we end up with the usual proximal operator of the $L_1$ norm that is given by the shrinkage operator, shrink$_\kappa(X) = \max(0,X-\kappa) - \max(0,-X-\kappa)$, applied element-wise to matrix $X$, giving $Z_{k+1} = \textnormal{shrink}_{\lambda / \mu}(C_{k+1} + \mathcal{T}_{k})$. Similarly with respect to $\mathcal{T}$, we have the usual Lagrange multiplier gradient ascent update $\mathcal{T}_{k+1} = \mathcal{T}_{k} + C_{k+1} - Z_{k+1}$.  

The formal updates for Kron-ADMM are presented in Algorithm~\ref{alg:Kron-ADMMunder}.
\begin{algorithm}
\caption{Kron-ADMM (for undercomplete dictionaries)}
\begin{algorithmic}
\label{alg:Kron-ADMMunder}
\STATE Choose: $\mu, \lambda, \epsilon.$
\STATE Precompute: $V_\Gamma, \Delta_\Gamma, V_\Psi, \Delta_\Psi, \Delta_\mu$.
\STATE Initialize: $k=0, Z_0 = \mathbf{0}, \mathcal{T}_0 = \mathbf{0}, \hat{S} = \Gamma^\top S \Psi$. 
\WHILE {error $> \epsilon$}
\STATE 1: $Q_k = \hat{S} + \mu Z_k - \mathcal{T}_k$;\\
\STATE 2: $C_{k+1} = V_\Gamma (\Delta_\mu \circ ( V_\Gamma^\top Q_k V_\Psi)) V_\Psi^\top$;\\
\STATE 3: $Z_{k+1} = \textnormal{shrink}_{\lambda / \mu}(C_{k+1} + \mathcal{T}_{k})$;\\
\STATE 4: $\mathcal{T}_{k+1} = \mathcal{T}_{k} + C_{k+1} - Z_{k+1};$
\STATE 5: $k \leftarrow k+1;$
\ENDWHILE
\STATE Return: $C$.
\end{algorithmic}
\end{algorithm}
The update for $C$ in \eqref{eq:C_under} works well when $\Gamma$ and $\Psi$ are under-complete and the eigen-decompositions of $\Gamma^\top \Gamma$ and $\Psi^\top \Psi$ are easily computable.  However, dictionaries most commonly used for sparse coding and the application to CS are over-complete \ie\ $G < N_\G$ and $V < N_\P$ making these SVDs potentially expensive to compute.  In the case of an over-complete $\Phi$, for traditional vector ADMM, the matrix inversion lemma \citep{Boyd:FTML10} is involved in order to compute SVDs of the smaller $\Phi \Phi^\top$ instead of $\Phi^\top \Phi$. In the following proposition, we derive the equivalent result for the update of $C$ in \eqref{eq:C_under}.
\begin{prop}
For over-complete dictionaries $\Gamma$ and $\Psi$, update \eqref{eq:C_under} is equivalent to the more compact 
\begin{equation}
\label{eq:C_over}
C = Q/\mu - \G^\top U_\G (\Delta_\mu \circ (U^\top_\G \G Q \P^\top U_\Psi)) U_\Psi^\top \P /\mu.
\end{equation}
\end{prop}
\begin{proof}
For over-complete dictionaries $\Gamma = U_\Gamma [\Sigma_\Gamma, \  \mathbf{0}] V_\Gamma^\top$ and $\Psi = U_\Psi [\Sigma_\Psi, \ \mathbf{0}] V_\Psi^\top$, 
\begin{equation*}
\Gamma^\top \Gamma = V_\Gamma \left(  \begin{array}{cc}
\Delta_\Gamma & \mathbf{0} \\
\mathbf{0} & \mathbf{0}
\end{array} \right) V_\Gamma^\top \ \textnormal{and} \ \Psi^\top \Psi = V_\Psi \left( \begin{array}{cc} 
\Delta_\Psi & \mathbf{0}\\
\mathbf{0} & \mathbf{0} 
\end{array} \right) V_\Psi^\top.
\end{equation*}
For $G< i \leq N_\G, V< j \leq N_\P$, $\delta_{\G_i}, \delta_{\P_j} = 0$, so $\tilde{C}_{i,j} = \frac{\tilde{Q}_{i,j}}{\delta_{\Gamma_i} \delta_{\Psi_j} + \mu} = \frac{\tilde{Q}_{i,j}}{\mu}$.  For $i\leq G$ and $j\leq V$, we can rewrite 
 \begin{align*}
 \tilde{C}_{i,j} = \frac{\tilde{Q}_{i,j}}{\delta_{\Gamma_i} \delta_{\Psi_j} + \mu} &= \frac{\tilde{Q}_{i,j}}{\mu} - \frac{\delta_{\Gamma_i} \tilde{Q}_{i,j} \delta_{\Psi_j}}{\mu(\delta_{\Gamma_i} \delta_{\Psi_j} + \mu)} = \frac{\tilde{Q}_{i,j}}{\mu} - \frac{\sigma_{\Gamma_i}^2 \tilde{Q}_{i,j} \sigma_{\Psi_j}^2}{\mu(\delta_{\Gamma_i} \delta_{\Psi_j} + \mu)}\\
 &= \frac{\tilde{Q}_{i,j}}{\mu} - \sigma_{\Gamma_i} \frac{\sigma_{\Gamma_i}\tilde{Q}_{i,j} \sigma_{\Psi_j}}{\mu(\delta_{\Gamma_i} \delta_{\Psi_j} + \mu)}\sigma_{\Psi_j}\\
 \implies \tilde{C} & = \tilde{Q}/\mu - \Sigma_\Gamma^\top (\Delta_\mu \circ (\Sigma_\Gamma \tilde{Q} \Sigma_\Psi^\top))\Sigma_\Psi /\mu\\
C &= Q/\mu - V_\G \Sigma_\Gamma^\top (\Delta_\mu \circ (\Sigma_\Gamma V_\G^\top Q V_\P \Sigma_\Psi^\top))\Sigma_\Psi V_\P^\top/\mu\\
C &= Q/\mu - \G^\top U_\G (\Delta_\mu \circ (U^\top_\G \G Q \P^\top U_\Psi )) U_\Psi^\top \P /\mu.
\end{align*}
\end{proof}
Letting $\Gamma' = U_\Gamma^\top \Gamma$ and $\Psi' = U_\Psi^\top \Psi$, which can be precomputed, we have a final efficient update
\begin{equation}
C_{k+1} = Q_k/\mu - \G'^\top (\Delta_\mu \circ (\G' Q_k \P'^\top)) \P' /\mu.
\end{equation}
This allows us to compute the SVDs of $\Gamma \Gamma^\top$ and $\Psi \Psi^\top$ instead of the larger $\Gamma^\top \Gamma$ and $\Psi^\top \Psi$ and work with smaller matrices within each iteration.  We present Kron-ADMM for over-complete dictionaries in Algorithm~\ref{alg:Kron-ADMMover}.
\begin{algorithm}[H]
\caption{Kron-ADMM (for overcomplete dictionaries)}
\begin{algorithmic}
\label{alg:Kron-ADMMover}
\STATE Choose: $\mu, \lambda, \epsilon.$
\STATE Precompute: $U_\Gamma, \Delta_\Gamma, U_\Psi, \Delta_\Psi, \Gamma', \Psi', \Delta_\mu$.
\STATE Initialize: $k=0, Z_0 = \mathbf{0}, \mathcal{T}_0 = \mathbf{0}$. 
\WHILE {error $> \epsilon$}
\STATE 1: $Q_k = \Gamma^\top S \Psi + \mu Z_k - \mathcal{T}_k$;\\
\STATE 2: $C_{k+1} = Q_k/\mu - \Gamma'^\top (\Delta_\mu \circ (\Gamma' Q_k \Psi'^\top)) \Psi'/\mu$;\\
\STATE 3: $Z_{k+1} = \textnormal{shrink}_{\lambda / \mu}(C_{k+1} + \mathcal{T}_{k})$;\\
\STATE 4: $\mathcal{T}_{k+1} = \mathcal{T}_{k} + C_{k+1} - Z_{k+1}$;
\STATE 5: $k \leftarrow k+1$;
\ENDWHILE
\STATE Return: $C$.
\end{algorithmic}
\end{algorithm}
\subsection{Kronecker Dual ADMM}
\label{sec:KronDADMM}
As an alternative to ADMM, Dual ADMM, which applies ADMM to the dual of $l1$ problem \eqref{eq:l1vec}, has been shown to be more efficient than ADMM for over-complete dictionaries \citep{Gonccalves:Thesis15} by allowing one to compute SVDs of the more affordable $\Phi \Phi^\top$ instead of $\Phi^\top \Phi$.  In our previous work \citep{Schwab:MICCAI16} we proposed a Kronecker Dual ADMM (Kron-DADMM) that efficiently solves the spatial-angular sparse coding problem. Below, we give an alternative derivation of this algorithm directly based on the matrix formulation of \eqref{eq:l1_separable}.  The dual of \eqref{eq:l1_separable} is:
\begin{equation}
\label{eq:Kron_l1dual}
\max_A -\frac{1}{2}||A||^2_F + A^\top S \ \ \textnormal{s.t.}\ \ ||\Gamma^\top A \Psi||_\infty \leq \lambda,
\end{equation}
where $||X||_\infty = \max_{i,j} |X_{i,j}|$. To apply ADMM to this optimization problem, we replace $\Gamma^\top A \Psi$ with auxiliary variable $\mathcal{V}$ and add the additional constraint $\Gamma^\top A \Psi - \mathcal{V} = 0$ to get:
\begin{equation}
\label{eq:Kron_l1dual_aux}
\max_{A,\mathcal{V}} -\frac{1}{2}||A||^2_F + A^\top S \ \ \textnormal{s.t.}\ \ ||\mathcal{V}||_\infty \leq \lambda \ \ \textnormal{and} \ \ \mathcal{V} = \Gamma^\top A \Psi.
\end{equation}
Then the augmented Lagrangian is
\begin{equation}
\label{eq:Kron_l1dual_aug}
\mathcal{L}_\eta(A,\mathcal{V},C) = -\frac{1}{2}||A||^2_F + A^\top S + <C, \mathcal{V} - \Gamma^\top A \Psi> + \frac{\eta}{2} ||\mathcal{V} -\Gamma^\top A \Psi ||_F^2 + \delta_\lambda(\mathcal{V})
\end{equation}
where 
\begin{equation}
\delta_\lambda(\mathcal{V}) = \begin{cases} 0 &\mbox{if } ||\mathcal{V}||_\infty \leq \lambda \\ 
\infty &\mbox{if } ||\mathcal{V}||_\infty > \lambda \end{cases}
\end{equation}
and the Lagrange multiplier $C$ corresponds to the primal variable $C$ in \eqref{eq:l1_separable}, which is our variable of interest. We then have
\begin{align}
\frac{\partial \mathcal{L}_\eta(A,\V,C)}{\partial A} &= -A + S - \Gamma C \Psi^\top - \eta \Gamma (\V - \Gamma^\top A \Psi)\Psi^\top = 0\\
&\implies A - \eta \Gamma \Gamma^\top A \Psi \Psi^\top = S - \Gamma (C + \eta \V) \Psi^\top := P.
\end{align}
Now with eigen-decompositions $\Gamma \Gamma^\top = U_\Gamma \Delta_\Gamma U_\Gamma^\top$ and $\Psi \Psi^\top = U_\Psi \Delta_\Psi U_\Psi^\top$ and letting $\tilde{X}= U_\G^\top X U_\P$,
\begin{align}
A + \eta U_\Gamma \Delta_\Gamma U_\Gamma^\top A U_\Psi \Delta_\Psi U_\Psi^\top = P\\
\implies \tilde{A} + \eta \Delta_\Gamma \tilde{A} \Delta_\Psi = \tilde{P}.
\label{eq:Aupdate}
\end{align}
Then, $\tilde{A}$ can be found element-wise by:
\begin{equation}
\tilde{A}_{i,j} + \eta \delta_{\Gamma_i} \tilde{A}_{i,j} \delta_{\Psi_j} = \tilde{P}_{i,j} \implies \tilde{A}_{i,j} = \frac{\tilde{P}_{i,j}}{1 + \eta \delta_{\Gamma_i} \delta_{\Psi_j}}.
\end{equation}
Defining $[\Delta_\eta]_{i,j} \triangleq 1/(1 + \eta \delta_{\Gamma_i} \delta_{\Psi_j})$, the update is $\tilde{A} = \Delta_\eta \circ \tilde{P}$. As shown in \citep{Gonccalves:Thesis15} we can keep the update in terms of $\tilde{A}$ instead of $A$ since the variable we are interested in is $C$.  We can then precompute $S' = \Gamma'^\top S \Psi'$, $\Gamma' = U_\Gamma^\top \Gamma$ and $\Psi' = U_\Psi^\top \Psi$.  The updates of $\mathcal{V}$ and $C$ are as in \citep{Schwab:MICCAI16} and presented in Algorithm~\ref{alg:KronDADMM}, where $P_\lambda^\infty(X)$ sets all entries of matrix $X$ that are greater than $\lambda$ to $\lambda$. 
\begin{algorithm}[H]
\caption{Kron-DADMM}
\label{alg:KronDADMM}
\begin{algorithmic}
\STATE Choose: $\eta, \lambda, \epsilon.$
\STATE Precompute: $S', \Gamma', \Psi', \Delta_\eta$.
\STATE Initialize: $k=0, C_0=0,\mathcal{V}_0=0$. 
\WHILE {Duality Gap $> \epsilon$}
\STATE 1: $\tilde{A}_{k+1} = \Delta_\eta \circ (S' - \Gamma'(C_k - \eta \mathcal{V}_k)\Psi'^\top);$
\STATE 2: $\mathcal{V}_{k+1} = P^\infty_\lambda (\frac{1}{\eta}C_k + \Gamma'^\top \tilde{A}_{k+1}\Psi');$
\STATE 3: $C_{k+1} = \shrink_{\lambda \eta}(C_k + \eta \Gamma'^\top \tilde{A}_{k+1}\Psi');$
\STATE 4: $k \leftarrow k+1;$
\ENDWHILE
\STATE Return: $C$.
\end{algorithmic}
\end{algorithm}
%
\subsection{Kronecker FISTA}
\label{sec:KronFISTA}
The Fast Iterative Thresholding Algorithm (FISTA) \citep{Beck2009} is another well-known method for solving LASSO.  However, just as before, applying FISTA to \eqref{eq:l1vec} for large-scale dMRI data is largely intractable.  So here we adapt FISTA to \eqref{eq:l1_separable} in order to exploit the separability of our spatial-angular basis. 
FISTA is a proximal gradient descent
\begin{equation}
C_{k+1} = \textnormal{shrink}_{\lambda/L}(C_k - \nabla f(C_k) /L), 
\end{equation}
where the proximal operator is the soft-thresholding shrinkage operator associated with the $l1$ norm and $1/L$ is a chosen step size.  The gradient is simply computed as:
\begin{equation}
 \nabla f(C) =  \Gamma^\top (\Gamma C \Psi^\top) \Psi - \Gamma^\top S \Psi.
 \end{equation}
To help speed convergence, we use a line search subroutine to update $L$ at each iteration in addition to the usual Nesterov acceleration.  By \citep{Beck2009}, FISTA will converge for any $L$ greater than the Lipschitz constant of $\nabla f$, which can be estimated by bounding 
\begin{equation}
||\nabla f(C) - \nabla f(\bar{C})||_F = ||\Gamma^\top\Gamma (C - \bar{C}) \Psi^\top \Psi||_F \leq \lambda^\Gamma_{\textnormal{max}} \lambda^\Psi_{\textnormal{max}} ||C - \bar{C}||_F
\end{equation}
where $\lambda^\Gamma_{\textnormal{max}}$ and $\lambda^\Psi_{\textnormal{max}}$ are the maximum eigenvalues of $\Gamma^\top \Gamma$ and $\Psi^\top \Psi$ respectively.  Therefore we initialize $L= \lambda^\Gamma_{\textnormal{max}} \lambda^\Psi_{\textnormal{max}}$.  The Kronecker FISTA (Kron-FISTA) is presented in Algorithm~\ref{alg:Kron-FISTA}. This natural Kronecker extension to FISTA has also been recently presented in \citep{Qi:CVPR16}, but has not been adapted and tested on data of our scale.
\begin{algorithm}[H]
\caption{Kron-FISTA}
\begin{algorithmic}
\label{alg:Kron-FISTA}
\STATE Choose: $\epsilon$.
\STATE Precompute: $\hat{S} = \Gamma^\top S \Psi$
\STATE Initialize: $Z_1 = C_0 = \mathbf{0}$, $n_1 = 1, L = \lambda_{\textnormal{max}}^\Gamma \lambda_{\textnormal{max}}^\Psi$.
\WHILE {error $> \epsilon$}
\STATE 1: $L = \textnormal{linesearch}(Z_k);$
\STATE 2: $\nabla f(Z_k) = \Gamma^\top (\Gamma Z_k \Psi^\top) \Psi - \hat{S};$
\STATE 3: $C_k = \textnormal{shrink}_{\lambda/L}(Z_k - \nabla f(Z_k) /L);$
\STATE 4: $n_{k+1} = \frac{1}{2}(1 + \sqrt{1 + 4n_k^2});$
\STATE 5: $Z_{k+1} = C_{k+1} + \frac{n_k-1}{n_{k+1}}(C_{k+1} - C_{k});$
\STATE 6: $k \leftarrow k+1;$
\ENDWHILE
\STATE Return: $C$.
\end{algorithmic}
\end{algorithm}
\subsection{Complexity Analysis}
\label{sec:complexity}
To evaluate the efficiency of each algorithm and the gains of Kronecker separability compared to the original algorithms we summarize the complexity of each algorithm for general $\Psi$ and $\Gamma$ in Table~\ref{table:complexity}. We notice that classical LASSO algorithms have complexity on the order of the size of the $\Phi$ matrix, including terms that multiply all four dimensions $GVN_\Gamma N_\Psi$. When applying the Kronecker LASSO algorithms, the complexity is reduced to a summation that includes only 3 of the dimensions $GVN_\Psi$, a reduction on the order of $N_\Gamma$ ($\approx 200$ for some of our dictionary choices). We compare the Kronecker LASSO algorithms empirically in Section~\ref{sec:experiments} to identify which is fastest for our regime.  Next we address the fact that the dimensions of $\Gamma \in \mathbb{R}^{G\times N_\Gamma}$ and $\Psi \in \mathbb{R}^{V\times N_\Psi}$ will be orders of magnitude different since $G \approx 100$ and $V \approx 100^3$.  We consider a few specific assumptions on the structure of spatial dictionary $\Psi$ which can decrease the complexity and simplify computations of some of the proposed algorithms:
\begin{table}
\center
\footnotesize{
    \begin{tabular}{ | c | c | c | }
    \hline
    Algorithm & Standard & Kronecker \\ \hline
    OMP & $k^2 + kGV + GVN_\Gamma N_\Psi$ & $k^2 + kGV + GVN_\Gamma  + VN_\Gamma N_\Psi$ \\ \hline
    OMP-PGD & -- & $TG|\I^k||\J^k| + TGV|\J^k| + |\J^k| N_\Gamma N_\Psi $ \\ \hline
    ADMM & $(GV)^2N_\Gamma N_\Psi  + GV(N_\Gamma N_\Psi)^2$  & $(G N_\Gamma N_\Psi + GVN_\Psi) + GV$ \\ \hline
    DADMM & $(GV)^2N_\Gamma N_\Psi $ & $(G N_\Gamma N_\Psi + GVN_\Psi) + GV$ \\ \hline
    FISTA & $(N_\Gamma N_\Psi)^2 +GVN_\Gamma N_\Psi$ & $(G N_\Gamma N_\Psi + GVN_\Psi)$ \\    
    \hline
    \end{tabular}}
    \caption{Comparison of algorithms complexity at iteration $k$. For Kron-OMP-PGD, $T$ is the number of sub-iterations of PGD.}
    \label{table:complexity}
\end{table}

\myparagraph{$\Psi$ Tight Frame}
In the case that $\Psi$ is a tight frame, $\ie \ \Psi \Psi^\top =$ I, which is commonly an assumption in compressed sensing theorems, our method can still be simplified.  In Kron-ADMM (overcomplete) and Kron-DADMM, we may avoid the SVD of $\Psi\Psi^\top$ and respective updates \eqref{eq:Cupdate} and \eqref{eq:Aupdate} can be simplified.


\myparagraph{$\Psi$ Fast Transform}
In the case that $\Psi$ corresponds to a well-studied transform such as wavelets, curvelets, etc., fast transform implementations can be utilized to reduce complexity further.  For the case of FISTA, for example, matrix multiplications of $\Gamma^\top (\Gamma Z_k \Psi^\top) \Psi$ (See Algorithm~\ref{alg:Kron-FISTA} Step 2) involve fast transform reconstructions ($\Psi^\top$) of each DWI ($\Gamma Z_k$) and then deconstructions ($\Psi$) which we parallelize over all DWI in our implementation.

\myparagraph{$\Psi$ Orthonormal}
In the case that $\Psi$ is orthonormal, $\ie \ \Psi^\top \Psi = \Psi \Psi^\top =$ I then \eqref{eq:l1_separable} can be simplified to \eqref{eq:VoxReconMult} after noticing:
\begin{equation}
\label{eq:Kron_l1_orthonormal}
||\Gamma C \Psi^\top - S||_F^2 = ||\Gamma C \Psi^\top\Psi - S\Psi||_F^2 =  ||\Gamma C - \hat{S}||_F^2.
\end{equation}
This optimization can be solved using traditional methods after precomputing $\hat{S}=S\Psi$. 

\myparagraph{$\Psi$ Separable Tensor Product}
In the case that $\Psi$ can be separated into a 3D tensor product $\Psi = \Psi_x \otimes \Psi_y \otimes \Psi_z$, the complexity of multiplication can be simplified by another degree, in the same vein as the decrease in complexity we gained from using $\Phi = \Psi \otimes \Gamma$.  In this case, instead of the matrix multiplication, $S = \Gamma C \Psi^\top$ can be written using n-mode products of tensors $\mathcal{S} =  \mathcal{C} \times_x \Psi_x \times_y \Psi_y \times_z \Psi_z \times_q \Gamma$. Furthermore, if we consider DSI acquisition where $q$-space measurements are acquired in a grid over $\mathbb{R}^3$, and assume we can represent these measurements over a separable basis over each dimension, then we can take $\Gamma = \Gamma_{q_x} \otimes \Gamma_{q_y} \otimes \Gamma_{q_z}$ and $\Phi$ becomes a 6-tensor.

\section{Experiments}
\label{sec:experiments}
\subsection{Data}
\label{sec:data}
We perform our experiments on single-shell HARDI data, though as we emphasized earlier, our framework and algorithms can be applied to any dMRI acquisition protocol with a suitable choice of the angular basis $\Gamma$. We experimented on a phantom and a real HARDI brain dataset. Specifically, we applied our methods to the ISBI 2013 HARDI Reconstruction Challenge Phantom dataset \footnote{http://hardi.epfl.ch/static/events/2013\_ISBI/},
a $V\!=\!50\!\times\!50\!\times\!50$ volume consisting of 20 phantom fibers crossing intricately within an inscribed sphere, measured with $G\!=\!64$ gradient directions (SNR $=30$). Our initial experiments test on a 2D $50\!\times\!50$ slice of this data for simplification. The real HARDI brain dataset consists of a $V\!=\!112\!\times\!112\!\times\!65$ volume with $G=127$ gradient directions.  
We conducted experiments on the core white matter brain region of size $V\!=\!60\!\times\!60\!\times\!30$.
\subsection{Kronecker Algorithm Comparison}
\label{sec:algcomp}
In this section we compare the computational time performance of each of the proposed Kronecker LASSO algorithms, (Kron-ADMM, Kron-DADMM, and Kron-FISTA) on a 2D $50\times50$ slice of phantom data for various values of $\lambda$ using Haar-SR.  For our experiment, we ran Kron-FISTA until we reached a very small mean squared error of $\epsilon = 10^{-8}$.  The objective value obtained was then taken to be a rough ground truth minimum.  We then tested each of Kron-ADMM, Kron-DADMM, and Kron-FISTA and recorded the time it took to reach a relative error of $10^{-4}$ from the known minimum. Figure~\ref{fig:algorithmComparison} reports the objective value of each algorithm as a function of computing time for various sparsity levels associated to choices of $\lambda$. Table~\ref{fig:algorithmComparison} gives the number of iterations until completion for each method and sparsity level. For our experiments, Kron-FISTA appears to be the fastest algorithm in all cases, followed by Kron-DADMM.  The superior performance of DADMM over ADMM is consistent with the findings of \citep{Gonccalves:Thesis15}. 
With these results, we henceforth use Kron-FISTA for subsequent experiments.
\begin{figure}
\centering
\includegraphics[width=.9\linewidth,trim=120 0 120 0,clip]{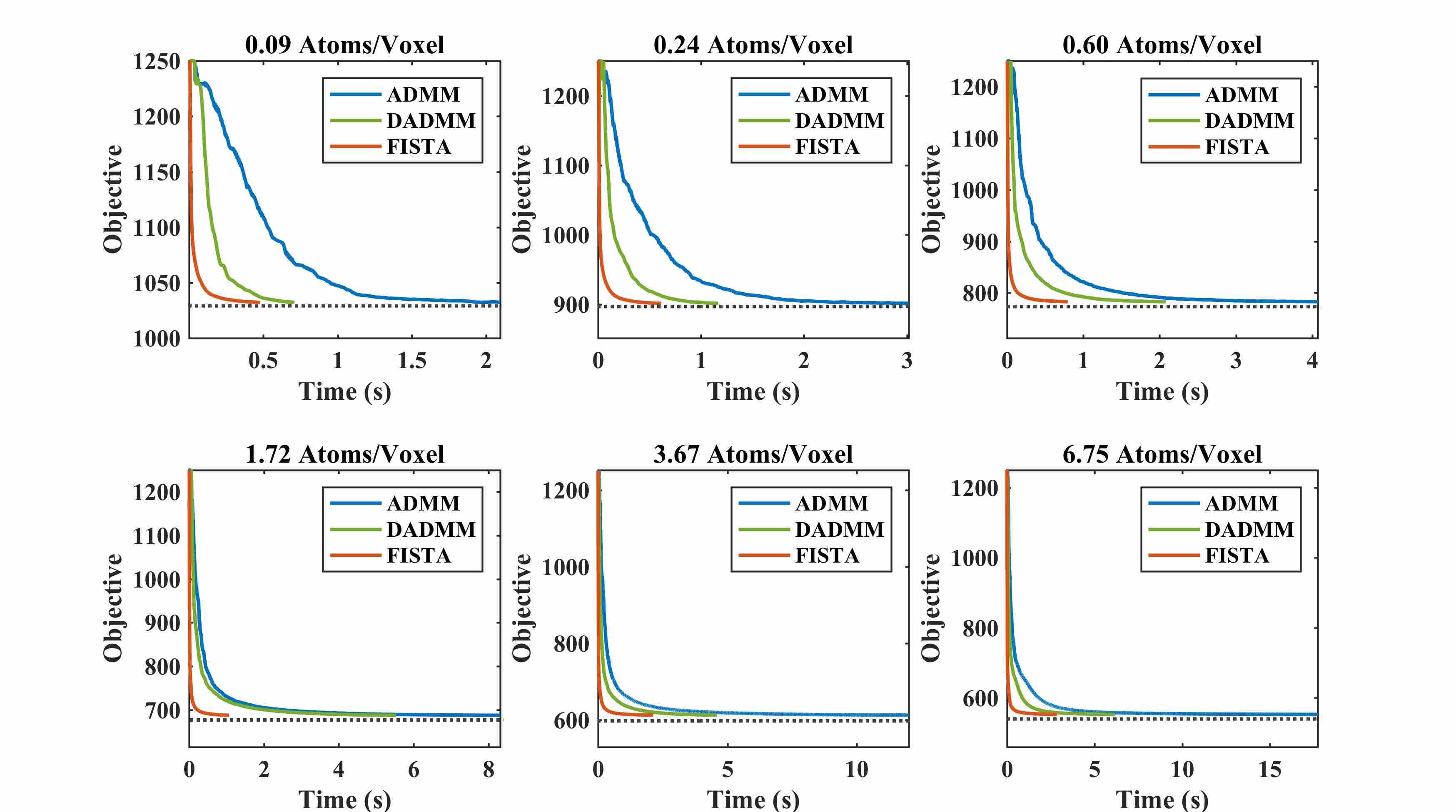}
\caption{Comparison of time for completion of Kron-ADMM, Kron-DADMM, and Kron-FISTA on a 2D $50\times50$ phantom HARDI data using Haar-SR for various sparsity levels. Kron-FISTA consistently reaches the minimum objective in the least amount of time. For number of iterations and lambda values, see Table~\ref{table:algorithmComparison}.}
\label{fig:algorithmComparison}
\end{figure}

\subsection{Choice of Spatial-Angular Dictionaries}
The experiments in this section are conducted using fixed spatial and angular dictionaries.  For the choice of the angular dictionary $\Gamma$, we consider the over-complete Spherical Ridgelet (SR) basis \citep{TristanVega:MICCAI11}, which has been shown to sparsely model HARDI signals. The corresponding dictionary in the space of ODFs is the set of spherical wavelets (SW) (see Figure~\ref{fig:kronbasis} for an example of one spherical wavelet atom).  With order $L=2$ and $4$, the SR dictionary contains $N_\Gamma=210$ and $N_\Gamma=1169$ atoms, respectively.  We used both amounts of atoms for the small 2D $50\times50$ phantom dataset and found roughly identical results suggesting that a basis of order $L=2$ contains enough atoms if the number of gradients is below $210$.  This reduces computation significantly.  In comparison, the spherical harmonic (SH) basis has been shown in prior work \citep{TristanVega:MICCAI11} to not exude sparse signals and so we do not repeat this comparison in the current work.  

\begin{table}[H]
\center
\begin{tabular}{|c|cccccc|}
    \hline
    Atoms/Voxel & 0.09 & 0.24 & 0.60 & 1.72 & 3.67 & 6.75\\
    $\lambda$ & $1.4^1$ & $1.4^{-1}$ & $1.4^{-3}$ & $1.4^{-5}$ & $1.4^{-7}$ & $1.4^{-9}$ \\ \hline
    Kron-ADMM & 797 & 1462 & 2096 & 3660 & 4365 & 4667 \\
    Kron-DADMM & 357 & 597 & 1060 & 1722 & 1928 & 1953 \\
    Kron-FISTA & 161 & 219 & 288 & 346 & 584 & 611 \\
    \hline
\end{tabular}
\caption{
Number of iterations to completion for Kron-ADMM, Kron-DADMM, Kron-FISTA on a 2D $50\times 50$ phantom HARDI data using Haar-SR. Kron-FISTA converges in the fewest number of iterations. For computation time, see Figure~\ref{fig:algorithmComparison}.}
\label{table:algorithmComparison}
\end{table}

For the choice of spatial dictionary $\Psi$, the spatial wavelet transform is a popular sparsifying basis for natural images and structural MRI volumes.  In our previous work \citep{Schwab:MICCAI16} we compared the performance of Daubechies wavelets and Haar wavelets and concluded that Daubechies wavelets resulted in a smoothing of the boundaries between isotropic and anisotropic regions which was not indicative of the more abrupt boundaries that real HARDI data exhibits.  Haar wavelets outperformed Daubechies wavelets in terms of reconstruction error arguably due to the fact that HARDI data exhibits more rigid boundaries and piece-wise consistencies, a spatial feature that has motivated the use of total-variation penalties in many other reconstruction methods.  For this reason, we do not consider Daubechies wavelets in this work.  

In addition to Haar wavelets, we consider the spatial curvelets dictionary \citep{Candes:MMS06} (featured as the spatial atom in Figure~\ref{fig:kronbasis}) which, in addition to variations in position and scale, offers directional variations which may be useful for sparsely modeling the naturally directional HARDI fiber tracts regions.  An important criteria for choosing our spatial basis is that they be tight frames as this choice has important theoretical implications for compressed sensing and offers computational advantages (as discussed in Section~\ref{sec:complexity}). They additionally have fast transform implementations which also reduce computational complexity.  
Finally, to compare our formulation to state-of-the-art voxel-wise angular sparse coding, we can simply choose $\Psi$ to be the $V\!\times\!V$ identity I$_V$. 
For ease of notation, we use a spatial-angular $\Psi$-$\Gamma$ labeling: Haar-SR, Curve-SR, I-SR for Haar wavelets, curvelets, and the identity, respectively, for the spatial domain with SR for the angular domain.
\begin{figure}
\centering
\includegraphics[width=.8\linewidth,trim=0 0 0 0,clip]{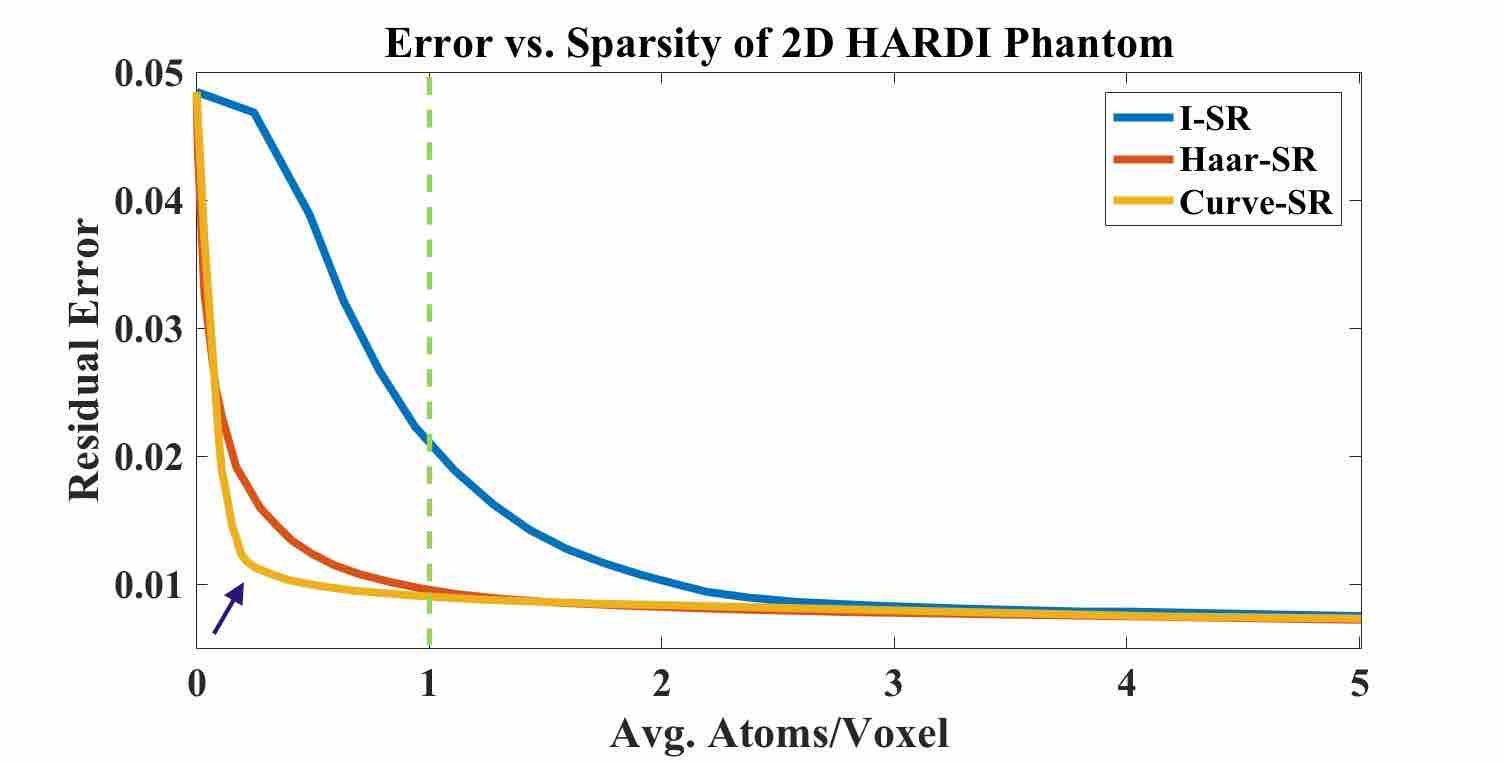}
\caption{Quantitative results of residual error vs. spatial-angular sparsity levels for I-SR, Haar-SR, and Curve-SR on 2D phantom data for various values of $\lambda$. Curve-SR out performs Haar-SR for low sparsity levels while I-SR has very high relative reconstruction error.  The reconstruction of I-SR data points are displayed in Figure~\ref{fig:IDanalysis} and Haar-SR/Curve-SR in Figure~\ref{fig:PhantomQualitative5050}. Our finding of I-SR requiring 6-8 atoms per voxel for accurate reconstruction is consistent with previous findings \citep{Michailovich:ISBI08,Michailovich:TIP10}.}
\label{fig:ResidualVsSparsityPhantom5050}
\end{figure}
\begin{figure}
\center
 \includegraphics[width=.38\linewidth,trim=0 0 0 0,clip]{PhantomSlice25_SHGTsparsity15_lowres2.jpg}
     \includegraphics[width=.38\linewidth,trim=210 190 170 65,clip]{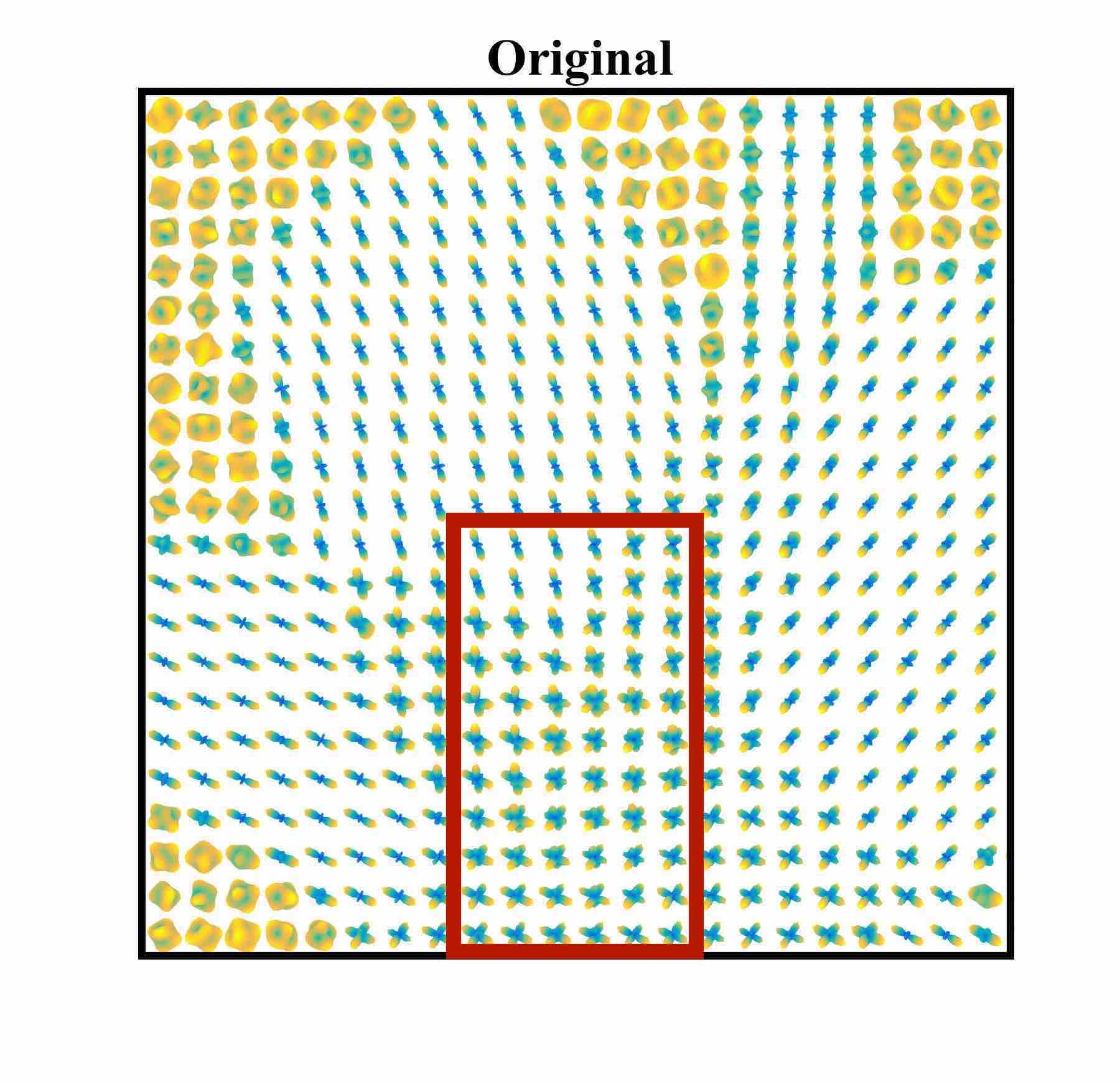}
      \includegraphics[width=.21\linewidth,trim=0 0 0 0, clip]{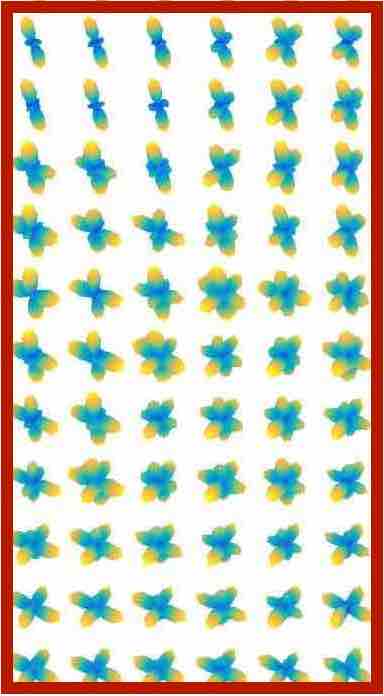}\\
          \includegraphics[width=.38\linewidth,trim=0 0 0 0,clip]{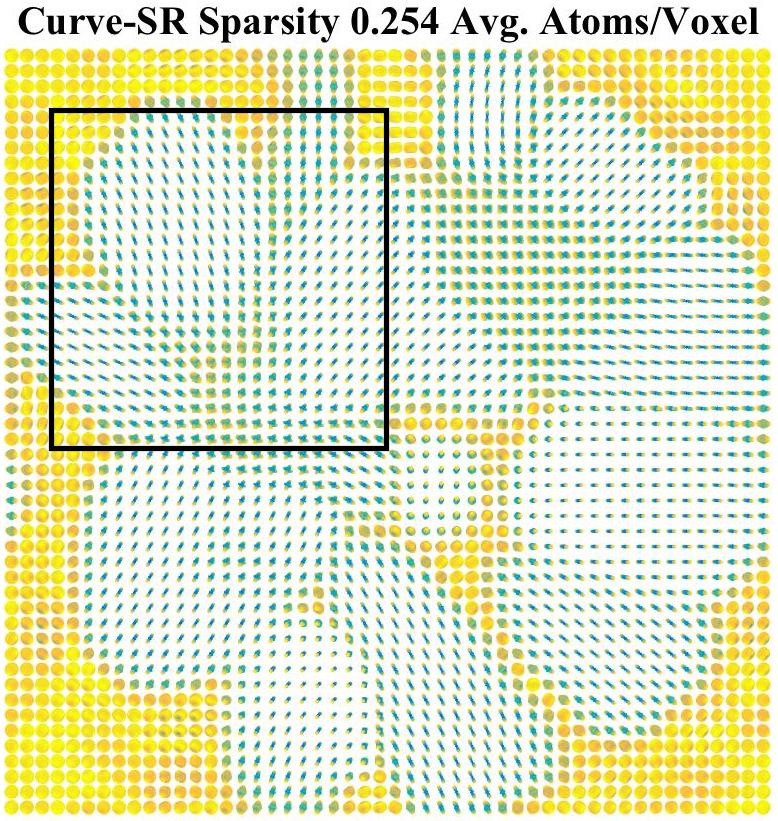}
    \includegraphics[width=.38\linewidth,trim=0 0 0 0,clip]{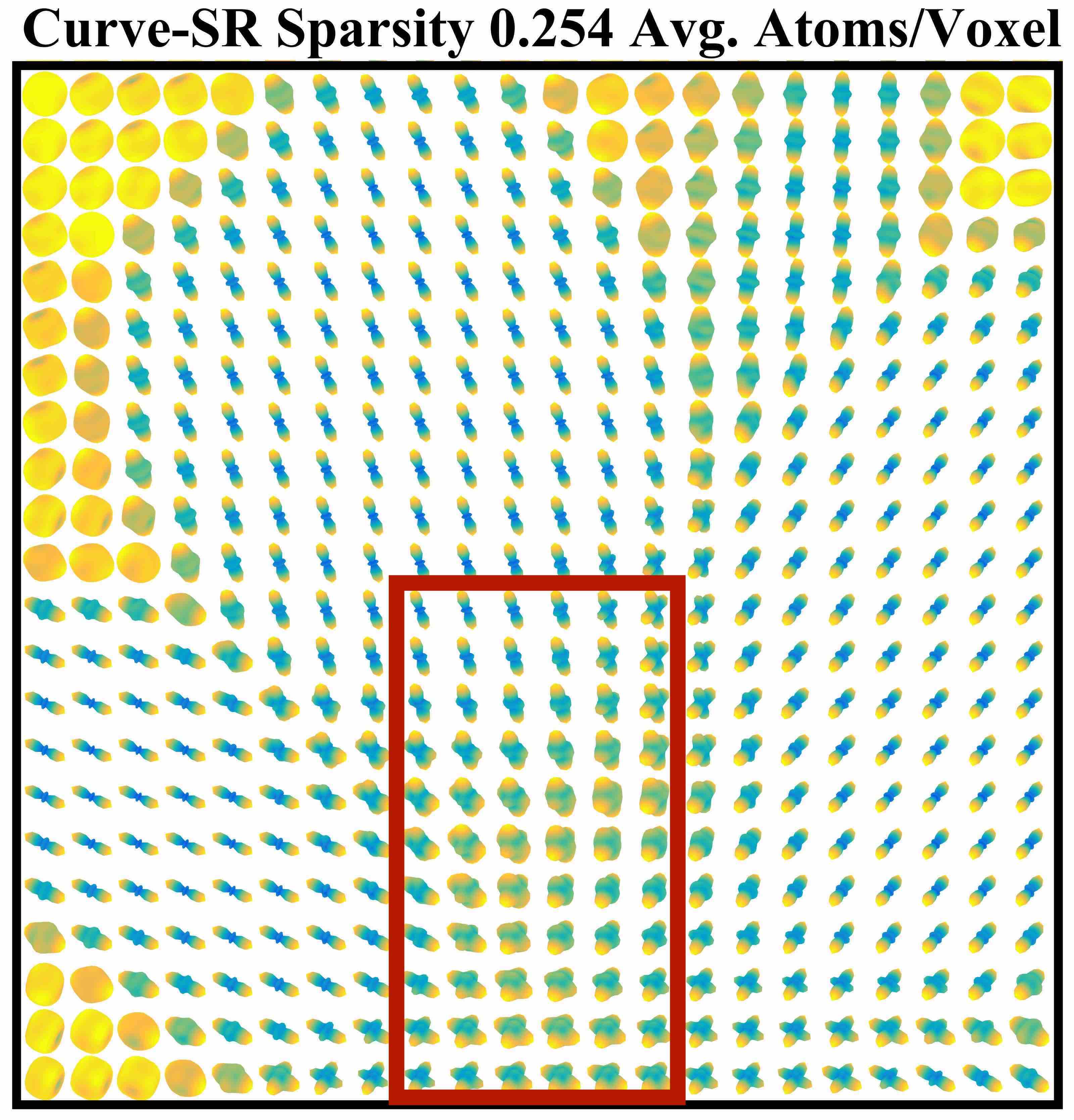}
    \includegraphics[width=.21\linewidth,trim=0 0 0 0, clip]{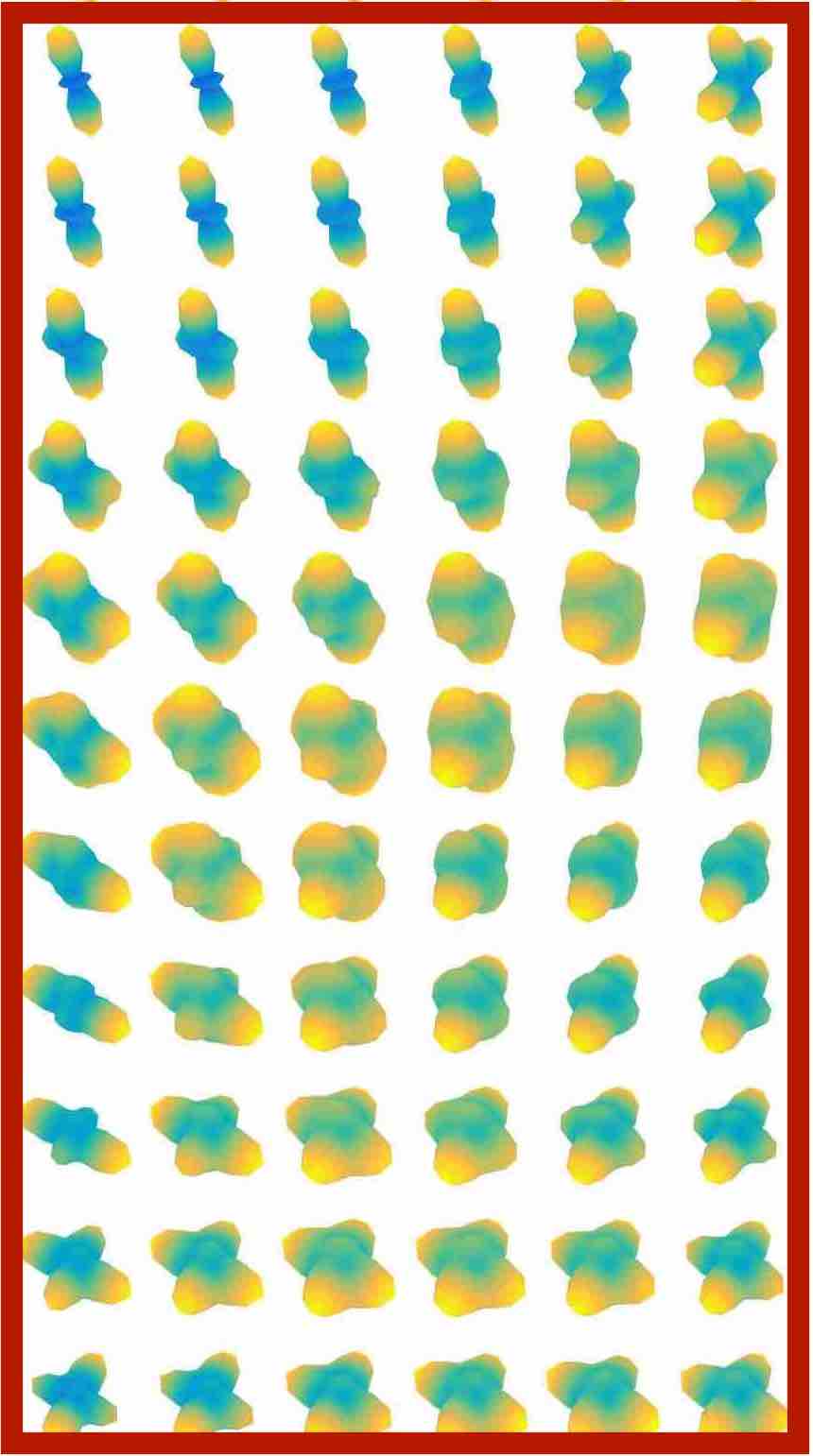}\\
 \includegraphics[width=.38\linewidth,trim=0 0 0 0,clip]{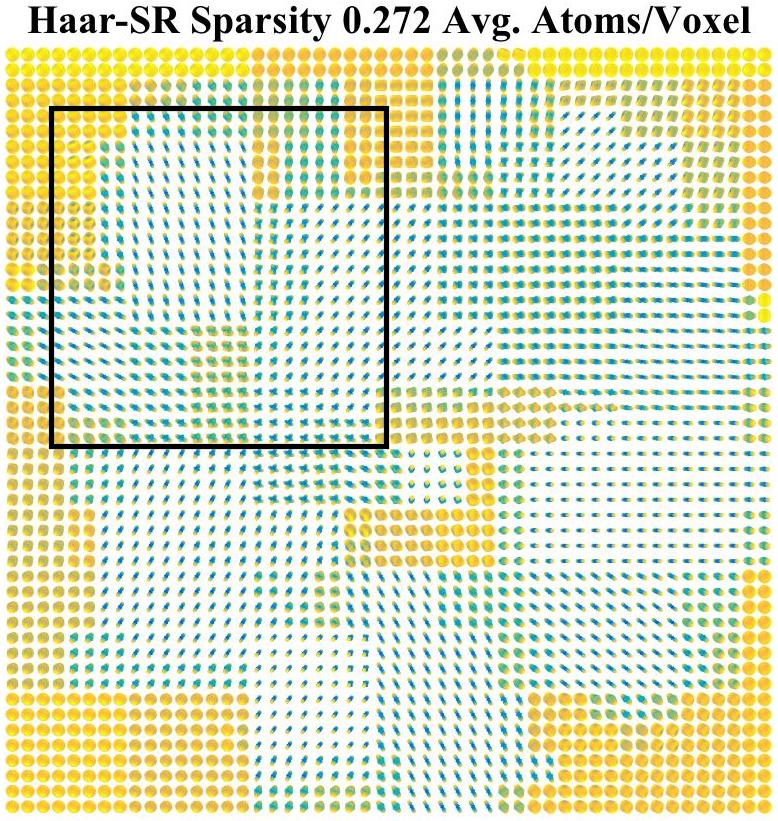}
       \includegraphics[width=.38\linewidth,trim=0 0 0 0,clip]{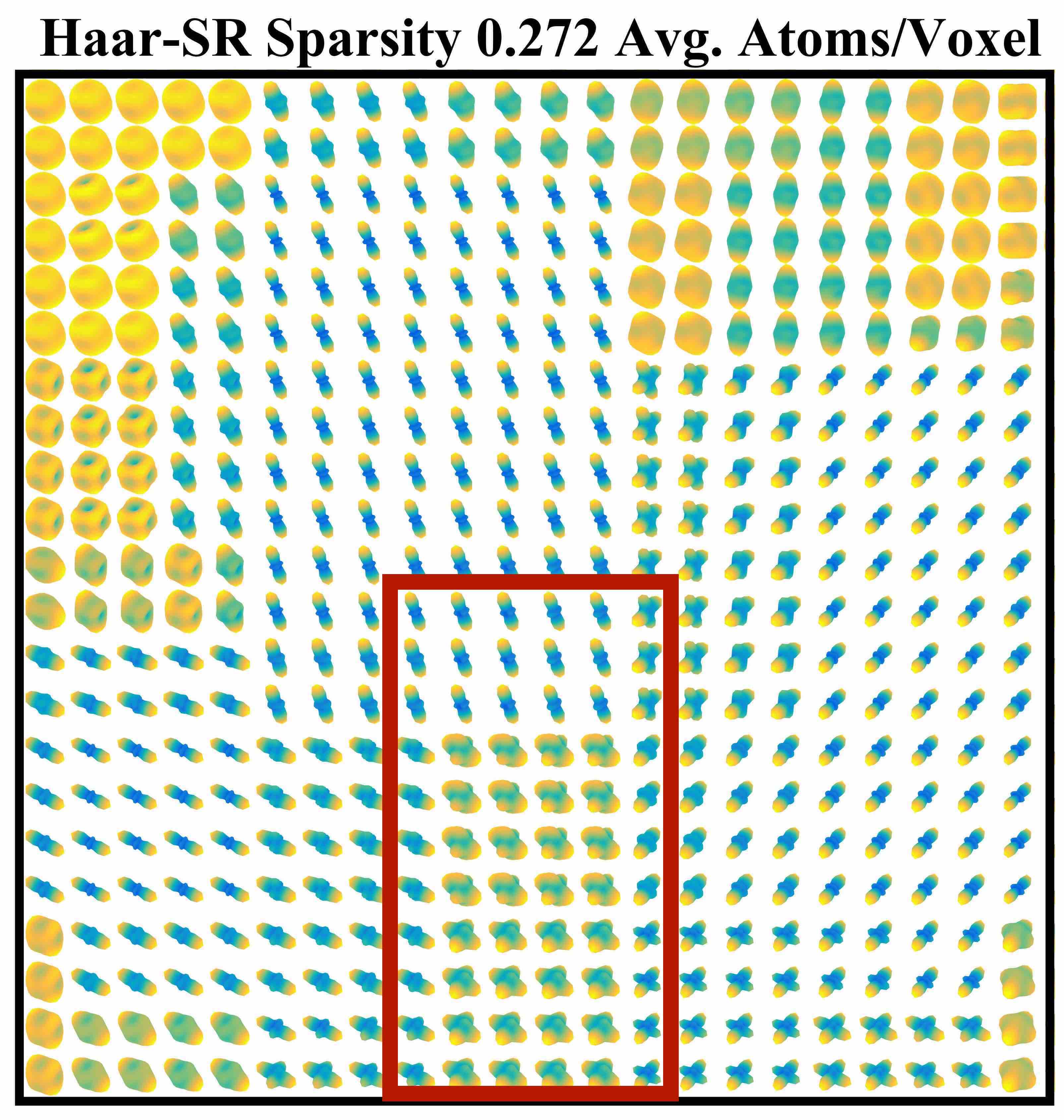}
      \includegraphics[width=.21\linewidth,trim=0 0 0 0,clip]{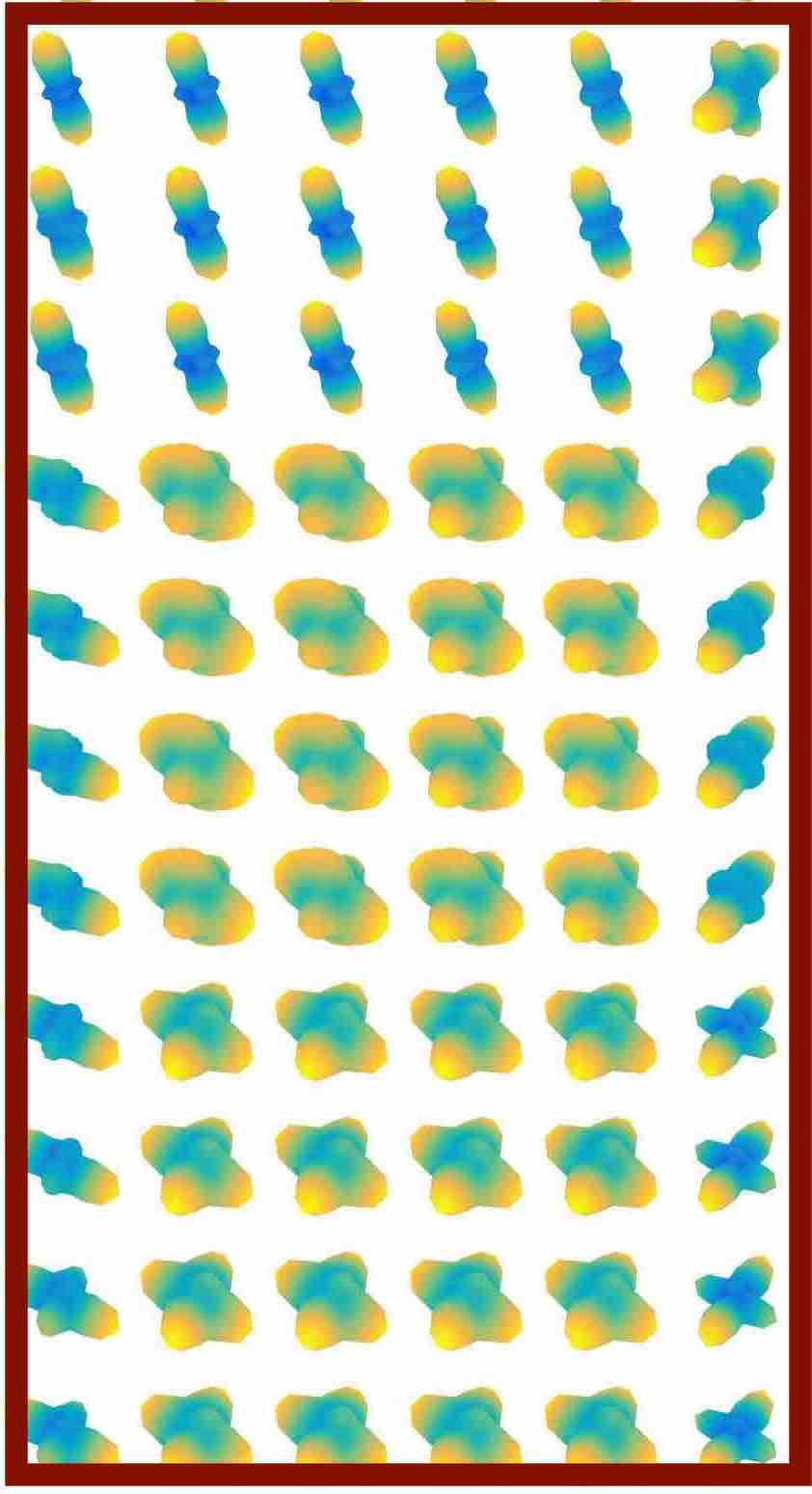}
      \\
 \caption{Results of the proposed spatial-angular sparse coding using Kron-FISTA for Haar-SR and Curve-SR using an average of $\sim0.25$ atoms/voxel compared to original signal.  Curve-SR outperforms Haar-SR in this regime due to its additional directionality. We can see a drastically better reconstruction compared to the state-of-the-art at the same sparsity level in the top left of Figure~\ref{fig:IDanalysis}.  This clearly shows that we can achieve accurate reconstruction with less than 1 atom/voxel.}
 \label{fig:PhantomQualitative5050}
\end{figure}

\subsection{Sparsity Results}
\label{sec:results}
In this section we compare the performance of our spatial-angular sparse coding method to the state-of-the-art angular sparse coding by analyzing reconstruction accuracy using very few nonzero coefficients.  The first experiment is tested on the $50\times50$ phantom data slice.  We ran Kron-FISTA for various values of $\lambda$ for Haar-SR, Curve-SR and I-SR.  In Figure~\ref{fig:ResidualVsSparsityPhantom5050} we show the results of residual reconstruction error $\frac{1}{GV}||S^* - S_{orig}||_F$ vs. spatial-angular sparsity levels in terms of the average number of atoms per voxel ($||C^*||_0/V$).   The ideal reconstruction will have a very low average number of atoms per voxel with low residual error, which happens in the lower left-hand corner of our plot. We can see that in this range, Curve-SR outperforms Haar-SR while I-SR is unable to perform at this level.  Reconstruction of I-SR for various sparsity levels are visualized in Figure~\ref{fig:IDanalysis}.  In comparison, Figure~\ref{fig:PhantomQualitative5050} displays the sparse reconstruction of Haar-SR and Curve-SR with an average of 0.25 atoms/voxel. Notice that Curve-SR leads to a somehow smoother and more accurate reconstruction than the expectedly boxy reconstruction of Haar-SR at this very high sparsity level.  Still, in both cases, the proposed joint spatial-angular sparse coding can reconstruct accurate signals with much fewer number of atoms than angular sparse coding, which as seen again from Figure~\ref{fig:IDanalysis} can be achieved with an average of around 4 atoms per voxel. More strikingly, in cases of high signal complexity for crossing fibers, the sparse code requires on the order of 6-12 atoms per voxel (see Figure~\ref{fig:nonzerocount}).

We repeat this same analysis on real HARDI data. First, as was investigated for the phantom data in Section~\ref{sec:state-of-the-art}, we analyze the effect of adding TV regularization to the angular sparse coding with different weighting $\alpha$ in \eqref{eq:TV}, with $\alpha =0$ being equivalent to the purely angular I-SR model. The algorithm used to solve the resulting optimization problem is the Split Bregman procedure outlined in \citep{Michailovich:TMI11}. Consistent with the phantom experiment of Figure~\ref{fig:angTV}, we observe from Figure~\ref{fig:ResidualVsSparsityReal606030_ISRplusTV} that adding spatial regularization again increases the total number of atoms for a given reconstruction error compared to the $\alpha=0$ case. While the reconstructed signals may have a qualitatively  spatial regularity which may result in a qualitatively better output. In comparison, the joint model we propose achieves better sparsity levels for comparable reconstruction errors. Note also that both algorithms displayed very similar performances in terms of running time for that particular experiment.

We finally validate the approach in the case of a full 3D HARDI volume for a very sparse number of atoms. Figure~\ref{fig:ResidualVsSparsityReal606030} presents the reconstruction error vs. sparsity results for the state-of-the-art framework (in which we set $\alpha=0$ as the results of Figure \ref{fig:ResidualVsSparsityReal606030_ISRplusTV} suggest) versus the joint Haar-SR and Curve-SR. The plot shows again that curvelets outperforms Haar for high sparsity levels in the range of 0.5-2 atoms/voxel. As expected and consistent with our phantom data experiment, the state-of-the-art I-SR has comparable reconstruction error in the range of 6-8 atoms/voxel.  Figure~\ref{fig:RealQualitative606030} shows the quality of reconstruction of I-SR, Haar-SR, and Curve-SR compared to the original signal using an average of $\sim1$ atom/voxel. Haar-SR presents boxy regions while Curve-SR maintains a smoother reconstruction with a preservation of smaller detailed fiber tract regions.  In contrast, the state-of-the-art I-SR is unable to model intricate fiber regions and is forced to set most voxels to zero atoms.  All in all, we can see that using our proposed method, we can achieve much higher sparsity levels than the state-of-the-art, and accurate reconstructions using less than 1 atom/voxel.  In terms of efficiency, Kron-FISTA was completed on the real HARDI data of size $V\!=\!60\!\times\!60\!\times\!30$, $G\!=\!127$ in 1.5 hours for our sparsity level of interest using the fast 3D wavelet transform implemented in MATLAB.
\begin{figure}
\centering
\includegraphics[width=.8\linewidth,trim=0 0 0 0,clip]{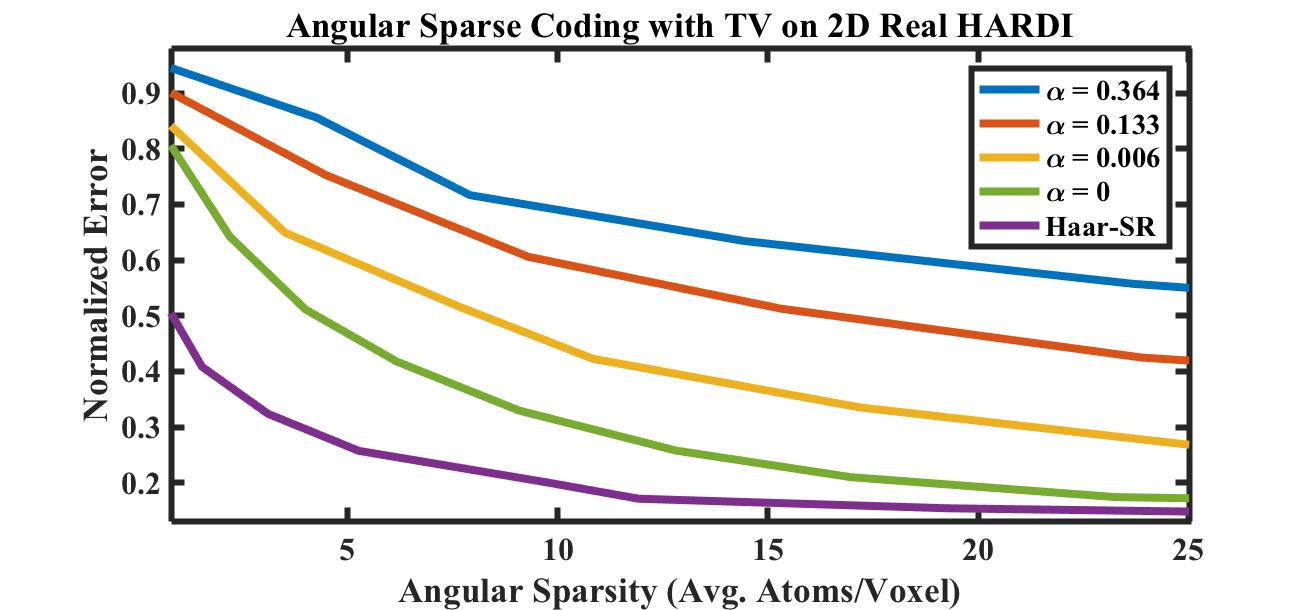}
\caption{Comparison of the spatial-angular sparsity levels achieved by our proposed joint method with Haar-SR dictionaries and state-of-the-art method of \citep{Michailovich:TMI11} with different spatial regularization parameters $\alpha$, both applied on a 2D slice of a real HARDI scan. Recall $\alpha=0$ corresponds to purely angular sparse coding, I-SR. Adding TV regularization results in an increase of the number of atoms for a given reconstruction error. The joint method achieves better sparsity levels than using separate sparsity penalties.}
\label{fig:ResidualVsSparsityReal606030_ISRplusTV}
\end{figure}
\begin{figure}
\centering
\includegraphics[width=.8\linewidth,trim=0 0 0 0,clip]{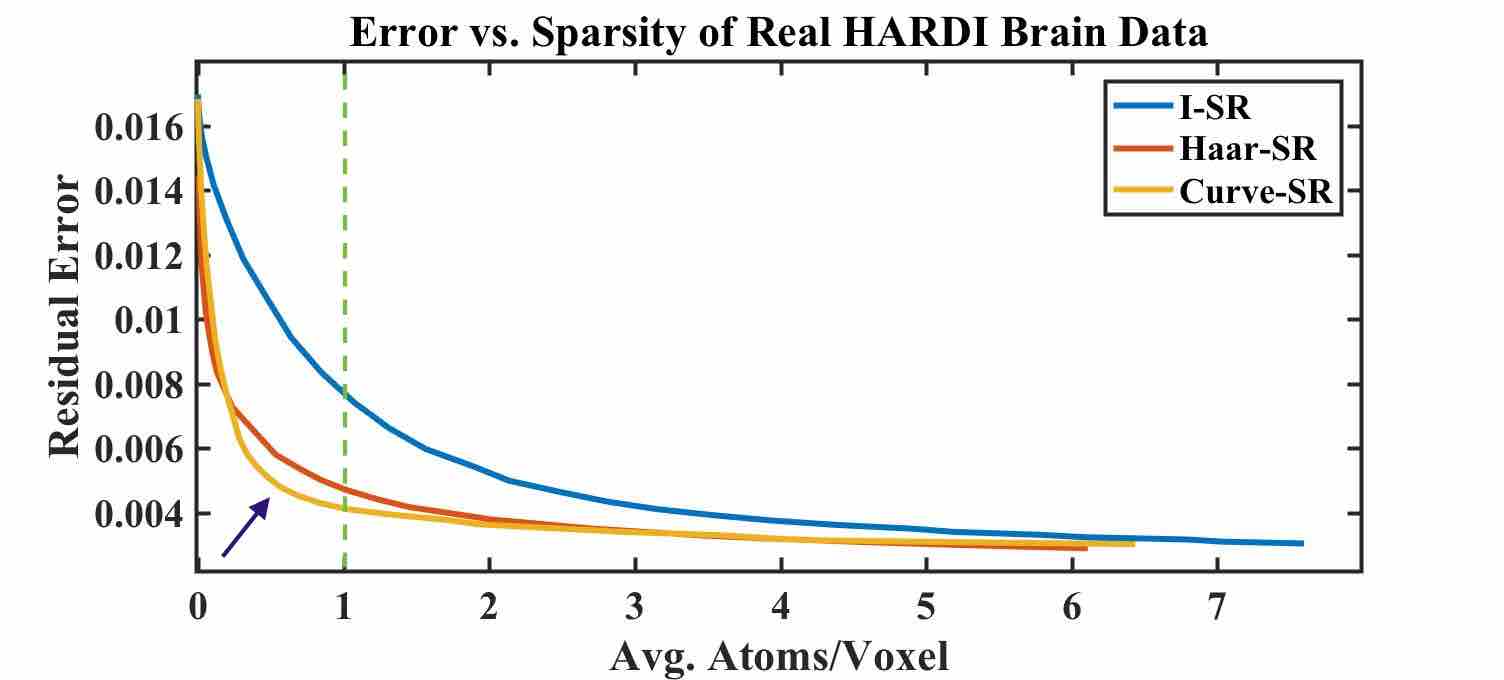}
\caption{Comparison of the spatial-angular sparsity level achieved by Haar-SR and Curve-SR with respect to the state-of-the-art I-SR on the entire 3D real HARDI volume. The curvelets provide a good reconstruction error with the sparsest number of atoms, in the range of 0.5 to 2 atoms/voxel.  The state-of-the-art error is much larger in this sparsity range and only comparable in the predicted range of 6-8 atoms/voxel, consistent with the previously reported \citep{Michailovich:ISBI08,Michailovich:TIP10} for I-SR.}
\label{fig:ResidualVsSparsityReal606030}
\end{figure}
\begin{figure}
\center
 \includegraphics[width=.48\linewidth,trim=0 0 0 0,clip]{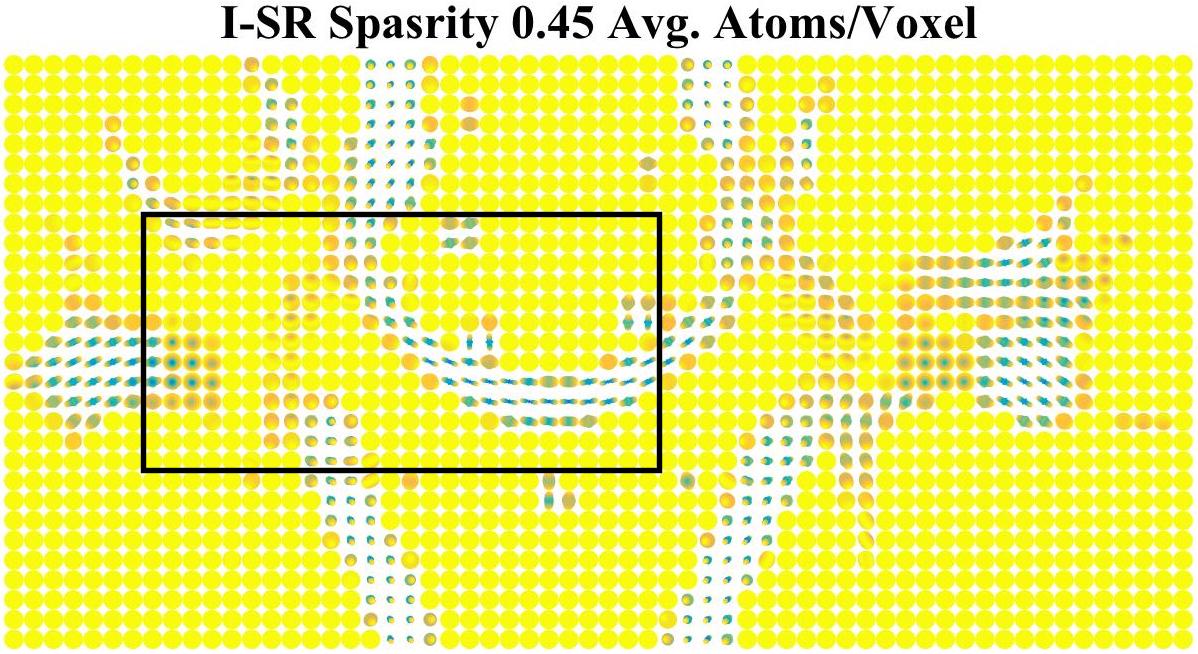}
     \includegraphics[width=.48\linewidth,trim=0 0 0 0,clip]{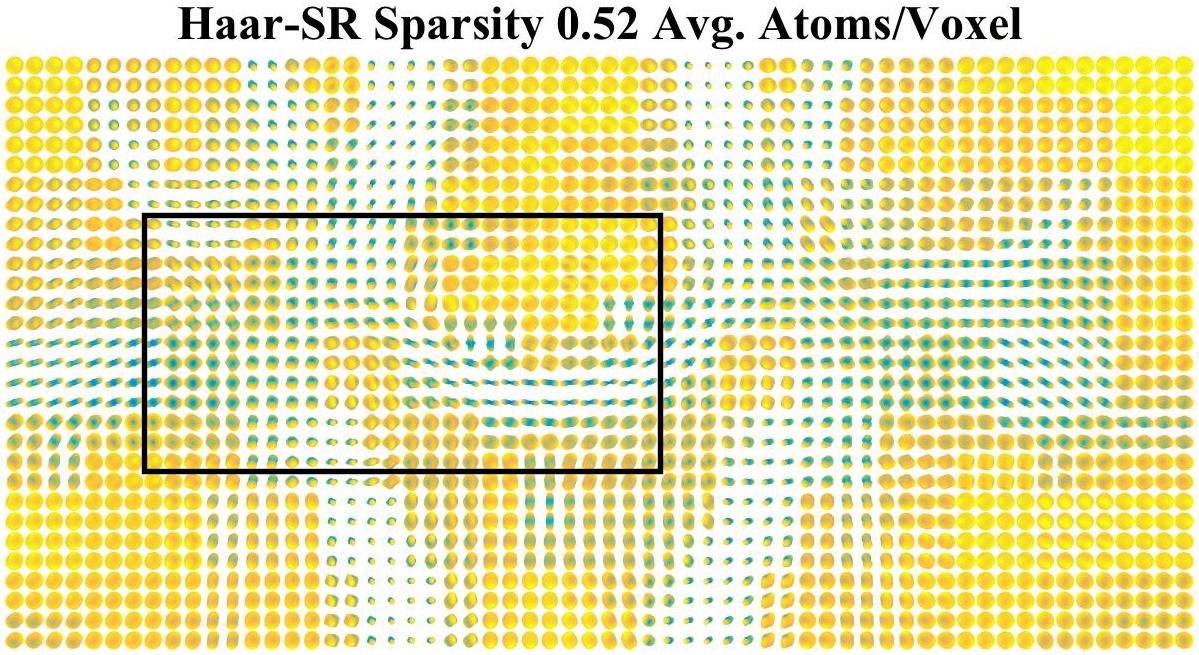}\\
      \includegraphics[width=.48\linewidth,trim=0 0 0 0,clip]{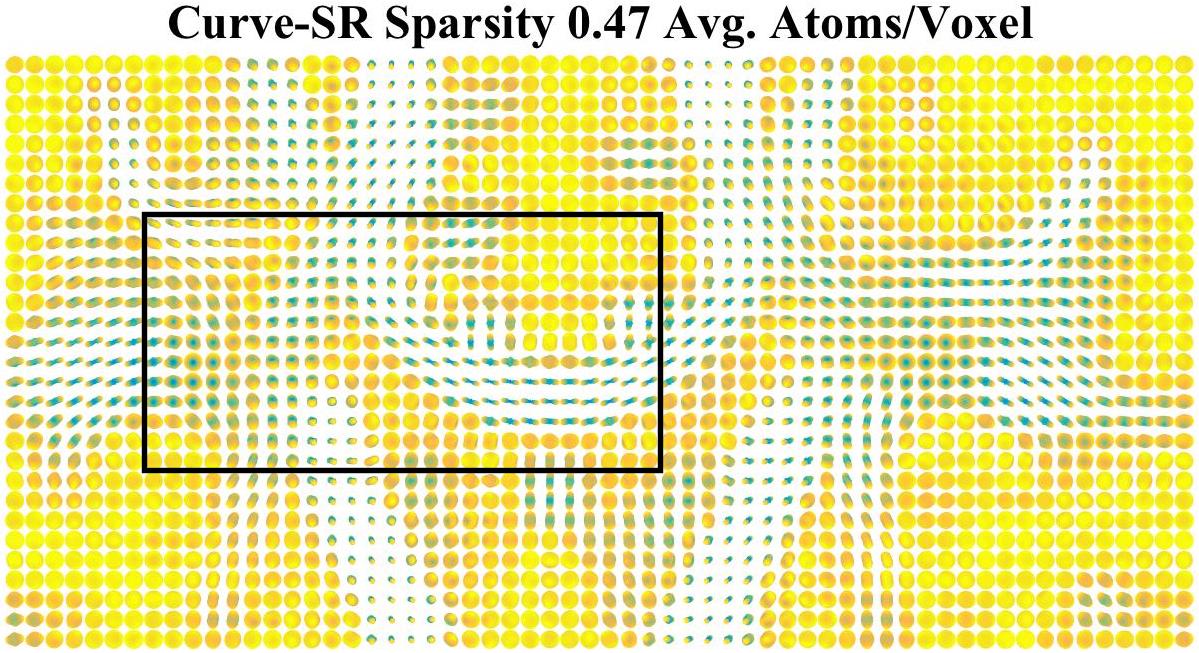}
    \includegraphics[width=.48\linewidth,trim=0 0 0 0,clip]{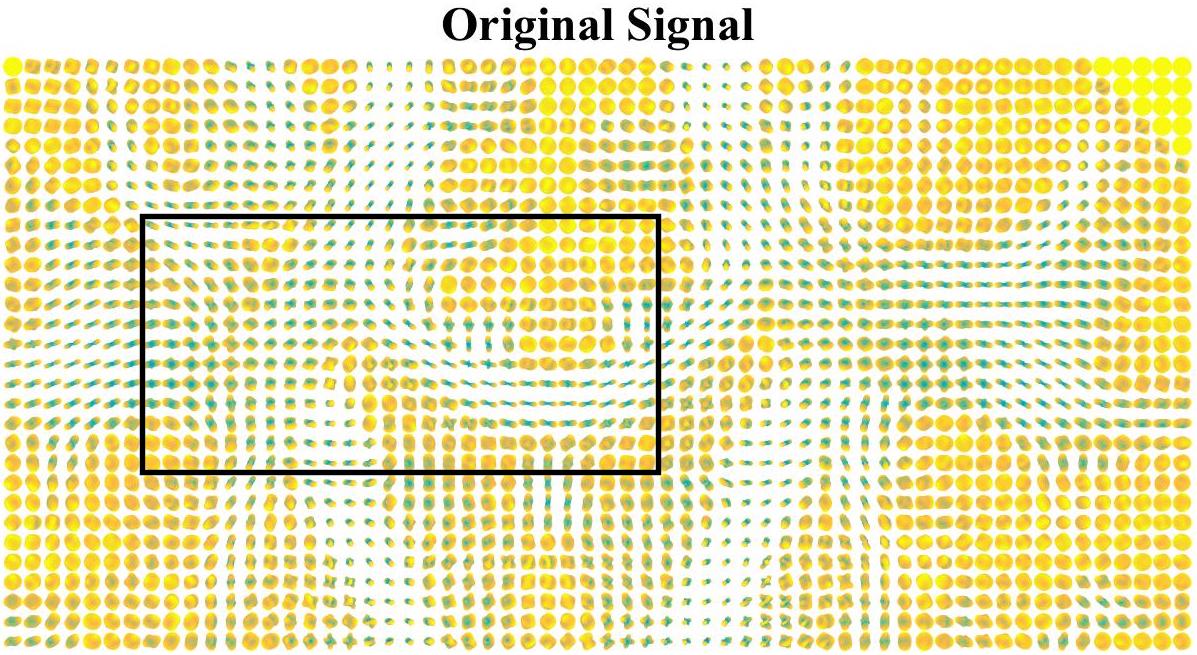}\\
     \includegraphics[width=.48\linewidth,trim=0 0 0 0,clip]{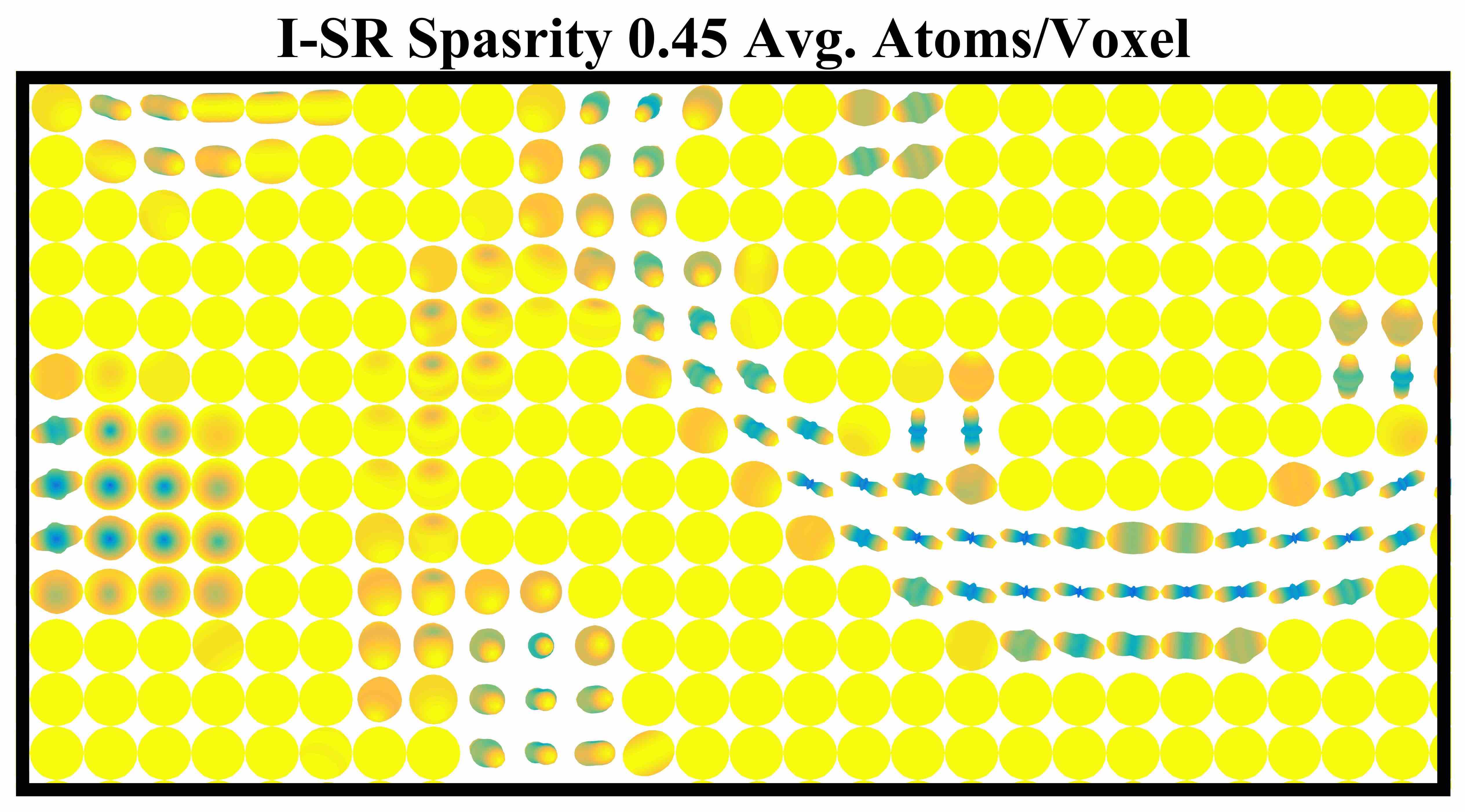}
     \includegraphics[width=.48\linewidth,trim=0 0 0 0,clip]{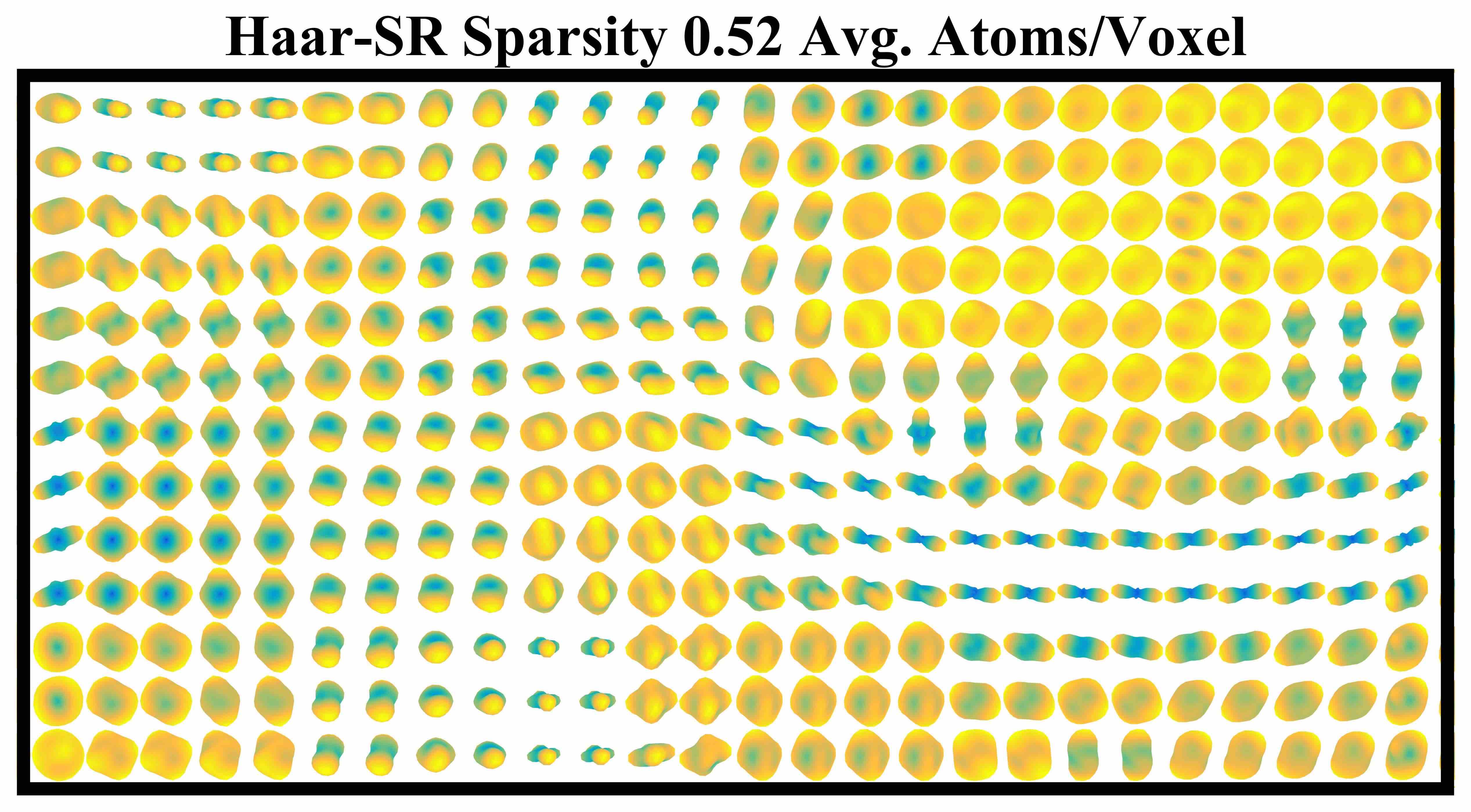}\\
      \includegraphics[width=.48\linewidth,trim=0 0 0 0,clip]{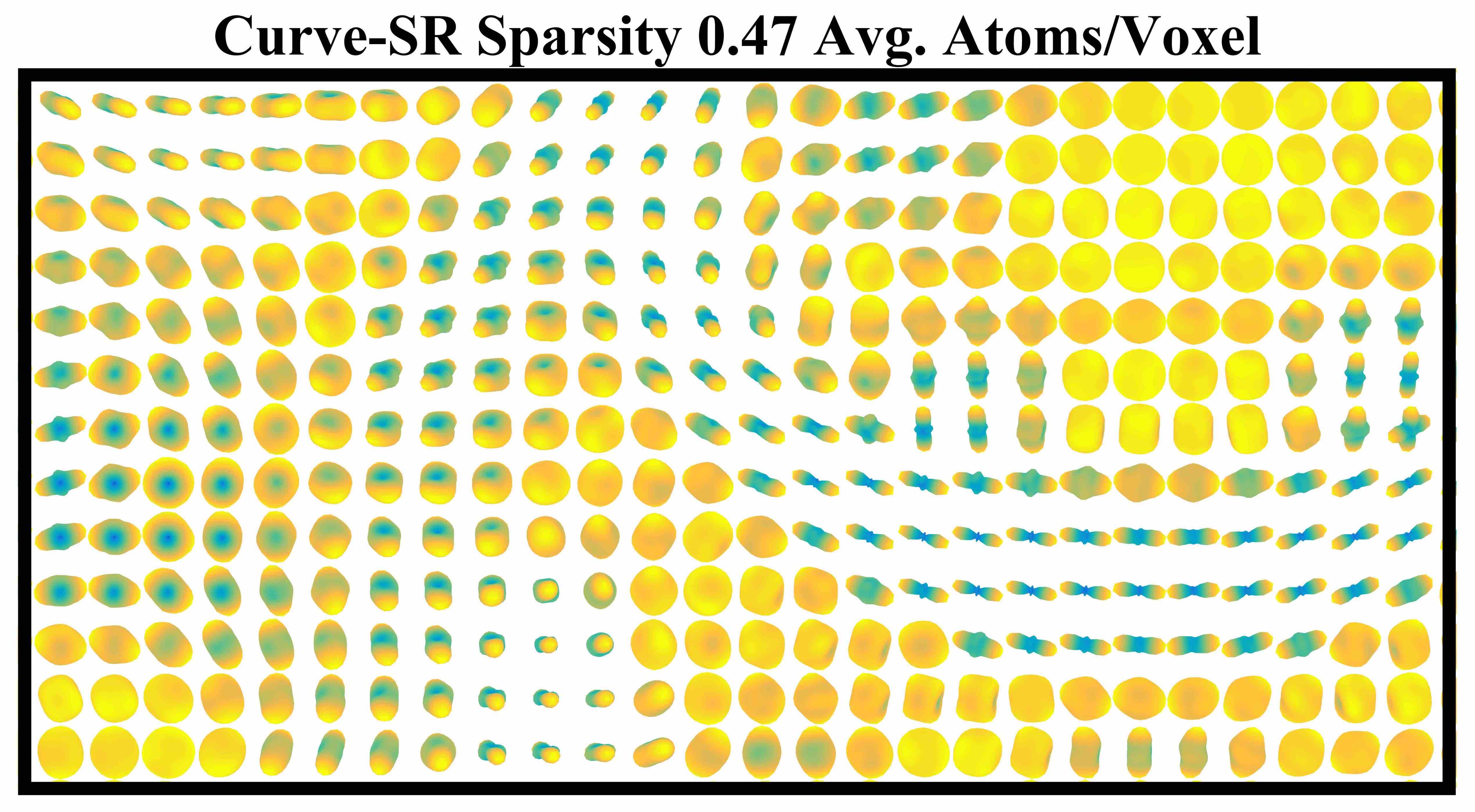}
    \includegraphics[width=.48\linewidth,trim=0 0 0 0,clip]{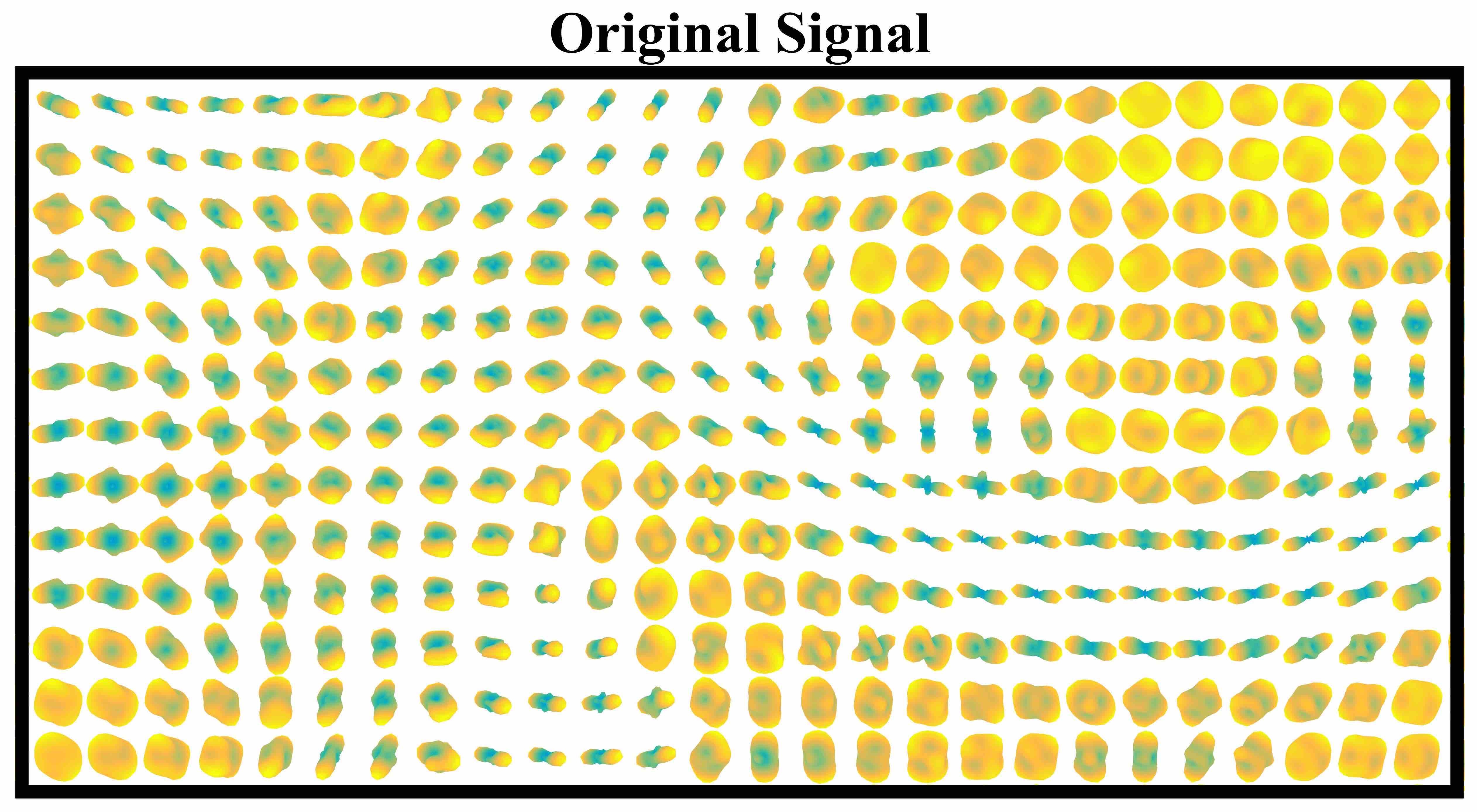}
 \caption{Results of proposed spatial-angular sparse coding on real HARDI brain data using Kron-FISTA for I-SR, Haar-SR and Curve-SR using an average of $\sim0.5$ atoms/voxel compared to original signal.  Curve-SR outperforms Haar-SR in this sparsity range due to its directionality.  The state-of-the-art I-SR is unable to compete at this sparsity level.}
 \label{fig:RealQualitative606030}
\end{figure}
\section{Discussion and Conclusion}
\label{sec:conclusion}
In this work, we have demonstrated that by using a joint spatial-angular dictionary, we can obtain accurate HARDI reconstruction with spatial-angular sparsity levels of less than 1 atom per voxel, surpassing the limitations of state-of-the-art angular representations.  This provides a new general reconstruction framework to achieve sparser dMRI representations than previously possible with optimal choices of spatial and angular dictionaries.  In particular, we have shown promising sparsity results for HARDI from the combination of curvelet (spatial) and spherical ridgelet (angular) dictionaries, but other spatial and angular dictionaries may be chosen for other dMRI protocols like DSI or MS-HARDI.  In future work, we aim to further optimize sparsity levels by learning a joint spatial-angular dictionary directly from dMRI data.

Furthermore, to efficiently solve this large-scale global sparse coding problem, we have proposed three novel extensions of popular sparse coding algorithms for the Kronecker dictionary setting.  All strategies improve upon previously proposed algorithms by explicitly exploiting the separability of the dictionary and each may be beneficial depending on the problem regime and size of data.  For our large-scale HARDI data, Kron-FISTA was the leader in speed.  
In future work, we will investigate other efficient active set methods such as the recent ORacle Guided Elastic Net (ORGEN) \citep{You:CVPR16-EnSC}. 


In addition to sparse coding, our spatial angular representation may have novel applications in other areas of dMRI processing such as denoising, feature extraction, global ODF non-negativity, fiber tract segmentation, and tractography.  
However, our main application for spatial-angular sparse coding framework is the promising improvements of acquisition acceleration of dMRI through CS. One natural future extension of this work will be to incorporate our joint spatial-angular sparsifying dictionaries within a unified ($k,q$)-CS framework to subsample signal measurements both in $k$- and $q$-space. With the adequate design of joint sensing schemes, CS recovery results such as \citep{Candes:ACHA11} predict that the minimum number of samples needed for stable and accurate reconstruction is directly linked to the sparsity of the signal in the chosen dictionary, which argues in favor of the joint representation we have proposed.

\myparagraph{Acknowledgements} This work is supported by Johns Hopkins University startup funds.
\section*{References}

\end{document}